\documentclass{article}

\usepackage{amsfonts}
\usepackage{amssymb,bm}
\usepackage{amstext}
\usepackage{wrapfig}
\usepackage{booktabs} 
\usepackage{amsmath}
\usepackage{xspace}
\usepackage{theorem}
\usepackage{graphicx}
\usepackage{url}
\usepackage{enumitem}
\usepackage{graphics}
\usepackage{colordvi}
\usepackage{subfigure}
\usepackage{xr-hyper}
\usepackage{hyperref}
\usepackage{comment}
\usepackage{threeparttable}

\usepackage{xcolor}

\DeclareMathOperator*{\argminA}{arg\,min}
\newtheorem{lemma}{Lemma}
\newtheorem{assumption}{Assumption}
\newtheorem{theorem}{Theorem}
\newtheorem{corollary}{Corollary}
\newtheorem{definition}{Definition}
\newtheorem{claim}{Claim}

\newif\ifsubmit

  \submitfalse

\ifsubmit
\newcommand{\zhe}[1]{}
\newcommand{\ruoyu}[1]{}
\else
\newcommand{\zhe}[1]{{\color{blue}{[Zhe: #1]}}}
\newcommand{\ruoyu}[1]{{\color{red}{[Ruoyu: #1]}}}
\fi

\newif\ifsmalldata

 \smalldatatrue 

\newif\ifnonsmooth

 \nonsmoothtrue 

\newenvironment{proof}{\par \smallskip{\bf Proof:}}{\hfill\stopproof}
\def\stopproof{\square}
\def\square{\vbox{\hrule height.2pt\hbox{\vrule width.2pt height5pt \kern5pt
\vrule width.2pt} \hrule height.2pt}}

\begin{document}
	
	%
	
	%

	\title{Sub-Optimal Local Minima Exist for Neural Networks with Almost All Non-Linear Activations}
	\author{{Tian Ding}\thanks{Department of Information Engineering, The Chinese University of Hong Kong, Hong Kong. \texttt{tianding@link.cuhk.edu.hk}. The work is done while the author is visiting Department of ISE, Univeristy of Illinois at Urbana-Champaign. The author contributes equally to this paper.}
		{\quad \quad \quad Dawei Li} \thanks{Coordinated Science Laboratory, Department of ISE, University of Illinois at Urbana-Champaign, Urbana, IL. \texttt{dawei2@illinois.edu}. The author contributes equally to this paper.}  
		{\quad \quad  \quad Ruoyu Sun} \thanks{Coordinated Science Laboratory, Department of ISE, University of Illinois at Urbana-Champaign, Urbana, IL. \texttt{ruoyus@illinois.edu}.} \
		\date{Nov 4, 2019}
	}
	\maketitle

	
	
	\begin{abstract}
	Does over-parameterization eliminate sub-optimal local minima for neural networks? An affirmative answer was given by a classical result in \cite{yu1995local} for 1-hidden-layer wide neural networks. A few recent works have extended the setting to multi-layer neural networks, but none of them has proved every local minimum is global. Why is this result never extended to deep networks? 
		
	In this paper, we show that the task is impossible because the original result for 1-hidden-layer network in \cite{yu1995local} can not hold. More specifically, we prove that for any multi-layer network with generic input data and non-linear activation functions, sub-optimal local minima can exist, no matter how wide the network is (as long as the last hidden layer has at least two neurons). While the result of \cite{yu1995local}  assumes sigmoid activation, our counter-example covers a large set of activation functions (dense in the set of continuous functions), indicating that the limitation is not due to the specific activation. Our result indicates that ``no bad local-min'' may be unable to explain the benefit of over-parameterization for training neural nets.
	\end{abstract}

	\section{Introduction}

	
	
	Understanding why deep neural networks work well has attracted
	much attention recently. Researchers have formed some understanding of the design choices. For instance, it is believed that increasing the depth (i.e. increasing the number of layers) can increase the representation power
	\cite{safran2017depth,telgarsky2016benefits,liang2016deep}, and increasing the width (i.e. increasing the neurons per layer)
	   can reduce the optimization error (e.g. \cite{livni2014computational,geiger2018jamming}).
 Why does increasing the width reduce the optimization error?
 One popular conjecture is that increasing the width ``smooths the landscape''.
 \cite{li2018visualizing} visualized the loss surface in 2-dimensional space and demonstrated by graphics that increasing the width can make the landscape much smoother. 
However, these empirical studies do not provide  a rigorous theoretical result. 
	
There have been some efforts on providing theoretical evidence
for a nice landscape of wide neural networks.
	The most well known evidence is the recent results on ultra-wide networks (e.g. \cite{jacot2018neural,allen2018convergence, du2018gradientb, zou2018stochastic, zou2019improved}) which proved that gradient descent can converge to global minima. But the number of neurons required in all these works is at least $\Omega(N^6)$ where $ N $ is the number of samples, thus not a practical setting.
	
Another encouraging theoretical evidence is the result in a classical work \cite{yu1995local} which stated that ``no sub-optimal local-min exists'' for 1-hidden-layer over-parameterized networks with sigmoid activation. Formally, we 
state their main result here.
\begin{claim}[\cite{yu1995local} Theorem 3]
		\label{thm::Yu no local min}
		Consider the problem 
		$$  \min_{ v \in \mathbb{R}^{m},  W  \in \mathbb{R}^{m \times d_0 }} \sum_{n = 1}^N  \left[y^{(n)} -  v^\top \sigma\left( W \mathbf{x}^{(n)} \right) \right] ^2 , $$
		where $\sigma$ is the sigmoid function. 
		Assume the width $m \geq N$, and there is one index $ k $
		such that $x^{(n)}_{k} \neq 	x^{(n')}_{k}$, $\forall n\not = n'$. 
		Then every local minimum of the problem is a global minimum. 
	\end{claim}
	Claim \ref{thm::Yu no local min} is considered as one of the main results on neural network optimization in the survey \cite{bianchini1996optimal}. Besides the assumption of the width (which becomes reasonable nowadays), this result is quite clean. A few recent works have continued this line of research to study the global landscape of deep over-parameterized networks, and proved non-existence of sub-optimal valleys or basins (\cite{nguyen2018loss2,nguyen2019connected, li2018over}), which are weaker conclusions than ``no sub-optimal local-min''.
	It was not clear whether ``no sub-optimal local-min'' can be proved
	for deep over-parameterized networks, and \cite{nguyen2018loss2} conjectured
	that this is true under some mild conditions.
	
	Given the promising empirical and theoretical evidence  on the nice landscape of wide neural-nets, a final missing piece seems to be a proof of ``no sub-optimal local-min'' for deep over-parameterized neural-nets under suitable conditions.
	With such a result, the story of ``over-parameterization eliminates bad local-min'' would be verified on both empirical and theoretical sides.
	
	Note that there are some examples of sub-optimal local minima under various settings (reviewed detailedly in Section \ref{subsec: examples in literature}), but this line of works seems to be evolving independently of the line of positive results on over-parameterized nets since they often consider different settings. 
	Due to this discrepancy, some researchers challenge the story of ``no bad local-min'' while other researchers disregard this challenge. 
	More specifically, many positive results assume wide enough networks and smooth activations; thus, the existence of sub-optimal local-min for ReLU networks or for narrow networks could not dissuade readers of \cite{yu1995local} from pursuing the story of ``no bad local-min''. If someone could extend 
	Claim \ref{thm::Yu no local min} to deep over-parameterized networks with smooth activations, then in the debate of bad local minima, 
	the positive side would take an advantageous position. 
	
	
	
	It is somewhat strange, though, that no one has proved a direct extension of Claim \ref{thm::Yu no local min} to multi-layer networks yet. In this work, we provide a counter-example to Claim \ref{thm::Yu no local min}, showing that this result actually does not hold. More generally, we construct sub-optimal local minima to a broad class of networks including most settings considered in \cite{yu1995local,nguyen2018loss2,nguyen2019connected,li2018over}. Thus, despite how promising it looks, it is much more difficult than one thought, if not impossible, to rigorously prove ``no bad local-min'' for over-parameterized networks (without further assumptions). 
	
	\subsection{Our Contributions}
We consider a supervised learning problem where the prediction is parameterized by a multi-layer neural network.
We show that sub-optimal local minima are quite common in
neural networks with a large class of activations (including most commonly used activations).
More specifically, our contributions include (suppose $d_0$ is the input dimension and $ N $ is the number of samples):
	
	\begin{itemize}
		\item In the first result, we prove that for any activation function that has continuous and non-zero second-order derivatives in an arbitrarily small region (including common smooth activation functions such as sigmoid, tanh, swish, SoftPlus, etc.) and for generic input data where $d_0 < O(\sqrt{N})$, there exist output data such that the considered loss function contains sub-optimal local minima.
		This result holds for networks with any depth and width as long as the number of neurons of the last hidden layer is no less than two. 
		
		
      \item The first result implies that under the typical settings of over-parameterized networks (e.g. \cite{yu1995local,nguyen2018loss2,nguyen2019connected,li2018over}), it is impossible to prove ``no sub-optimal local-min'' without additional assumptions. As a special case, we provide a counter-example to the main result of \cite{yu1995local}.
		
		\item In the second result, we study a single hidden layer network with sigmoid activation, $d_0=1$ and any $N\geq 5$. We prove that if the width of the hidden layer is no less than $N$, for generic input data and a positive measure of output data,
		 the considered loss function contains sub-optimal local minima.
		
			\ifnonsmooth
		\item In the third result, we prove that for any activation function
		that is constant in a segment (including ReLU,
		leakly ReLU, ELU, SELU), for generic input and generic output data, the considered loss function contains sub-optimal local minima.
		This result covers common non-smooth activation
		function. 
		It includes as special cases the existence of sub-optimal local minima for ReLU activations, but generalizing
		 to a large set of non-smooth activation and multi-layer
		 neural networks. 
		\fi  
		
		\item  We present a novel general approach to construct sub-optimal
		local minima for neural networks. It is based on attributing the task of finding local minimum to solving a system of equations and inequalities. This is different from previous approaches of constructing local minima (e.g. utilizing dead neurons for ReLU activations). 
		
		
		\ifsmalldata 
		\item The existence of sub-optimal local-min is not universal, and	there is a phase transition on the number of data samples. 
		We show that as the number of data samples increases, there is a transition from no bad local-min to having bad local-min. 
		In particular, if there are only one or two data samples, the considered network has no bad local minima for a wide class of activation functions. However, with three or more data samples, bad local minima exist for almost all nonlinear activations.
		\fi 
		
	\end{itemize}
	
   In the first result, we construct sub-optimal local minima
       for ``certain $Y$'' where $Y$ is the output;
    our second results holds for ``a positive measure of $Y$''.
     Note that the existing works in  \cite{yu1995local,nguyen2018loss2,nguyen2019connected,li2018over}
     assume ``generic $X$ and any $Y$'', thus for a negative result
     we only need to consider
     ``a positive measure of $X$ and certain $Y$'',
     as done in the first result above. 
     In the second result, we construct sub-optimal local-min
     under a more general condition of ``generic $X$ and a positive measure of $Y$''
     for sigmoid activation. 
     The detailed conditions are summarized in Table \ref{table:our results}.

	\begin{table*}[htbp]
		\footnotesize
		\centering
		\begin{threeparttable}[b]
			\caption{Detailed Conditions of Our Results} 
			\label{table:our results}
			\begin{tabular}{p{25pt}p{20pt}p{20pt}p{53pt}p{40pt}p{30pt}p{65pt}} 
				\toprule
				Result & Width & Depth & $d_0$ and $N$  & Activation & Input Data & Output Data \\
				\midrule
			 Thm \ref{thm::deep-bad_localmin}  & any  & any & $d_0 < O(\sqrt{N})$ &   smooth   \footnotemark[1]
			 & generic   &  zero measure  \\
			  Thm \ref{thm::bad_localmin_sigmoid}  & any & $ 2 $  & $d_0=1$, any $N$  & sigmoid  & generic    &  positive measure \\
				\bottomrule
			\end{tabular}
			\begin{tablenotes}
				\item[1] Any  activation function that has continuous and non-zero second-order derivatives in an arbitrarily small region, including
				sigmoid, tanh, Swish, SoftPlus, etc.
			\end{tablenotes}
		\end{threeparttable}
	\end{table*}

	

	\subsection{Discussions}
	  \textbf{Benefit of overparameterization.}
	  	To our knowledge, our result is the first 
	  rigorous result on the existence of sub-optimal local-min that covers arbitrarily wide networks with popular smooth activations such as sigmoid
	  (for empirical risk minimization). 
		Our result does not necessarily mean over-parameterized networks are hard to train, but means that ``no sub-optimal local minima'' \textit{may be unable to explain} the benefit of over-parameterization.
		We believe that our negative result is important as it reveals our limited understanding of neural-net optimization, despite the optimism in the area that the optimization of over-parameterized networks is relatively well-understood.  This may help researchers  develop better explanations for the success	of neural networks.
		
		
    \textbf{Negative result for existence of local minima.}
  ``No sub-optimal local-min'' has been a popular characterization for non-convex
  matrix problems (e.g. \cite{sun2016guaranteed,bhojanapalli2016global,ge2016matrix}). 
 For instance, \cite{bhojanapalli2016global} explained the success
 of low-rank matrix sensing by the non-existence of sub-optimal local minima.
 However, recently \cite{NIPS2018_7802} pointed out that
 the assumption in \cite{bhojanapalli2016global} is too strong to be satisfied in practice, and sub-optimal local minima indeed exist for most problems. 
 It then argued that ``no sub-optimal local min'' \textit{may be unable 
 to explain} the success of gradient descent for non-convex
  matrix recovery,
and also pointed out that the true reason is still not clear. 
 This claim has a similar flavor to ours which highlights the limitation
 of using ``no sub-optimal local-min'' to understand deep learning.

	\textbf{Negative result for generalization bounds.}
	We noticed another recent result \cite{nagarajan2019uniform} with some similarity to our work in high level: it proved that uniform convergence \textit{may be unable to explain} the generalization error of neural networks. That result has attracted much attention, as it reveals the limitation of the existing approaches on analyzing generalization error. 
  

	


	\ifsmalldata
	
	\subsection{Phase Transition and Story Behind Our Findings}
	We notice that the two major results on the neural-net landscape
	can be illustrated by the simplest 1-neuron 1-data example
	and thus have no phase transition.
	The simplest 1-neuron 1-data linear network $F(v, w) = (1 - v w)^2 $ has no sub-optimal local minima, and it turns out 
	deep linear networks also have no sub-optimal local minima
	(\cite{kawaguchi2016deep}).
	The simplest 1-neuron 1-data non-linear network $F(v, w) = (1 - v \sigma(w) )^2 $ has no bad basin, and it turns out
	deep over-parameterized non-linear networks also have
	no bad basin (\cite{li2018over}).
	
	What about ``no sub-optimal local-min'' result for non-linear networks?
	We start from the very special case
	$ (1 - \sigma(w) )^2 $ (with the second layer weight $v$ fixed to $1$).
	We notice that if $\sigma$ has a ``bump'', i.e., not monotone, then 
	$ F( w) = (1 -  \sigma(w) )^2 $ has sub-optimal local-min.
	Such a counter-example can easily be fixed:  if $\sigma( t ) $ is strictly increasing on $t$ then $  (1 - \sigma(w) )^2 $ has no sub-optimal local-min. 
	As many practical activation functions are strictly increasing (e.g., sigmoid, ELU, SoftPlus), adding a minor assumption
	of strictly increasing seems reasonable. 
	
	We then checked the 2-layer case 
	$F(v, w) = (1 - v \sigma(w) )^2 $. 
	It is not straightforward to see what the landscape of this function
	is.  For special $\sigma $ (e.g. $\sigma(t) = t^2 $),
	sub-optimal local-min exists.
	But this counter-example is again quite brittle,
	as for $\sigma(t) = (t - 0.01)^2 $ or $\sigma(t) = t^2 + 0.01 $
	there is no sub-optimal local minima. 
	With some effort, we proved that for almost all
	$\sigma$ (except $\sigma$ that achieves zero value
	at its local-min or local-max), the function has no sub-optimal local-min.
	In particular, for strictly increasing activations, $(1 - v \sigma(w) )^2 $ has no sub-optimal local minima. 
	The exact conditions on the activations are not that important for our purpose, as long as it covers a broad range of activations,
	especially the set of strictly increasing functions.
	Since previous major results are extendable to more general cases,
	we initially made the following conjecture:
	
	\textbf{Conjecture}: for a large set of
	activation functions (a superset of the set of all strictly increasing smooth functions), over-parameterized networks have no sub-optimal local-min, for generic data.
	
	
	We already proved this conjecture for the case with one data point.
	We are able to prove it for two data points, i.e.,
	the function $ F(v, w) = (y_1 - v\sigma(w x_1))^2 + 
	(y_2 - v\sigma(w x_2 ))^2 $, giving us more confidence.
	
	Unfortunately, with $n = 3$ samples and $ m = 3 $ neurons,
	our proof no longer works. This failure of proof has led us to
	construct a counter-example to this conjecture, and later generalization of this counter-example to any number of neurons. 
	More specifically, our counter-examples hold for 
	almost all activation functions (in the sense
	that any continuous activation is arbitrarily close to 
	an activation in our class). 
	This is a strong negative answer to the conjecture,
	as disproving the conjecture only requires giving a counter-example
	for \textit{one} strictly increasing activation and our results imply
	that for almost all strictly increasing activation functions (and almost all non-increasing functions as well) the conjecture
	fails. 
	Finally, we note that it is \textit{impossible} to show
	``sub-optimal local minima exist for \textit{any} activations'' since linear activation leads to no sub-optimal local minima.
	Thus our result for almost all activations is
	in some sense ``tight''.

	\fi

	
	\subsection{Notation, Terminology and Organization}
	\textbf{Standard terminology}. The term ``generic'' is used in measure theory 
	and algebraic geometry. We adopt the notion in measure theory: a property P is said to hold for generic $ x \in \mathbb{R}^K $ if P holds for all but a zero-measure subset of $\mathbb{R}^K$. 
	For instance, a generic matrix is a full rank matrix (even though
	there are infinitely many non-full-rank matrices).
	 ``Sub-optimal local minimum (local-min)'' or ``bad local minimum (local-min)'' in this paper both refer to ``a local minimum that is not a global minimum''.

	\textbf{Informal terminology}.
	We use the term ``negative example'' in this paper to indicate
	``an example of a neural-net problem with sub-optimal local-min''
	unless otherwise specified. 
	We use the term ``positive result'' to
	indicate a result that shows that a certain bad point (not necessarily
	sub-optimal local-min) does not exist (e.g. \cite{li2018over} shows that bad basins do not exist), or global convergence under certain conditions (e.g. \cite{allen2018convergence}).
	
	
	\textbf{Notation}. Throughout this paper, we use $\mathbf{0}_{p\times q}$ and $\mathbf{1}_{p \times q}$ to represent $p\times q$ all-zero and all-one matrices, respectively. Similarly, $\mathbf{0}_{p}$ and $\mathbf{1}_{p}$ denote $p$-dimensional all-zero and all-one column vectors, respectively. The indexes $p$ and $q$ may be omitted if the dimension(s) can be easily inferred from the context. For a column vector $\mathbf{a} \in \mathbb{R}^d$, we denote its $i$-th entry by $a_i$. For a matrix $A \in \mathbb{R}^{d_1 \times d_2}$, the entry in the $i$-th row and the $j$-th column is denoted by $(A)_{i,j}$ or $a_{i,j}$. Furthermore, we use $A_{(i,:)}$ to represent a \textit{column vector} consisting of the entries in the $i$-th row of $A$, i.e,
	\begin{equation}
	A_{(i,:)} = [a_{i,1}, a_{i,2},\cdots, a_{i,d_2}]^\top \in \mathbb{R}^{d_2}.
	\end{equation}
	We use $\|\cdot\|_F$, $\langle \cdot, \cdot\rangle_F$ and $\circ$ to represent the Frobenius norm, the Frobenius inner product and the Hadamard product. We also use $\mathrm{row}(A)$ to represent the the linear space spanned by the rows of matrix $A$, and $B(x,r)$ to denote an open ball in an Euclidean space centered at $x$ with radius $r$.
	
	\textbf{Organization}. The paper is organized as follows. We first discuss some related works in Section \ref{sec::related-works}. Then we present our network model in Section \ref{sec::model}. In Section \ref{sec::main-result}, we present the main results and make some discussions. 
	The proof for the main result is presented in Section \ref{sec::proof_main}. We finally draw our conclusions in Section \ref{sec::conclusion}.

	\section{Related Works}\label{sec::related-works}
	
	\subsection{Examples of Sub-optimal Local Minima}\label{subsec: examples in literature}
	
	The question ``whether there exists a neural network problem
	with sub-optimal local-min'' is not interesting, 
	since it was known for a long time that the answer is ``yes''.
	One extremely simple example is $ \min_{w} (1 - \sigma(w) )^2 $
	 where $\sigma(t) = \max\{ t, 0 \} $ is the ReLU activation.
	To the best of our knowledge, the earliest construction of sub-optimal local-min dates back to 1980s \cite{sontag1989backpropagation}.
	 
	We argue that an interesting negative example
	should  allow at least a large width and smooth activations, and
	possibly  generic input data, for two reasons.
	First, these conditions are standard conditions used in a majority of positive results including the landscape analysis \cite{yu1995local,nguyen2018loss2,nguyen2019connected,li2018over}
and the convergence analysis \cite{jacot2018neural,allen2018convergence, du2018gradientb, zou2018stochastic, zou2019improved}.
To bridge the gap between positive results and negative examples,
we should impose the same conditions. 
	Second, the mathematical construction of sub-optimal local-min
	 without these restrictions is relatively well understood.  
	 We will explain each of the three requirements one by one.




	\textbf{Wide Network (Over-parameterized)}.
	Without the restriction on the width, a negative example is easy to construct.
	The classical work \cite{auer1996exponentially} presented a concrete counter-example where exponentially many sub-optimal local minima exist in a \textit{single-neuron} network with sigmoid activation. However, the counter-example was an unrealizable case (i.e. the network cannot fit data), and the authors proved that under the same setting but assuming
	realizable (the network can fit data), every local-min
	is  a global-min.
	These two results show that it is not that interesting to
	construct bad local-min for non-realizable case, and we shall
	 consider realizable case. 
	Note that a ``wide'' network with more neurons than the number of samples
	is a common setting used in recent works, and is a stronger condition
	than ``realizable'',
	thus it could make the construction of sub-optimal local-min harder. 
	Our work adopts a much more general setting: we only require the last hidden layer has at least two neurons, thus covering an arbitrarily wide network. 
	
	
	
	
	
	\textbf{Smooth Activations}.
	Without restriction on the activation function, a negative 
	example is easy to construct. 
		As mentioned earlier, it is very easy to 
	give a negative example for ReLU: $ \min_{w} (1 - \text{ReLU}(w) )^2 $.
	A few works showed that more complicated ReLU networks have sub-optimal local minima (e.g., \cite{swirszcz2016local}, \cite{zhou2017critical}, \cite{safran2017spurious},
	\cite{venturi2018spurious}, \cite{liang2018understanding}\footnote{\cite{safran2017spurious} and \cite{venturi2018spurious} both provided counter-examples when the objective function is the population risk, a different setting from the empirical risk minimization considered in this paper.}). 
	\cite{he2020piecewise} further demonstrated that networks with piecewise linear activation functions contain sub-optimal local minima.
  
	One intuition why ReLU or other piecewise linear activation functions can lead to bad local minima in general is that it can create flat regions (``dead regions'') where the empirical loss remains constant. 
	Intuitively, such flat regions may disappear if the activations are smooth; for instance, $ \min_{w} (1 - \sigma (w) )^2 $ has no sub-optimal local-min if $\sigma$ is strictly increasing. 
		Therefore, from both conceptual and mathematical perspectives, 
		an ideal counter-example should apply to smooth activations.
		A recent work \cite{yun2018small} constructs
	 sub-optimal local-min for certain smooth activations satisfying
	 a few conditions; however, they only construct such examples
	 for a network with two neurons and three specific data points, thus
	 not a wide network setting. 

		To our knowledge, there was no previous work that
	provides a negative example that considers both arbitrarily wide networks
	and smooth activations except \cite{venturi2018spurious} which studies population risk and special data distribution. 
	
	\textbf{Generic Input Data}.
    Besides the two requirements, we add one more requirement
     that the input data $X$ should be generic.
     This is because almost all recent positive results
     assume generic input data to avoid degenerate cases. 
     	An ideal negative example should apply to generic data, or at least a positive measure of data points.
     Our negative examples all apply to generic input data $X$. 
	 
	 For the convenience of readers, we summarize the existing examples and our examples in Table \ref{table:local-min-results}.
	 Note that there are a few factors such as depth and dimension
	 that are listed in Table \ref{table:our results} 
	 but not in Table \ref{table:local-min-results}.

	 \textbf{Depth, input dimension, sample size, output data distribution}.
	 These are the factors that may also affect
	 the existence of sub-optimal local minima.
	 Previous negative examples did not consider the effects of these factors;
	 for instance, all the results listed in Table 
	 \ref{table:local-min-results} only consider shallow networks,
	 while our Theorem \ref{thm::deep-bad_localmin}
	  considers arbitrarily deep networks. 
	 Due to the space limit, we did not list every advantage of our results 
	 over previous results, but only highlighted three factors in Table \ref{table:local-min-results}.
	  
	  \textbf{Limitation and open question}.
	 We try to construct negative examples that are as general as possible,
	 but it is very difficult to show that for every case
	 there is a sub-optimal local-min \footnote{This is a common situation 
	 for any negative theoretical result: NP-hardness does not mean
	 every problem instance is hard, but just that there exists
	 a hard instance. Of course one negative example may not
	 reflect the ``typical'' scenario, thus more general negative examples
	 are also useful. }.
	 There are two notable restrictions for Theorem \ref{thm::deep-bad_localmin}: first, $d_0 < O( \sqrt{N} ) $;
	  second, zero measure of $Y$ for given $X$.
	 In practical application of deep learning, the number of samples can be larger than $ d_0^2 $, so the first restriction is not strong (though
	 it is still an interesting question to remove this restriction).
	  In an effort to remove the second restriction, we established
	   Theorem \ref{thm::bad_localmin_sigmoid}, but
	   with relaxed conditions on other factors (e.g. only consider 2-layer
	   network and sigmoid activation).
	Rigorously speaking, we do not rule out the possibility
	 that there exists one combination of conditions (e.g. 5-layer sigmoid network
	 with randomly perturbed output $Y$) that can lead to absence of sub-optimal
	 local-min. The identification of such conditions
	 is left as future work.

	\begin{table*}[htbp]
		\small
		\centering
		\begin{threeparttable}[b]
			\caption{Summary of existing examples and our main result}
			\label{table:local-min-results}
			\begin{tabular}{llcll}
				\toprule
				Reference & Width  & Activation & Input Data \\
				\midrule
				Auer et al. & 1      & Continuous, bounded \footnotemark[2] & Positive measure \\
				Swirszcz et al. & 2 or 3  & Sigmoid, ReLU & Fixed\\
				Zhou et al. & 1     & ReLU  & Fixed\\
				Safran et al.\footnotemark[1] & 6 to 20  & ReLU  & Gaussian \\
				Venturi et al.\footnotemark[1] & Any  & $L^2(\mathbb{R}, e^{-x^2/2})$ & Adversarial\\
				Yun et al. & 2   & Small nonlinearity & Fixed \\
				 Theorem \ref{thm::deep-bad_localmin} & Any  & Generic smooth & Generic  \\
				\bottomrule
			\end{tabular}%
			\begin{tablenotes}
				\item[1] In these two examples, the objective function is the population risk, which is a different setting from the empirical risk minimization.
				\item[2] The actual requirement is that $l(\cdot, \sigma(\cdot))$ is continuous and bounded, where $l(\cdot, \cdot)$ and $\sigma(\cdot)$ are the loss function and the activation function, respectively.
			\end{tablenotes}
		\end{threeparttable}
	\end{table*}%

	\subsection{Other Related Works}\label{subsec:: other related-works}
	We discuss a few other related works in this section.
	
	\textbf{Absence of sub-optimal local minima under extra modifications}.
	There are some works that proved no sub-optimal local minima
	for non-linear networks, but their settings are quite
	special. \cite{liang2018understanding} assumes special training datasets (such as linearly separable data),
	\cite{soltanolkotabi2017learning} assumes quadratic activation,
	and  \cite{liang2018adding} assumes a special exponential activation
	and an extra regularizer. 
	These assumptions and modifications are not very practical
	and will not be the focus of this paper.
	We will study the rather general setting along the line of
	\cite{yu1995local}.
	
	\textbf{Numerical experiments on wide networks}.
	It is widely reported in numerical experiments that over-parameterized neural networks have nice landscape (see, e.g., \cite{livni2014computational,goodfellow2014qualitatively,lopez2018easing,geiger2018jamming,garipov2018loss}). 
	However, these works do not prove a rigorous theoretical result. 
	In particular, two interesting works \cite{garipov2018loss, draxler2018essentially}
	showed that in the current neural networks for image classification
	there are no barriers between different global minima.
	However, this only implies that there is no sub-optimal basin  and does not rule out the possiblity that sub-optimal local minima
	exist (which can be flat).

	\textbf{Positive results on linear and ultra-wide networks}.
	For deep linear networks,
a series of works \cite{kawaguchi2016deep, swirszcz2016local,lu2017depth, laurent2018deep, zhang2019depth} proved that no sub-optimal local minima
exist under very mild conditions. 
The assumption of linear activation does not match the practice where
 people always use non-linear activation functions. 
 
	For deep non-linear networks, \cite{jacot2018neural, allen2018convergence, du2018gradientb, zou2018stochastic, zou2019improved} proved that ``GD can converge to global minima'' for deep neural networks under the assumptions of a large number of neurons  and special initialization. 
	These results assume the iterates stay in a tiny 
	local region and are thus not far from training a linear network (called
	 ``lazy training'' by \cite{chizat2018global}).
	In addition, \cite{chizat2018global} provided numerical evidence that
	 practical training is different from lazy training,
	 thus these results do not necessarily explain the real training of neural networks.
	
	\textbf{Other theoretical results on landscape or convergence}.
	There have been many works on the landscape or convergence analysis of shallow neural-nets. 
	References \cite{freeman2016topology, soudry2017exponentially,haeffele2017global, ge2017learning, gao2018learning, feizi2017porcupine, panigrahy2017convergence} analyzed the global landscape
	of various shallow networks. 
	References
	\cite{laurent2017multilinear,tian2017analytical,soltanolkotabi2019theoretical,mei2018mean,brutzkus2017globally,zhong2017recovery,li2017convergence,brutzkus2017sgd,wang2018learning,du2018power,oymak2019towards,janzamin2015beating, mondelli2018connection} analyzed gradient descent for shallow networks.
	Along another line, \cite{mei2018mean,sirignano2018mean,chizat2018global,rotskoff2018neural} analyzed the limiting behavior of SGD when the number of neurons goes to infinity. 
	These results often have stronger assumptions or only study a subset
	of the optimization landscape, thus are not directly related
	to our question whether sub-optimal local minima exist under 
	the condition of Claim \ref{thm::Yu no local min}. 
	We refer the readers to \cite{sun2019optimization} for a more detailed
	overview of these results. 
	
	\section{Network Model}\label{sec::model}
	\subsection{Network Structure}
	Consider a fully connected neural network with $H$ hidden layers. Assume that the $h$-th hidden layer contains $d_h$ neurons for $1\leq h\leq H$, and the input and output layers contain $d_0$ and $d_{H+1}$ neurons, respectively. Given an input sample $\mathbf{x} \in \mathbb{R}^{d_0}$, the input to the $i$-th neuron of the $h$-th hidden layer (called ``pre-activation'' in deep learning area), denoted by $z_{h,i}$, is given by
	\begin{subequations}
		\begin{gather}
		z_{1, i}(\mathbf{x})= \sum_{j=1}^{d_0} w_{1, i, j}x_j+b_{1, i}, \quad 1\leq i\leq d_1 \\
		z_{h, i}(\mathbf{x})= \sum_{j=1}^{d_{h-1}}w_{h, i, j}z_{h-1, j}(\mathbf{x})+b_{h, i}, \quad 1\leq i\leq d_h, \quad 2\leq h \leq H
		\end{gather}
	\end{subequations}
	where $x_j$ is the $j$-th entry of the input data, $w_{h, i, j}$ is the weight from the $j$-th neuron of the $(h-1)$-th layer to the $i$-th neuron of the $h$-th layer, $b_{h, i}$ is the bias added to the $i$-th neuron of the $h$-th layer.
	
	Let $\sigma(\cdot)$ be the neuron activation function. Then the output of the $i$-th neuron of the $h$-th hidden layer (called ``post-activation''), denoted by $t_{h,i}$, is given by
	\begin{equation}
	t_{h,i}(\mathbf{x}) = \sigma\left( z_{h,i}(\mathbf{x}) \right), \quad 1 \leq i \leq d_h, \quad 1\leq h \leq H.
	\end{equation}
	Finally, the $i$-th output of the network, denoted by $t_{H+1, i}$, is given by
	\begin{equation}
	t_{H+1, i}(\mathbf{x})=\sum_{j=1}^{d_H}w_{H+1, i, j}t_{H, j}(\mathbf{x}), \quad 1 \leq i \leq d_{H+1}
	\end{equation}
	where $w_{H+1, i, j}$ is the weight to the output layer, defined similarly to that in the hidden layers.
	
	Then, we define $W_h \in \mathbb{R}^{d_h\times d_{h-1}}$ and $\mathbf{b}_h \in \mathbb{R}^{d_{h}}$ as the weight matrix and the bias vector from the $(h-1)$-th layer to the $h$-th layer. The entries of each matrix are given by
	\begin{equation}
	(W_h)_{i,j} = w_{h, i, j},\quad\!\! (\mathbf{b}_h)_i = b_{h,i}.
	\end{equation}

	\subsection{Training Data}
	Consider a training dataset consisting of $N$ samples. Noting that the input and output dimensions are $d_0$ and $d_{H+1}$ respectively, we denote the $n$-th sample by $\left(\mathbf{x}^{(n)}, \mathbf{y}^{(n)}\right)$, $n=1,\cdots, N$, where $\mathbf{x}^{(n)}\in \mathbb{R}^{d_0}, \mathbf{y}^{(n)} \in \mathbb{R}^{d_{H+1}}$ are the $n$-th input and output samples, respectively. We can rewrite all the samples in a matrix form, i.e.
	\begin{subequations}
		\begin{align}
		X \triangleq & \left[\mathbf{x}^{(1)}, \mathbf{x}^{(2)}, \cdots, \mathbf{x}^{(N)}\right]\in\mathbb{R}^{d_0 \times N} \\
		Y \triangleq & \left[\mathbf{y}^{(1)}, \mathbf{y}^{(2)}, \cdots, \mathbf{y}^{(N)} \right]\in\mathbb{R}^{d_{H+1}\times N}.
		\end{align}
	\end{subequations}
	
	With the input data given, we can represent the input and output of each hidden-layer neuron by
	\begin{subequations}
		\begin{gather}
		z_{h,i,n} = z_{h,i}\left(\mathbf{x}^{(n)}\right)\\
		t_{h,i,n} = t_{h,i}
		\left(\mathbf{x}^{(n)}\right)
		\end{gather}
	\end{subequations}
	for $h = 1,2,\cdots, H$, $i = 1,2,\cdots,d_h$, and $n = 1,2, \cdots, N$. Then, we define $Z_h \in \mathbb{R}^{d_h \times N}$ and $T_h \in \mathbb{R}^{d_h \times N}$ as the input and output matrix of the $h$-th layer with
	\begin{subequations}
		\begin{gather}
		(Z_{h})_{n,i} =  z_{h,i,n} \\
		(T_{h})_{n,i} =  t_{h,i,n}.
		\end{gather}
	\end{subequations}
	Similarly, we denote the output matrix by $T_{H+1}\in \mathbb{R}^{d_{H+1} \times N}$, where
	\begin{subequations}
		\begin{align}
		(T_{H+1})_{i,n} = t_{H+1,i}\left(\mathbf{x}^{(n)}\right)
		\end{align}
	\end{subequations}
	for $i = 1,2,\cdots, d_{H+1}$, $n = 1, 2,\cdots, N$.
	
	\subsection{Training Loss}
	Let $\Theta$ denote all the network weights, i.e.
	\begin{equation}
	\Theta=(W_1, \mathbf{b}_1, W_2, \mathbf{b}_2 \cdots, W_H, \mathbf{b}_H, W_{H+1})
	\end{equation}
	In this paper, we use the quadratic loss function to characterize the training error. That is, given the training dataset $(X,Y)$, the empirical loss is given by
	\begin{equation}
	E(\Theta) = ||Y - {T}_{H+1}(\Theta)||^2_F.
	\end{equation}
	Here we treat the network output $T_{H+1}$ as a function of the network weights. Then, the training problem of the considered network is to find $\Theta$ to minimize the empirical loss $E$.

	\section{Main Theorems and Discussions}\label{sec::main-result}
	\subsection{General Example of Bad Local Minima}\label{subsec::bad-local_min}
	In this subsection, we present our main result of bad local minima. To this end, we first specify the assumptions on the data samples and the activation function.
	
	\begin{assumption} \label{ass::inputdata}
		\quad
		\begin{enumerate}[label=(\alph*)]
			\item \label{ass::inputdata_a}
			The input dimension $d_0$ satisfies $d_0^2/2+3d_0/2 < N$.
			\item \label{ass::inputdata_b} All the vectors in the set
			\begin{multline}
			\mathcal{X} \triangleq  \{\mathbf{1}_N\} \cup \{X(i,:) | 1 \leq i \leq d_0\} \cup  
			\{X(i,:)\circ X(j,:)| 1 \leq i ,j \leq d_0\}.
			\end{multline}
		 are linearly independent, where ``$\circ$'' represents Hadamard product. 
	
		\end{enumerate}
	\end{assumption}
	
	The set $\mathcal{X} \subseteq \mathbb{R}^N$ has in total $d_0^2/2 +3 d_0/2+1 $ vectors, including $\mathbf{1}_N$, all the rows of $X$, and the Hadamard product between any two rows of $X$.
	Notice that if Assumption \ref{ass::inputdata}\ref{ass::inputdata_a} is satisfied, Assumption \ref{ass::inputdata}\ref{ass::inputdata_b} holds for generic input data. That is, provided \ref{ass::inputdata}\ref{ass::inputdata_a}, the input data violating Assumption \ref{ass::inputdata}\ref{ass::inputdata_b} only constitutes a zero-measure set in $\mathbb{R}^{d_0\times N}$. Further, if Assumption \ref{ass::inputdata}\ref{ass::inputdata_a} holds, Assumption \ref{ass::inputdata}\ref{ass::inputdata_b} can be always achieved if we allow an arbitrarily small perturbation on the input data.
	
	\begin{assumption}
		\label{ass::non-linear-activation}
		There exists $a\in \mathbb{R}$ and $\delta>0$ such that
		\begin{enumerate}[label=(\alph*)]
			\item $\sigma$ is twice differentiable on $[a-\delta, a+\delta]$.
			\item $\sigma(a), \sigma'(a), \sigma''(a) \not = 0$.
		\end{enumerate}
	\end{assumption}
	Assumption \ref{ass::non-linear-activation} is very mild as it only requires the activation function to have continuous and non-zero derivatives up to the second order in an arbitrarily small region. It holds for many widely used activations such as ELU, sigmoid, softplus, Swish, and so on. In fact, the function class specified by Assumption \ref{ass::non-linear-activation} is a dense set in the space of continuous functions in the sense of uniform convergence. Therefore, by Assumption \ref{ass::non-linear-activation} we specify a ``generic" class of activation functions.
	
	\begin{theorem}
		\label{thm::deep-bad_localmin}
		Consider a multi-layer neural network with input data $X\in \mathbb{R}^{d_0\times N}$ and $N\geq 3$. Suppose that Assumption \ref{ass::inputdata} and \ref{ass::non-linear-activation} hold, and the last hidden layer has at least two neurons, i.e., $d_{H} \geq 2$. Then there exists an output matrix $Y\in \mathbb{R}^{d_{H+1}\times N}$ such that the empirical loss $E$ has a sub-optimal local minimum.
	\end{theorem}
	This is a rather general counter-example. Theorem \ref{thm::deep-bad_localmin} states that for generic input data and non-linear activation functions, there exists output data
	such that the loss has a sub-optimal local minimum regardless of the depth and width (besides the requirement of $d_H \geq 2$). 
	\ifnonsmooth
	The activation considered in Assumption \ref{ass::non-linear-activation} is required to has non-zero derivatives up to the second order in a small region. This implies that the activation should contain a tiny ``smooth curve''. In what follows, we consider networks with activation functions that contain a tiny ``segment''.
	
	\begin{assumption}
		\label{ass::partial-linear-activation}
		\quad
		\begin{enumerate}[label=(\alph*)]
			\item \label{ass::partial-linear-activation_a}There exists $a\in \mathbb{R}$ and $\delta>0$, such that $\sigma$ is linear in $(a-\delta, a+\delta)$.
			\item \label{ass::partial-linear-activation_b}Each hidden layer is wider than the input layer, i.e., $d_h > d_0$ for $h=1, 2, \cdots, H$. 
			\item \label{ass::partial-linear-activation_c}The training data satisfies $\mathrm{rank}\left(\left[X^\top,\mathbf{1}_N, Y^\top\right]\right) > \mathrm{rank}\left(\left[X^\top, \mathbf{1}_N\right]\right) $.
		\end{enumerate}
	\end{assumption}
	Assumption \ref{ass::partial-linear-activation}\ref{ass::partial-linear-activation_a} requires the activation to be at least ``partially linear", i.e., linear in an arbitrarily small interval, which holds for many activation functions, such as piecewise-linear activations  ReLU and leaky ReLU, and
	partially-linear-partially-nonlinear activations ELU and SeLU. Assumption \ref{ass::partial-linear-activation}\ref{ass::partial-linear-activation_b} requires each hidden layer to be wider than the input layer.
	Assumption \ref{ass::partial-linear-activation}\ref{ass::partial-linear-activation_c} holds for generic training data $(X, Y)$ satisfying $N > d_0+ 1$.
	
	\begin{theorem}
		\label{thm::partiallinear-bad-localmin}
		Consider a fully-connected deep neural network with data samples $X\in \mathbb{R}^{d_0\times N}, Y\in \mathbb{R}^{d_{H+1}\times N}$. Suppose that Assumption \ref{ass::partial-linear-activation} holds. Then the empirical loss $E$ has local minimum $\Theta$ with $E(\Theta) > 0$.
	\end{theorem}
	
	Theorem \ref{thm::partiallinear-bad-localmin} gives the condition where the network with ``partially linear" activations has a local minimum with non-zero training error. Compared to Theorem \ref{thm::deep-bad_localmin}, Theorem \ref{thm::partiallinear-bad-localmin} does not require a specific choice of the output data $Y$, but holds for generic data $(X, Y)$. However, the requirement that the network is wide in every layer is stronger than the requirement of Theorem \ref{thm::deep-bad_localmin}.
	Specifically, if the network is realizable, 
	the considered network has a bad local minimum:
	
	\begin{corollary}
		\label{cor::partial-linear-activation}
		Consider a fully-connected deep neural network with data samples $X\in \mathbb{R}^{d_0\times N}, Y\in \mathbb{R}^{d_{H+1}\times N}$. Suppose that Assumption \ref{ass::partial-linear-activation} holds and that the network is realizable for $(X, Y)$. Then the empirical loss $E(\Theta)$ has a sub-optimal local minimum.
	\end{corollary}
	
	Remark: After the first version, the authors of \cite{christof2020stability} point out that their technique in \cite{christof2020stability} can potentially help to improve Corollary \ref{cor::partial-linear-activation} by relaxing the realizability assumption and the width requirement. We leave this for future work.
	\fi

	One limitation of Theorem \ref{thm::deep-bad_localmin} is that the output data $Y$ need to be chosen based on the input data, thus $Y$ lies in a zero measure set for a given $X$.
   We have removed the restriction of zero-measure $X$ compared to many earlier works, but have added the restriction of zero-measure $Y$ in Theorem \ref{thm::deep-bad_localmin}. 
   One many wonder whether this is a reasonable setting. 
   
In our opinion, this setting is reasonable in the context of the literature. Existing positive results on wide networks \cite{yu1995local} \cite{li2018over}
  \cite{nguyen2018loss2} all require generic input $X$ but make no restriction on $Y$.
Therefore, for a negative result we only need to prove the existence of sub-optimal local-min for certain $ Y$: this would mean that it is impossible to prove the absence of sub-optimal local-min under the common assumption of ``generic $X$ and \textit{any} $Y$''. Thus our result is an important first step to understand the existence of sub-optimal local minima. As an immediate corollary, the main result of \cite{yu1995local} does not hold.
      
      	\begin{corollary}
   Claim \ref{thm::Yu no local min}, which is Theorem 3 of \cite{yu1995local},
   does not hold. There is a counter-example to it.  
		\end{corollary}
      
Of course one could always ask whether a stronger negative example can be constructed, and we do go one step further. The restriction on $Y$ can be relaxed: the Theorem below shows that sub-optimal local minima can exist for a positive measure of $Y$ (while we still allows generic $X$) for sigmoid activation.
   
     
	
	\begin{theorem}
		\label{thm::bad_localmin_sigmoid}
		Consider a $1$-hidden-layer neural network with $d_0 = d_2 = 1$ and $d_1 \geq N > 5$. Suppose that the activation is sigmoid function. Then for generic $X \in \mathbb{R}^{1\times N}$,
		 there exists a positive measure of $Y \in \mathbb{R}^{1 \times N}  $
		 such that the network has a sub-optimal local minimum.
	\end{theorem}
	
	This result implies that it is impossible to prove the absence of sub-optimal local-min even under a more general assumption of ``generic $X$ and \textit{generic} $Y$''. Note that we are not aware of any existing work that could utilize the condition of ``generic $Y$'' to prove a positive result, so our Theoerm \ref{thm::bad_localmin_sigmoid} actually disproves a hypothesis that was not considered before: if relaxing the condition of $Y$ from ``any $Y$'' to ``generic $Y$'', then every local-min is a global-min for wide networks.
In Table \ref{table:positive and negative},
we summarize our results in terms of the condition on $Y$,
and the corresponding positive claims that are disproved by
our results. 


	\begin{table*}[htbp]
		\footnotesize
		\centering
		\begin{threeparttable}[b]
			\caption{Positive and Negative Results} 
			\label{table:positive and negative}
			\begin{tabular}{p{25pt}p{140pt}p{140pt} } 
				\toprule
			          &  Positive claim: no bad local-min  &  Negative claim: bad local-min exists  \\
				\midrule
			   &  any $Y$ (\cite{yu1995local}, incorrect result)   &   $\exists$ $Y$  (Thm \ref{thm::deep-bad_localmin}) \\
  	\midrule  
			    & generic $Y$ (no result) &  $\exists$ positive measure of $Y$ (Thm \ref{thm::bad_localmin_sigmoid}) \\
				\bottomrule
			\end{tabular}
			\begin{tablenotes}
				\item Remark: Assuming generic input data $X$, 1-hidden-layer neural-net
				with enough neurons, and sigmoid activation.
			\end{tablenotes}
		\end{threeparttable}
	\end{table*}

Finally, we emphasize that our contribution is not merely finding sub-optimal local minima for one setting, but also introducing a general approach for constructing sub-optimal local-min that may apply to many settings. It is by this general approach that we are able to extends our result from a specific  $Y$ to a positive measure of $Y$. 
	
	
	\ifsmalldata
	
	\subsection{No Bad Local-Min for Small Data Set}\label{subsec::no-bad-local_min}
	The understanding of local minima for neural networks is divided. 
	On one hand, many researchers thought over-parameterization eliminates bad local minima
	and thus the results of this paper is a bit surprising.
	On the other hand, experts may think the existence of bad local-min is not surprising
	since symmetry can easily cause bad local minima. More specifically, a common intuition is that if there are two distinct global minima with barriers in between, then bad local minima can arise on the paths connecting these two global minima. However, this intuition is not rigorous, since it is possible that all points between the two global minima are saddle points. For instance, the function $ F(v, w) = (1 - v w)^2  $ contains two branches of global minima in the positive orthant and the negative orthant respectively, but on the paths connecting the two branches, there are no other local minma but only saddle points.
	
	In this subsection, we rigorously prove that for a 1-hidden-layer network, if the number of data samples is no more than $2$, then there is no bad local-min for a large class of activations (though not a dense set of continuous functions for the case with $2$ data samples). Recall that in Theorem \ref{thm::deep-bad_localmin} the number of data samples is at least $3$. This reveals an interesting phenomenon that the size of training data will also affect the existence of local minimum.
	First, if the network has one data sample and one neuron, we gives a sufficient and necessary condition for the existence of sub-optimal local minima. This result shows that networks with one data sample have no bad local minima for almost all continuous activations.
	
	\begin{assumption}
		\label{ass::activation2}
		The activation function $\sigma$ is continuous. Further, for any $t\in\mathbb{R}$, if $\sigma(t) = 0$, $t$ is not a local minimum or local maximum of $\sigma$.
	\end{assumption}
	Assumption \ref{ass::activation2} identifies a class of functions that have no local minimum or maximum with zero function values. They constitute a dense set in the space of continuous functions.
	
	
	\begin{theorem}
		\label{thm::bad_localmin_1neuron}
		Consider a $1$-hidden-layer neural network with $d_1 = N = 1$ and data sample $x,y \in\mathbb{R}$ with $y\not =0$. Then the network is realizable, and the empirical loss $E(\Theta)$ has no bad local minima if and only if Assumption \ref{ass::activation2} holds.
	\end{theorem}
	
	Next, for networks with 2 data samples and two hidden neurons, we also establish a theorem that guarantees the non-existence of sub-optimal local minima.
	
	\begin{assumption}
		\label{ass::activation3}
		The activation function $\sigma$ is analytic and satisfies the following conditions:
		\begin{enumerate}[label=(\alph*)]
			\item $\sigma(0) \not = 0$; 
			\item $\sigma'(t) \not = 0$,\quad\!\! $\forall t \in \mathbb{R}$.
		\end{enumerate}
	\end{assumption}
	Assumption \ref{ass::activation3} holds for a wide class of strictly increasing/decreasing analytic functions, e.g., exponential functions, but these functions are not dense in the space of continuous functions.
	\begin{theorem}
		\label{thm::bad_localmin_2neuron}
		Consider a $1$-hidden-layer neural network with $d_1 = N = 2$ and data samples $x^{(1)}, x^{(2)}, y^{(1)}, y^{(2)}\in\mathbb{R}$. Suppose that $x^{(1)}\neq x^{(2)}$ and that Assumption \ref{ass::activation3} holds. Then the network is realizable, and the empirical loss $E(\Theta)$ has no sub-optimal local minima. 
	\end{theorem}
	
	Theorem \ref{thm::bad_localmin_1neuron} and \ref{thm::bad_localmin_2neuron} only consider the case of $d_1 = N \leq 2$. However, the conclusions of no bad local minima directly carry over to the case with $d_1 > N  $ and $N \leq 2$. The reason is that if we have $d_1 > N = 2$ or $d_1 > N = 1 $, by Theorem \ref{thm::bad_localmin_1neuron} or \ref{thm::bad_localmin_2neuron}, any sub-network with exactly $N$ hidden-layer neurons is realizable and has no bad local minima. Then, from any sub-optimal point of the original network, we can find a strictly decreasing path to the global minimum by only optimizing the weights of a sub-network. Therefore, the original network also admits no bad local minima.
	
	Our results on small dataset are somewhat counter-intuitive. In general, to determine whether sub-optimal local-min exists is a challenging task for networks of practical sizes. A natural idea is to begin with a simplest toy model, say, networks with one or two data samples, and attempt to extend the result to the general case. 
	
	Now, we see that if the training set has only one or two data samples, and the activation meets some special requirements, over-parameterized networks have no bad local minima. This result is quite positive, echoing with other positive results on over-parameterized neural networks \cite{yu1995local, venturi2018spurious,nguyen2018loss2,nguyen2019connected}. Then, a direct conjecture is that, as the size of training set grows, the network still contains no bad local minima if appropriate conditions are posted on the activation function. However, our main result shows that this is not true. In fact, once the size of dataset exceeds two, bad local minima exist for almost all networks regardless of the width and depth. It turns out that the results on toy models do not reveal the true landscape property of general models, and even convey misleading information.
	
	\fi 

\section{Proof Overview}
	
\subsection{Proof Overview of Theorem \ref{thm::deep-bad_localmin}}
To prove a point is a local minimum, a natural method is to verify the first-order and second-order conditions, i.e., to show that the gradient is zero and the Hessian is positive definite.
However, computing the Hessian of a multi-layer neural network is extremely difficult due to the huge number of network parameters and the parameter coupling between layers.

In this paper, we verify the local minimum directly by its definition, i.e., to show that there exists a neighborhood such that any perturbation within this neighborhood can not decrease the loss value. 
Specifically, let $Y$, $\hat{Y}$ and $\hat{Y}'$ be the ground-truth output (part of the training data), the network output at a certain weight configuration, and the network output at any perturbed weight configuration in a sufficiently small neighbourhood, respectively. To show that the considered weight configuration is indeed a local minimum, it suffices to show that
\begin{equation}
\label{eq::overview1}
    0\leq\|Y-\hat{Y}'\|_F^2-\|Y-\hat{Y}\|_F^2=\|\hat{Y}'-\hat{Y}\|_F^2+2\langle \hat{Y}-Y, \hat{Y}'-\hat{Y}\rangle_F.
\end{equation}
Moreover, it suffices to show that \begin{equation}
    \label{eq::half_space}
    \langle \hat{Y}-Y, \hat{Y}'-\hat{Y}\rangle_F\geq0.
\end{equation}
This inequality means that for any weight perturbation, the resulting output $\hat{Y}'$ can only lie in a half-space. In that sense, to show the existence of sub-optimal local minimum, a key difficulty is to control all the potential ``perturbation directions", i.e., to find a weight configuration such that any perturbed output is restricted in a desired half-space. In the proof of Theorem \ref{thm::deep-bad_localmin}, we perform two important tricks, detailed as follows.

First, we carefully design the network weights such that the perturbation directions are the same for each hidden layer and only depend on the input data. Then, for each hidden layer we characterize a quadratic approximation of the potential perturbed direction, obtained simply by Taylor expansion to the second order. Specifically, for the output of each hidden layer, the perturbation direction can be represented by the span of the ``zero-order direction" $\{\mathbf{1}_N\}$, the ``first-order directions"
\begin{equation}
    \mathcal{X}_{1} \triangleq \left\{X_{(1,:)}, X_{(2,:)}, \cdots, X_{(d_0,:)}\right\},
\end{equation}
the ``second-order directions"
\begin{equation}
    \mathcal{X}_{2} \triangleq \left\{ X_{(i,:)}\circ X_{(j,:)} |  1 \leq i,j\leq d_0 \right\},
\end{equation}
and some infinitesimal directions which can be ignored if the considered neighbourhood is sufficiently small. Moreover, for some of the second-order directions, the possible coefficients are restricted to be all-positive or all-negative. In this way, we avoid the potential perturbation directions from expanding across different layers, so that our setting allows an arbitrary depth.

Second, we pick appropriate output data $Y$, such that each row of $\hat{Y}-Y$ is orthogonal to the zero-order, the first-order, and a portion of the second-order perturbation directions. For the remaining second-order directions with restricted coefficients, we require them to lie in the same half space characterizd by \eqref{eq::half_space}. It turns out that finding such $Y$ is equivalent to solving an equation-inequality system (specified by \eqref{eq::proof_thm1_Y_construct}). As such, proving the existence of the desired $Y$ boils down to analyzing the solvability of a linear system. 

So far, we have constructed a local minimum, but not necessarily a sub-optimal
one. At the final step, we construct a new weight configuration with a lower objective
value, implying that the constructed local minimum is indeed sub-optimal.

Remark: The final step of proving sub-optimality can be done in a different way. Suppose that the network is wide enough, then the network is realizable, meaning that for generic training data, the global minimum achieves zero empirical loss. Since the constructed local-min has a non-zero empirical loss, it is sub-optimal. In the first version of this paper, we use this idea to establish our theorem, which requires the network to be wide enough. But in this version, we directly construct a point with smaller loss at the final step. This weaken the condition on the network width. Now we only need two neurons in the last hidden layer.

\subsection{Proof Overview of Theorem \ref{thm::bad_localmin_sigmoid}}
To prove Theorem \ref{thm::bad_localmin_sigmoid}, we must show that the considered network has a local minimum for any output data $Y$ in a positive-measure subset of $\mathbb{R}^{d_y\times N}$. However, this is challenging even for the 1-hidden-layer network. The reason is that the local minimum we constructed by the technique of Theorem \ref{thm::deep-bad_localmin} is non-strict, and non-strict local minimum may disappear if we slightly perturb the input or output data (which leads to a small modification of the
optimization landscape). 

To tackle this problem, we map the considered model to a single-neuron network, which admits a strict local minimum. Specifically, our proof consists of four stages. First, we show that the considered network has a sub-optimal local minimum for generic $X$ and a specific $Y$. This stage inherits the idea of Theorem \ref{thm::deep-bad_localmin}'s proof. Second, we show that the constructed local minimum in the first stage corresponds to a strict local minimum of a 1-neuron sigmoid network. This is done by ``merging" all the hidden neurons in the wide network into a single neuron. Third, we consider an arbitrary perturbation on the training data in a sufficiently small region. Since a strict local minimum persists after a sufficiently small perturbation on the landscape, the ``perturbed" 1-neuron network also has a local minimum that is close to the original one. Finally, we prove the existence of sub-optimal local minimum for the wide network with the perturbed data. This local minimum, correspondingly, is produced by ``splitting" the hidden neuron of the perturbed 1-neuron network. In this way, we prove the existence of local minimum for a positive measure of $Y$.

\section{Proof of Theorem \ref{thm::deep-bad_localmin}}\label{sec::proof_main}
	\subsection{Preliminaries}
	For convenience, we first introduce the following notations. For $1 \leq h_1 \leq h_2 \leq H$, let 
	\begin{equation}
	\Theta_{[h_1:h_2]} = (W_{h_1}, \mathbf{b}_{h_1},W_{h_1+1}, \mathbf{b}_{h_1+1},\cdots,W_{h_2}, \mathbf{b}_{h_2} )
	\end{equation}
	be the weights from the $h_1$-th layer to the $h_2$-th layer and
	\begin{equation}
	\Theta_{[h_1:(H+1)]} = (W_{h_1}, \mathbf{b}_{h_1},W_{h_1+1}, \mathbf{b}_{h_1+1},\cdots,W_{H}, \mathbf{b}_{H}, W_{H+1} )
	\end{equation}
	be the weights from the $h_1$-th layer to the $(H+1)$-th layer. Then for the $i$-th neuron in the $h$-th hidden layer, the input and output are functions of $\Theta_{[1:h]}$, written as $Z_{h}\left(\Theta_{[1:h]}\right)$ and $T_{h}\left(\Theta_{[1:h]}\right)$, respectively.
	
	Before we start the construction, we discuss some requirements that the desired local-minimum should satisfy. For the first hidden layer, we set
	\begin{subequations}
		\label{eq::proof_thm1_localmin_construct}
		\begin{gather}
		\label{eq::proof_thm1_localmin_construct1}
	 W_1 = \mathbf{0}\\
	 	\label{eq::proof_thm1_localmin_construct2}
	\mathbf{b}_1 = a \cdot \mathbf{1}_{d_1}.
	\end{gather}
	For the other hidden layers, we require the weights to be positive, i.e.,
	\begin{equation}
	 \label{eq::proof_thm1_localmin_construct3}
	 w_{h,i,j} > 0 
	 \end{equation}
	 and the biases to satisfy
	 \begin{gather}
	\label{eq::proof_thm1_localmin_construct4}
	 b_{h,i} = a - \sigma(a)\sum^{d_{h-1}}_{j = 1} w_{h,i,j}
	\end{gather}
	for $2 \leq h \leq H$, $ 1\leq i \leq d_h$, $1 \leq j \leq d_{h-1}$.
	For the output layer, we set
	\begin{equation}
	\label{eq::proof_thm1_localmin_construct5}
	    W_{H+1} = \frac{1}{d_{H}} \cdot \mathbf{1}_{d_{h+1}\times d_{h}}. \footnote{In fact, if replace the constructed $W_{H+1}$ by any matrix with all positive entires, our proof is still valid.}
	\end{equation}
	\end{subequations}
    Noting that the weights to the first hidden layer are zero, with biases in \eqref{eq::proof_thm1_localmin_construct2} and \eqref{eq::proof_thm1_localmin_construct4}, the input to every hidden-layer neuron equals to $a$, regardless of the input data sample, i.e.,
	\begin{equation}
		z_{h,i,n} = a, \quad t_{h,i,n} = \sigma(a), \quad h = 1,2,\cdots, H, \quad \forall i,n.
	\end{equation}
	With \eqref{eq::proof_thm1_localmin_construct5}, the network output is given by
	\begin{equation}
	T_{H+1} = \frac{1}{d_{H}}\cdot \mathbf{1}_{d_{H+1}\times d_H} \cdot  \sigma\left(a\cdot\mathbf{1}_{d_H \times N}\right) = \sigma(a)\cdot \mathbf{1}_{d_{H+1}\times N}.
	\end{equation}

	Subsequently, we construct the output data $Y$. Given the weights satisfying \eqref{eq::proof_thm1_localmin_construct}, denote
	\begin{equation}
	    \Delta Y \triangleq T_{H+1}(\Theta) - Y = \sigma(a)\cdot \mathbf{1}_{d_{H+1} \times N}  - Y.
	\end{equation}
	We would like each row of $\Delta Y$ to satisfy
	\begin{subequations}
		\label{eq::construct_Y}
		\begin{gather}
		\label{eq::inner1}
		\left\langle \Delta Y_{(i,:)}, \mathbf{1}\right\rangle = 0 \\ 
		\label{eq::innerX}
		\left\langle \Delta Y_{(i,:)}, X_{(j,:)}\right\rangle = 0 \\
		\label{eq::innerXXdiff}
		\left\langle \Delta Y_{(i,:)}, X_{(j,:)} \circ X_{(j',:)}\right\rangle = 0 \\
		\label{eq::innerXXsame}
		[\sigma'(a)]^{H-1}\sigma''(a)\left\langle \Delta Y_{(i,:)}, X_{(j,:)} \circ X_{(j,:)}\right\rangle  > 0
		\end{gather}
	\end{subequations}
	For all $i = 1,2,\cdots, d_{H+1}$ and $j,j' = 1,2,\cdots, d_0$ with $j \not= j'$. To show the existence of such $Y$, we present the following lemma.
	\begin{lemma}
	\label{lemma::dim_requirement}
		Consider arbitrary two integers $L_1$ and $L_2$ with $L_1 \leq L_2$. Suppose that we have a set of linearly independent vectors $\{\mathbf{a}_1, \mathbf{a}_2, \cdots, \mathbf{a}_{L_1} \} \subseteq \mathbb{R}^{L_2}$. Then for any $1 \leq d < L_1$ we can find another set of linearly independent vectors $\{\mathbf{u}_1, \mathbf{u}_2, \cdots, \mathbf{u}_{L_2 - d}\} \subseteq \mathbb{R}^{L_2}$, such that
		\begin{subequations}
			\label{eq::lemma4_origin}
			\begin{align}
			\langle\mathbf{u}_l, \mathbf{a}_j \rangle&= 0, \quad \!\! 1\leq j \leq d\\
			\langle \mathbf{u}_l, \mathbf{a}_j \rangle &> 0, \quad \!\! d< j \leq L_1
			\end{align}		
		\end{subequations}
		for any $1\leq l \leq L_2-d$.
	\end{lemma}
	From Assumption \ref{ass::inputdata}, the vectors in $\mathcal{X}$ are linearly independent. Define a subset of $\mathcal{X}$ as
	\begin{equation}
	\mathcal{X}_0 = \left\{X_{(j,:)} \circ X_{(j,:)} | j = 1, 2, \cdots, d_0\right\} \subseteq \mathcal{X}
	\end{equation}
	consisting of $d_0$ vectors in $\mathcal{X}$. For simplicity, we denote
	\begin{equation}
	    L = |\mathcal{X}| - |\mathcal{X}_0| = N - d_0^2/2 - d_0/2-1.
	\end{equation}
	From Assumption \ref{ass::inputdata}\ref{ass::inputdata_a}, we have $L > d_0 -1 \geq 0$. Then, by Lemma \ref{lemma::dim_requirement}, we can find $L$ linearly independent $N$-dimensional vectors $\{\mathbf{u}_1, \mathbf{u}_2, \cdots, \mathbf{u}_{L}\}$, such that each $\mathbf{u}_l$ has positive inner product with vectors in $\mathcal{X}_0$, and is orthogonal with all the vectors in $\mathcal{X}\backslash \mathcal{X}_0$. That is,
	\begin{subequations}
		\label{eq::construct_u}
		\begin{gather}
		\left\langle \mathbf{u}_l, \mathbf{1} \right\rangle= 0 \\
		\left\langle \mathbf{u}_l,  X_{(j,:)} \right\rangle= 0\\
		\left\langle \mathbf{u}_l, X_{(j,:)}\circ X_{(j',:)} \right\rangle= 0\\
		\left\langle \mathbf{u}_l, X_{(j,:)} \circ X_{(j,:)} \right\rangle> 0
		\end{gather}
	\end{subequations}
	for all $1 \leq l \leq L$ and $1 \leq j,j' \leq N$ with $j \not = j'$. We construct $Y$ as
	\begin{equation}
	\label{eq::proof_thm1_Y_construct}
	Y_{(i,:)} = \sigma(a)\cdot \mathbf{1}_{N} - \sum^{L}_{l = 1} \alpha_{i,l} \mathbf{u}_{l}
	\end{equation}
	for any $1\leq i \leq d_{H+1}$, where each $\alpha_{i,l}$ has the same sign with $[\sigma'(a)]^{H-1}\sigma''(a)$. Then
	\begin{equation}
	\Delta Y_{(i,:)} = \sigma(a)\cdot \mathbf{1}_{N} - Y_{(i,:)} = \sum^{L}_{l = 1} \alpha_{i,l} \mathbf{u}_{l}.
	\end{equation}
	It can be readily verified that such $\Delta Y$ satisfies \eqref{eq::construct_Y}.
	
	Given the constructed $Y$, in what follows, we show that there exists $\Theta$ satisfying \eqref{eq::proof_thm1_localmin_construct}, such that $\Theta$ is a sub-optimal local minimum of the empirical loss $E$.
	
	\subsection{Existence of Local Minimum}
	Consider a weight configuration $\Theta$ satisfying \eqref{eq::proof_thm1_localmin_construct} and a small perturbation $\Theta'$ around it. Based on $\Theta$ and $\Theta'$, we construct an ``intermediate" weight configuration as
	\begin{equation}
	\label{eq::tilde_W}
	\tilde{\Theta}' \triangleq \left(W_1, \mathbf{b}'_1, W'_2, \mathbf{b}'_2, \cdots, W'_H, \mathbf{b}'_H, W'_{H+1}\right)
	\end{equation}
	where the weights to the first hidden layer are picked from $\Theta$, while the biases to the first hidden layer and the remaining weights and biases are all from $\Theta'$. Clearly, $\tilde{\Theta}'$ is ``closer" to $\Theta$, i.e.,
	\begin{equation}
		\|\tilde{\Theta}' - \Theta\| \leq \|\Theta' - \Theta\|.
	\end{equation}
	For this intermediate weight configuration $\tilde{\Theta}'$, we have the following lemma.
	\begin{lemma}
		\label{lemma::perturb_direction}
		Consider a fully-connected deep neural network. Suppose that Assumption 1 and 2 hold. Further, the training data $Y$ satisfies \eqref{eq::proof_thm1_Y_construct}. Then, for the $h$-th layer where $h= 1,2,\cdots,H$, there exists $\Theta_{[1:h]}$ satisfying \eqref{eq::proof_thm1_localmin_construct1}-\eqref{eq::proof_thm1_localmin_construct4} and $\delta^{(h)},\gamma^{(h)}>0$, such that for any $\Theta_{[1:h]}' \in B\left(\Theta_{[1:h]}, \delta^{(h)}\right)$, we have
		\begin{multline}
		\label{eq::lemma2_lower_bound}
		[\sigma'(a)]^{(H-h)}\cdot\left\langle \Delta Y_{(i,:)}, (T_h)_{(j,:)}\left(\Theta_{[1:h]}'\right) - (T_h)_{(j,:)}\left(\tilde{\Theta}_{[1:h]}'\right)\right\rangle \\\geq \gamma^{(h)} \left\|(T_h)_{(j,:)}\left(\Theta_{[1:h]}'\right) - (T_h)_{(j,:)}\left(\tilde{\Theta}_{[1:h]}'\right)\right\|^2_2
		\end{multline}
		for any $1 \leq i \leq d_{h+1}$ and $1\leq j \leq d_h$. Further, for any $\delta_0 > 0$, there exists $\Theta_{[1:h]}' \in B\left(\Theta_{[1:h]},\delta_0 \right)$ such that
		\begin{equation}\label{eq::lemma2_nonzero}
		 [\sigma'(a)]^{(H-h)} \cdot  \left\langle \Delta Y_{(i,:)}, (T_h)_{(j,:)}\left(\Theta_{[1:h]}'\right)\right\rangle > 0
		\end{equation}
		for any  $1 \leq i \leq d_{h+1}$ and $1\leq j \leq d_h$.
	\end{lemma}
	Roughly speaking, for each hidden layer, there exists $\Theta_{[1:h]}$ satisfying \eqref{eq::proof_thm1_localmin_construct} such that for a sufficiently close perturbation $\Theta'_{[1:h]}$, the ``perturbation directions" of the $h$-th layer's output are restricted (formally characterized by \eqref{eq::lemma2_lower_bound}). Now consider a weight configuration $\Theta$ where $\Theta_{[1:H]}$ is specified by Lemma \ref{lemma::perturb_direction} and $W_{H+1}$ is specified by \eqref{eq::proof_thm1_localmin_construct5}. Obviously, such $\Theta$ satisfies \eqref{eq::proof_thm1_localmin_construct}. In what follows, we show that $\Theta$ is indeed a local minimum.

	Let's inspect the difference of the training loss between $\Theta$ and $\Theta'$, given by
	\begin{align}
	\label{eq::loss_decompose}
	&E(\Theta') - E(\Theta) \nonumber \\
	=& \|T_{H+1}(\Theta') - Y\|^2_F - \|T_{H+1}(\Theta) - Y\|^2_F \nonumber\\
	=& 2\left\langle \Delta Y,T_{H+1}(\Theta') - T_{H+1}(\Theta)\right\rangle_F + \|T_{H+1}(\Theta') - T_{H+1}(\Theta)\|_F^2.
	\end{align}
Therefore $E(\Theta') - E(\Theta) \geq 0$ if
	\begin{equation} 
	\left\langle 
	\Delta Y,T_{H+1}(\Theta') - T_{H+1}(\Theta)\right\rangle_F \geq 0.
	\end{equation}	
We can further decompose $T_{H+1}(\Theta') - T_{H+1}(\Theta)$ as
\begin{align}
	T_{H+1}(\Theta') - T_{H+1}(\Theta) =\left[ T_{H+1}(\tilde{\Theta}') - T_{H+1}(\Theta)\right] +\left[ T_{H+1}(\Theta') - T_{H+1}(\tilde{\Theta}')\right].
	\end{align}
	To prove that $\Theta$ is a local minimum, it suffices to show that for any $\Theta'$ that is sufficiently close to $\Theta$, we have
	\begin{subequations}
		\label{eq::fproduct}
		\begin{gather}
		\label{eq::fproduct_1}
		\left\langle \Delta Y,  T_{H+1}(\tilde{\Theta}') - T_{H+1}(\Theta) \right\rangle_F \geq 0
		\\
		\label{eq::fproduct_2}
		\left\langle \Delta Y, T_{H+1}(\Theta') - T_{H+1}(\tilde{\Theta}') \right\rangle_F \geq 0 .
		\end{gather}
	\end{subequations}

	We first show that, for the constructed $\Theta$ and any $\Theta'$, \eqref{eq::fproduct_1} holds. In fact, since $W_1 = \mathbf{0}$, each network output is invariant to the input data, i.e., each $t_{h,i,n}$ is independent of $n$. As a result, for the output layer we have 
	\begin{subequations}
		\begin{gather}
		(T_{H+1})_{i,1}(\Theta) = (T_{H+1})_{i,2}(\Theta) = \cdots = (T_{H+1})_{i,N}(\Theta) \\ 
		(T_{H+1})_{i,1}(\tilde{\Theta}') = (T_{H+1})_{i,2}(\tilde{\Theta}') = \cdots = (T_{H+1})_{i,N}(\tilde{\Theta}')
		\end{gather}
	\end{subequations}
	for $i = 1, 2, \cdots,d_{H+1}$. Thus, each row of the network output matrix can be written as
	\begin{subequations}
		\begin{gather}
		(T_{H+1})_{(i,:)}(\Theta) = (T_{H+1})_{i,1}(\Theta) \cdot \mathbf{1}_N \\ (T_{H+1})_{(i,:)}(\tilde{\Theta}') = (T_{H+1})_{i,1}(\tilde{\Theta}') \cdot \mathbf{1}_N
		\end{gather}
	\end{subequations}
	and from \eqref{eq::inner1} we have
	\begin{align}
	&\left\langle   \Delta Y , T_{H+1}(\tilde{\Theta}')-T_{H+1}(\Theta)\right\rangle_F \nonumber\\
	=&  \sum^{d_{H+1}}_{i=1}\left\langle \Delta Y_{(i,:)} , (T_{H+1})_{(i,:)}(\tilde{\Theta}')-(T_{H+1})_{(i,:)}(\Theta)\right\rangle\nonumber \\
	=&\sum^{d_H+1}_{i=1}\left[(T_{H+1})_{i,1}(\tilde{\Theta}')-(T_{H+1})_{i,1}(\Theta) \right]  \cdot \left\langle \Delta Y_{(i,:)} ,\mathbf{1}_N \right\rangle  = 0,
	\end{align}
	implying that \eqref{eq::fproduct_1} is satisfied. 
	
	Now define
	\begin{equation}
		\delta^{(H+1)} = \min \left\{\delta^{(H)},\frac{1}{d_{H}}\right\}
	\end{equation} 
	where $\delta^{(H)}$ is specified in Lemma  \ref{lemma::perturb_direction}. Since $w_{H+1,i,j} = 1/d_{H}$, we have $w_{H+1,i,j}'\geq 0 $ for any $\Theta' \in B\left(\Theta,\delta^{(H+1)}\right)$.
	Further, from Lemma \ref{lemma::perturb_direction}, we have
	\begin{subequations}
	\begin{align}
		&\left\langle \Delta Y , T_{H+1}(\Theta') - T_{H+1}(\tilde{\Theta}') \right\rangle_F \nonumber \\
		= &\sum^{d_{H+1}}_{i=1} \left\langle\Delta Y_{(i,:)}, \left(T_{H+1}\right)_{(i,:)}(\Theta') - \left(T_{H+1}\right)_{(i,:)}(\tilde{\Theta}') \right\rangle\\
		= & \sum^{d_{H+1}}_{i=1}\sum^{d_{H}}_{j=1} w'_{H+1,i,j} \left\langle\Delta Y_{(i,:)}, \left(T_H\right)_{(j,:)}(\Theta') - \left(T_H\right)_{(j,:)}(\tilde{\Theta}')\right\rangle \\
		\label{eq::thm1_final_inner_product}
		\geq & \sum^{d_{H+1}}_{i=1}\sum^{d_{H}}_{j=1}\gamma^{(H)} w'_{H+1,i,j}  \left\|(T_h)_{(j,:)}(\Theta') - (T_h)_{(j,:)}(\tilde{\Theta}')\right\|^2_2 \geq 0.
	\end{align}
	\end{subequations}	
	Thus, \eqref{eq::fproduct_2} is also satisfied, and hence $\Theta$ is indeed a local minimum.

	\subsection{Sub-Optimality of Constructed Local-Min}
	Note that \eqref{eq::innerXXsame} implies $\Delta Y \not = \mathbf{0}$, and hence
	\begin{equation}
		E(\Theta) = \left\|\Delta Y\right\|^2_F > 0.
	\end{equation}
	To show that the constructed local minimum $\Theta$ is sub-optimal, in what follows we construct a weight configuration $\Theta^* = \left(W_1^*,\mathbf{b}_1^*, \cdots, W_H^*,\mathbf{b}_H^*,W_{H+1}^*\right)$ with $E(\Theta^*) < E(\Theta)$.
	
	From Lemma \ref{lemma::perturb_direction}, for the constructed local minimum $\Theta$, there exists $\Theta'_{[1:H]} \in B\left(\Theta_{[1:H]},\delta^{(H)}\right)$ such that 
	\begin{equation}
	\label{eq::proof_thm1_subconstruct}
	\left\langle \Delta Y_{(i,:)}, (T_H)_{(j,:)}\left(\Theta_{[1:H]}'\right)\right\rangle > 0
	\end{equation}
	for any  $1 \leq i \leq d_{H+1}$ and $1\leq j \leq d_H$. We first set
	\begin{equation}
		\Theta^*_{[1:(H-1)]} = \Theta'_{[1:(H-1)]}
	\end{equation}
	implying $T_{H-1}\left(\Theta^*_{[1:(H-1)]}\right) = T_{H-1}\left( \Theta'_{[1:(H-1)]}\right)$. Next, we let
	\begin{subequations}
	\begin{gather}
	\left(W^*_{H}\right)_{(1,:)} = \left(W'_{H}\right)_{(1,:)}, 
	\quad \left(\mathbf{b}^*_H\right)_{1} = \left(\mathbf{b}'_H\right)_{1}\\  \left(W^*_{H}\right)_{(j,:)} = \mathbf{0}_{d_{H-1}}, \quad \left(\mathbf{b}^*_H\right)_j = a, \quad j =2, 3, \cdots, d_H
	\end{gather}
	\end{subequations}
	yielding
	\begin{subequations}
		\begin{gather} 
		\left(T_{H}\right)_{(1,:)}\left(\Theta^*_{[1:H]}\right) = \left(T_{H}\right)_{(1,:)}\left(\Theta'_{[1:H]}\right) \\
		\left(T_H\right)_{(j,:)}\left(\Theta^*_{[1:H]}\right) = \sigma(a)\cdot \mathbf{1}_N,\quad j = 2,3,\cdots, d_H.
		\end{gather}
	\end{subequations} 
	Finally, we set 
	\begin{equation}
		(W_{H+1})_{(i,:)} = \left[v, \frac{1}{d_H-1}, \frac{1}{d_H-1}, \cdots, \frac{1}{d_H-1}\right]^\top \in \mathbb{R}^{d_H}
	\end{equation}
	for every $i=1,2,\cdots, d_{H+1}$, and then represent the empirical loss in terms of $v$, yielding
	\begin{subequations}
		\begin{align}
			E\left(\Theta^*\right) &= \left\|Y - T_{H+1}\left(\Theta^*\right)\right\|^2_F \\
			& = \sum^{d_{H+1}}_{i=1} \left\|Y_{(i,:)} - \sigma(a) \cdot \mathbf{1}_{N} -  v\cdot \left(T_{H}\right)_{(1,:)}\left(\Theta'_{[1:H]}\right)  \right\|_2^2 \\ 
			& = d_{H+1} v^2 \left\|\mathbf{m}\right\|^2_2 + \sum^{d_{H+1}}_{i=1}\left\|\Delta  Y_{(i,:)}\right\|_2^2 - 2 \sum^{d_{H+1}}_{i=1}\left\langle \Delta Y_{(i,:)}, v\cdot \mathbf{m} \right\rangle \\
			\label{eq::proof_thm1_suboptimal}
			& = d_{H+1}v^2 \left\|\mathbf{m}\right\|^2_2 + E(\Theta) - 2v \sum^{d_{H+1}}_{i=1}\left\langle \Delta Y_{(i,:)}, \mathbf{m} \right\rangle
		\end{align}
	\end{subequations}
	where we define $\mathbf{m} \triangleq \left(T_{H}\right)_{(1,:)}\left(\Theta'_{[1:H]}\right)$ to ease notation. From \eqref{eq::proof_thm1_subconstruct} we have $\sum^{d_{H+1}}_{i=1}\left\langle \Delta Y_{(i,:)}, \mathbf{m}\right\rangle > 0 $. From the form of \eqref{eq::proof_thm1_suboptimal}, there exists $v < 0$ with sufficiently small $|v|$ such that
	\begin{equation}
	E\left(\Theta^*\right) <E(\Theta)
	\end{equation} 
	and hence $\Theta$ is sub-optimal. We complete the proof.

	\section{Conclusion}\label{sec::conclusion}
	
	In this paper, we studied the loss landscape of  of nonlinear neural networks. Specifically, we construct bad local minima for networks with almost all analytic activations, any depth, and any width. 
	This result solves a long-standing question of ``whether sub-optimal local minima exist in general over-parameterized neural networks", and the answer is somewhat astonishingly negative. Nevertheless, combining with other positive results, we believe that this work reveals the exact landscape of over-parameterized neural networks, which is not as nice as people generally think but much better than general non-convex functions. This work also provides a future research direction of how to avoid such sub-optimal local minima effectively in a general setting during the training process, and calls for a deeper understanding of the empirical efficiency of training neural networks.

	\appendix
	
	\section{Useful Lemma on Real Analytic Functions}
	Before we present the formal proofs, we borrow a useful result from \cite{mityagin2015zero}, stated formally as the following lemma.
	\begin{lemma}
		\label{lemma::analytic}
		For any $d\in \mathbb{Z}^+$, let $f: \mathbb{R}^d\rightarrow \mathbb{R}$ be a real analytic function on $\mathbb{R}^d$. If $f$ is not identically zero, then its zero set $\Omega=\{\mathbf{a} \in \mathbb{R}^d \mid f(\mathbf{a})=0\}$ has zero measure.
	\end{lemma}

	Lemma \ref{lemma::analytic} states that the zero set of a real analytic function is either the whole domain or zero-measure. Thus, to verify that a real analytic function $f: \mathbb{R}^d\rightarrow \mathbb{R}$ is not identically zero, it suffices to find one $\mathbf{a}\in \mathbb{R}^d$ such that $f(\mathbf{a})\not=0$. Further, if $f$ is not identically zero, then it is non-zero for generic input. That is, for any $\mathbf{a} \in \mathbb{R}^d$ and any $\epsilon>0$, there exists $\mathbf{a}' \in B(\mathbf{a}, \epsilon)$ with $f(\mathbf{a}') \not= 0$. This lemma will be applied multiple times in our proofs.
	
	\ifnonsmooth

	\section{Proof of Theorem \ref{thm::partiallinear-bad-localmin}}
	From Assumption \ref{ass::partial-linear-activation}, $\sigma$ is linear in $(a - \delta, a + \delta)$, say
	\begin{equation}
	\label{eq::proof_thm2_linear_part}
	\sigma(t) = \alpha  (t - a) + \beta, \quad\!\! t\in (a - \delta, a + \delta).
	\end{equation}

\subsection{Local-Min for Degenerate Case}
In this subsection, we prove the theorem for the degenerate case with $\alpha = 0$, i.e., $\sigma(t) = \beta$ for all $t \in (a-\delta, a+\delta)$. In this case, we can construct a trivial local minimum. First, we set
	\begin{subequations}
		\begin{gather}
	 	W_{h} = \mathbf{0}_{d_h \times d_{h-1}}\\
	 	\mathbf{b}_h = a \cdot \mathbf{1}_{d_h}
	 	\end{gather}
	\end{subequations}
	for $h = 1, 2, \cdots, H$. It can be readily verified that the input to every hidden-layer neuron is $a$ for any input data sample, yielding
	\begin{equation}
		T_h(\Theta) = \beta \cdot \mathbf{1}_{d_h \times N}, \quad h = 1, 2, \cdots, H.
	\end{equation}
	Then we let
	\begin{equation}
	\label{eq::proof_thm2_trivial_base}
		W_{H+1} \in \argminA_{V \in \mathbb{R}^{d_{H+1} \times d_H}} \|Y - \beta V \mathbf{1}_{d_H\times N}\|^2_F
	\end{equation}
	which is a minimizer of a convex optimization problem. From Assumption \ref{ass::partial-linear-activation}\ref{ass::partial-linear-activation_c}, we see that the dimension of $\text{row}(Y)$ is greater than that of $\text{row}(\mathbf{1}_{d_H \times N})$, implying $E(\Theta) > 0$.

	Since $\alpha = 0$, a sufficiently small perturbation on $\Theta$ will not change the output of each hidden neuron. That is, there exists $\delta_1 >0 $ such that for any $\Theta' \in B(\Theta,\delta_1)$, we have $z_{h,i,n} \in B(a,\delta)$, and hence $T_h(\Theta') = \beta \cdot \mathbf{1}_{d_h \times N}$ for each hidden layer. We further have
	\begin{subequations}
		\begin{align}
		E(\Theta') &= \|Y - T_{H+1}(\Theta')\|^2_F \\ 
		& = \|Y - \beta W'_{H+1}\mathbf{1}_{d_H \times N} \|^2_F\\
		\label{eq::proof_thm2_trivial}
		& \geq\| Y  - \beta W_{H+1}\mathbf{1}_{d_H \times N} \|^2_F \\
		& = E(\Theta) > 0
		\end{align}
	\end{subequations}
	where \eqref{eq::proof_thm2_trivial} follow from \eqref{eq::proof_thm2_trivial_base}. This implies that $\Theta$ is a local minimum of $E$ with $E(\Theta) > 0$.
	
	\subsection{Local-Min for Non-Degenerate Case}
	In what follows, we construct local minimum for the case $\alpha \not= 0$. Our target is to design the weights to meet
	\begin{subequations}
		\label{eq::proof_thm2_linear_con}
		\begin{gather}
		\label{eq::proof_thm2_linear_con1}
		Z_h \in B\left(a\cdot \mathbf{1}_{d_h \times N}, \delta\right)\\
		\label{eq::proof_thm2_linear_con2}
		\mathrm{row}\left(T_{h}(\Theta)\right) = \mathrm{row}\left(\left[\begin{matrix}
		X \\
		\mathbf{1}_N^\top
		\end{matrix}\right] \right)
		\end{gather}
	\end{subequations}
	for all $1 \leq h  \leq H$. This is done by induction, detailed as follows.
	
	We begin with the weights to the first hidden layer. Notice that $d_1 > d_0$, we let 
	\begin{equation}
	W_1 = V_1 \in \mathbb{R}^{d_1 \times d_0}
	\end{equation}
	with $V_1$ satisfies $\|V_1X\|_F < \delta/2$, to be determined later, and 
	\begin{equation}
		\mathbf{b}_1 = a \mathbf{1}_{d_1} + \mathbf{u}_1
	\end{equation}
	where $\mathbf{u}_1 \in \mathbb{R}^{d_1}$ is an arbitrary vector satisfying $\|\mathbf{u}_1\|_2 < \delta/(2N)$, also to be determined later. Noting that
	\begin{equation}
		Z_1 =  W_1X + \mathbf{b}_1\mathbf{1}^\top_N = a\cdot \mathbf{1}_{d_1\times N} + V_1 X + \mathbf{u}_1\mathbf{1}^\top_N
	\end{equation}
	and
	\begin{equation}
	\left\|V_1 X + \mathbf{u}_1\mathbf{1}_N^\top\right\|_F \leq \|V_1X\|_F + N\|\mathbf{u}_1\|_2 < \delta,
	\end{equation}
	we have $Z_1 \in B(a\cdot \mathbf{1}_{d_1\times N}, \delta)$. That is, \eqref{eq::proof_thm2_linear_con1} holds. From \eqref{eq::proof_thm2_linear_part}, we have
	\begin{subequations}
	\begin{align}
	T_1 =&\sigma\left(Z_1\right) 
	= \alpha \left(V_1X + \mathbf{u}_1\mathbf{1}_N^\top \right)+ \beta\mathbf{1}_{d_1 \times N} \\
	=& \left(\left[V_1, \mathbf{u}_1\right] + \left[\mathbf{0}_{d_1 \times d_0}, \frac{\beta}{\alpha} \mathbf{1}_{d_1}\right]\right) \left(\alpha \cdot \left[
	\begin{matrix}
	X \\
	\mathbf{1}_N^\top
	\end{matrix}
	\right]\right).
	\end{align}
	\end{subequations}
	To proceed, we present a following lemma.
	\begin{lemma}
		\label{lemma::proof_thm2_matrix_existence}
		Consider arbitrary two matrices $A \in \mathbb{R}^{L_1 \times L_2}$ and $B \in \mathbb{R}^{L_2 \times L_3}$ with $L_1 \geq L_2$. For any $\epsilon >0 $, there exists $C \in \mathbb{R}^{L_1 \times L_2}$ with $\|C\|_F <\epsilon$, such that
		\begin{equation}
			\label{eq::proof_thm2_matrix_existence}
			\mathrm{row}\left((C+A)B\right) =  \mathrm{row}\left(B\right).
		\end{equation}
	\end{lemma}
	From Lemma \ref{lemma::proof_thm2_matrix_existence}, there exists $[V_1,  \mathbf{u}_1]$ with
	\begin{equation}
	\label{eq::proof_thm2_bound1}
	   \left\|[V_1, \mathbf{u_1}]\right\|_F < \min\left\{\frac{\delta}{2\max\{\|X\|_F,1\}}, \frac{\delta}{2N}\right\}
	\end{equation}
	such that
	\begin{subequations}
		\label{eq::proof_thm2_condition2_w1}
		\begin{align}
		\mathrm{row}\left(T_1\right) = & \mathrm{row}\left( \left(\left[V_1, \mathbf{u}_1\right] + \left[\mathbf{0}_{d_1 \times d_0}, \frac{\beta}{\alpha} \mathbf{1}_{d_1}\right]\right) \left(\alpha \cdot \left[
		\begin{matrix}
		X \\
		\mathbf{1}_N^\top
		\end{matrix}
		\right]\right)\right) \\
		= & \mathrm{row}\left(\alpha \cdot \left[
		\begin{matrix}
		X \\
		\mathbf{1}_N^\top
		\end{matrix}
		\right]\right)\\
		= & \mathrm{row}\left(\left[
		\begin{matrix}
		X \\
		\mathbf{1}_N^\top
		\end{matrix}
		\right]\right).
		\end{align}
	\end{subequations}
	Note that \eqref{eq::proof_thm2_bound1} implies $\|V_1X\|_F < \delta/2$ and $\|\mathbf{u}_1\|_2 < \delta/(2N)$. Therefore, condition \eqref{eq::proof_thm2_linear_con} is met for the first hidden layer.
	
	Suppose that \eqref{eq::proof_thm2_linear_con} holds for the $(h-1)$-th hidden layer ($1 < h \leq H$), following a similar analysis we can construct $W_{h}$ and $\mathbf{b}_h$ to meet \eqref{eq::proof_thm2_linear_con} for the $h$-th hidden layer. Specifically, From Lemma \ref{lemma::proof_thm2_matrix_existence}, there exists $V_h \in \mathbb{R}^{d_h \times d_{h-1}}$ and $\mathbf{u}_h \in \mathbb{R}^{d_h}$ with
	\begin{equation}
	\label{eq::proof_thm2_boundh}
	\left\|[V_h, \mathbf{u}_h]\right\|_F < \min\left\{\frac{\delta}{2\|T_{h-1}\|_F}, \frac{\delta}{2N}\right\}
	\end{equation}
	such that
	\begin{subequations}
		\label{eq::proof_thm2_condition2_wh}
		\begin{align}
		& \mathrm{row}\left( \left(\left[V_h, \mathbf{u}_h\right] + \left[\mathbf{0}_{d_h \times d_{h-1}}, \frac{\beta}{\alpha} \mathbf{1}_{d_h}\right]\right) \left(\alpha \cdot \left[
		\begin{matrix}
		T_{h-1} \\
		\mathbf{1}_N^\top
		\end{matrix}
		\right]\right)\right) \\
		= &  \mathrm{row}\left(\alpha \cdot \left[
		\begin{matrix}
		T_{h-1} \\
		\mathbf{1}_h^\top
		\end{matrix}
		\right]\right) =  \mathrm{row}\left(\left[
		\begin{matrix}
		X \\
		\mathbf{1}_h^\top
		\end{matrix}
		\right]\right).
		\end{align}
	\end{subequations}
	We let 
	\begin{equation}
	W_h = V_h, \quad \mathbf{b}_h = a \mathbf{1}_{d_h} + \mathbf{u}_h.
	\end{equation}
	From \eqref{eq::proof_thm2_boundh}, we have $\|V_hT_{h-1}\|_F < \delta/2$ and $\|\mathbf{u}_h\|_2 < \delta/(2N)$. Then
	\begin{equation}
	Z_h =  W_hT_{h-1} + \mathbf{b}_h\mathbf{1}^\top_N = a\cdot \mathbf{1}_{d_h\times N} + V_h T_{h-1} + \mathbf{u}_h\mathbf{1}^\top_N
	\end{equation}
	where
	\begin{equation}
	\left\|V_h T_{h-1} + \mathbf{u}_h\mathbf{1}_N^\top\right\|_F \leq \|V_hT_{h-1}\|_F + N\|\mathbf{u}_h\|_2 < \delta.
	\end{equation}
	Thus, $Z_h \in B(a\cdot \mathbf{1}_{d_h\times N}, \delta)$, and hence \eqref{eq::proof_thm2_linear_con1} holds. Then, from \eqref{eq::proof_thm2_linear_part} we have
	\begin{subequations}
		\label{eq::proof_thm2_condition2_wh2}
		\begin{align}
		T_h =&\sigma\left(Z_h\right) 
		= \alpha \left(V_hT_{h-1} + \mathbf{u}_h\mathbf{1}_N^\top \right)+ \beta\mathbf{1}_{d_h \times N} \\
		=& \left(\left[V_h, \mathbf{u}_h\right] + \left[\mathbf{0}_{d_h \times d_{h-1}}, \frac{\beta}{\alpha} \mathbf{1}_{d_h}\right]\right) \left(\alpha \cdot \left[
		\begin{matrix}
		T_{h-1} \\
		\mathbf{1}_N^\top
		\end{matrix}
		\right]\right).
		\end{align}
	\end{subequations}
	Combining \eqref{eq::proof_thm2_condition2_wh} and \eqref{eq::proof_thm2_condition2_wh2}, we conclude that condition \eqref{eq::proof_thm2_linear_con} also holds for the $h$-th hidden layer. By induction, we can construct $(W_1, \mathbf{b}_1, W_2, \mathbf{b}_2, \cdots, W_H, \mathbf{b}_H)$ such that \eqref{eq::proof_thm2_linear_con} holds for all hidden layers.
	  
	Finally, we set the weights to the output layer as
	\begin{equation}
	W_{H+1} \in \argminA_{V \in \mathbb{R}^{d_{H+1} \times d_H}} \|Y - V T_H\|^2_F.
	\end{equation}
	which is a minimizer of a convex optimization problem. From \eqref{eq::proof_thm2_linear_con2}, we see that $W_{H+1}$ equivalently minimizes the distance from $Y$ to $\mathrm{row}\left(\left[\begin{matrix}
	X \\
	\mathbf{1}_N^\top
	\end{matrix}\right]\right)$, i.e.,
	\begin{subequations}
	\begin{align}
	E(\Theta) & = \min_{V \in \mathbb{R}^{d_{H+1} \times d_H}}\left\| Y - V T_H(\Theta) \right\|^2_F\\
	& = \min_{V \in \mathbb{R}^{d_{H+1} \times (d_H+1)}}\left\| Y - V  \left[\begin{matrix}
	X \\
	\mathbf{1}_N^\top
	\end{matrix}\right]\right\|^2_F.
	\end{align}
	\end{subequations}
	From Assumption \ref{ass::partial-linear-activation}\ref{ass::partial-linear-activation_c}, we further have $E(\Theta)>0$.
	
	To complete the proof, it suffices to show that the constructed $\Theta$ is indeed a local minimum. From Assumption \ref{ass::partial-linear-activation}, there exists $\delta_2$ such that for any $\Theta' \in B(\Theta,\delta_2)$, the input to any hidden-layer neuron is within $(a-\delta, a+\delta)$. Then, it can be shown by induction that
	\begin{equation}
	\mathrm{row}(T_h(\Theta')) \in \mathrm{row}\left(\left[\begin{matrix}
	X \\
	\mathbf{1}_N^\top
	\end{matrix}\right]\right)
	\end{equation}
	for $h = 1, 2,\cdots, H$. Therefore,
	\begin{align}
	E(\Theta') &= \left\| Y - W'_{H+1} T_H\left(\Theta'\right)\right\|^2_F \nonumber\\
	&\geq \min_{V \in \mathbb{R}^{d_{H+1} \times (d_H+1)}}\left\| Y - V  \left[\begin{matrix}
	X \\
	\mathbf{1}_N^\top
	\end{matrix}\right]\right\|^2_F \nonumber\\
	&=  E(\Theta)
	\end{align}
	Thus, $\Theta$ is a local minimum with $E(\Theta)>0$.
	
	\fi

	\section{Proof of Theorem \ref{thm::bad_localmin_sigmoid}}
	We consider a $1$-hidden-layer neural network with sigmoid function. We show that for generic $X$, there exists an output vector $Y\in \mathbb{R}^{1 \times N}$ such that the network has a bad local minimum. Moreover, there exists $\delta > 0$ such that for any perturbed input and output data $X' \in B(X,\delta)$, $Y' \in B(Y, \delta)$, the resulting network still has a bad local minimum.
	
	The proof consists of four stages. First, we show that the considered 1-hidden-layer sigmoid network has sub-optimal local minimum for generic $X$. Second, we show that a corresponding 1-neuron sigmoid network has a strict local minimum. Third, we consider an arbitrary perturbation in a sufficiently small region on the training data, and show that local minimum still exists for the ``perturbed" 1-neuron network. Fourth, we prove the existence of sub-optimal local minimum for the wide network with the perturbed data.
	
	Note that we consider a single-hidden-layer network with $d_0 = d_2 = 1$. For notational brevity, in this proof we denote the data and the weights in vector forms and throw out the subscripts, i.e.,
	\begin{subequations}
		\begin{gather}
		\mathbf{x} = X^\top \in \mathbb{R}^{N}, \quad \!\! \mathbf{y} = Y^\top \in \mathbb{R}^{N} \\
		\mathbf{b} = \mathbf{b}_1 \in \mathbb{R}^{d_1}, \quad \!  \mathbf{w} = W_1 \in \mathbb{R}^{d_1}, \quad \!\! \mathbf{v} =W_2^\top \in \mathbb{R}^{d_1}, 
		\end{gather}
	In this way, we represent the training data by $(\mathbf{x}, \mathbf{y})$ and re-parameterize the weight configuration as $\Theta = (\mathbf{v}, \mathbf{w}, \mathbf{b})$. Then, the network output and the empirical loss are respectively represented by
	\begin{gather}
		\mathbf{t} = \mathbf{v}^\top \sigma\left(\mathbf{w}\mathbf{x}^\top + \mathbf{b} \cdot \mathbf{1}_N^\top\right) \in \mathbb{R}^N \\
		E(\Theta) = \|\mathbf{y} - \mathbf{t}\|^2_2.
	\end{gather}
	\end{subequations}
	
	\subsection{Sub-Optimal Local-Min Construction}
	In the first stage, we aim to construct sub-optimal local minimum for generic $X$. We first set $\mathbf{w} = w \cdot \mathbf{1}_{d_1}$ and $\mathbf{b} = b \cdot \mathbf{1}_{d_1}$ for some $w, b \in \mathbb{R}$. That is, the input weights to all hidden neurons are the same, and so are the biases. Denote 
	\begin{equation}
		\mathbf{z}_i \triangleq (Z_1)_{(i,:)} \in \mathbb{R}^N
	\end{equation} 
	as the ``input vector" of the the $i$-th hidden neuron. Noting that the ``input vector" of each hidden neuron is the same, we denote
	\begin{equation}
	\mathbf{z} = \mathbf{z}_1 = \mathbf{z}_2 = \cdots = \mathbf{z}_{d_1} = w \cdot \mathbf{x} + b \cdot \mathbf{1}_N.
	\end{equation}
	
	We consider the following set consisting of 6 $N$-dimensional vectors
	\begin{equation}
	\mathcal{A} = \{ \sigma(\mathbf{z}) ,\, \sigma'(\mathbf{z}),\, \sigma'(\mathbf{z})\circ \mathbf{x},\, \sigma''(\mathbf{z}), \sigma''(\mathbf{z})\circ \mathbf{x},\, \sigma''(\mathbf{z})\circ \mathbf{x}\circ \mathbf{x}\} \subseteq \mathbb{R}^N.
	\end{equation}
	We claim that for generic $(w, b, \mathbf{x})$, the vectors in $\mathcal{A}$ are linearly independent. Define a matrix
	\begin{equation}
	A \triangleq [\sigma(\mathbf{z}) ,\, \sigma'(\mathbf{z}),\, \sigma'(\mathbf{z})\circ \mathbf{x},\, \sigma''(\mathbf{z}), \sigma''(\mathbf{z})\circ \mathbf{x},\, \sigma''(\mathbf{z})\circ \mathbf{x}\circ \mathbf{x}]\in \mathbb{R}^{N \times 6}
	\end{equation}
	with columns formed by the vectors in $\mathcal{A}$. Noting that $N \geq 6$, let $A_0 \in \mathbb{R}^{6\times 6}$ be the sub-matrix consisting of the first 6 rows of $A$. Since the sigmoid function is analytic, $\mathrm{det}(A_0)$ can be represented as a polynomial of analytic functions with respect to $(w,b,\mathbf{x})$. Thus, $\det(A_0)$ is also analytic with respect to $(w,b,\mathbf{x})$. Set $w=1, b= 0$, and $x_i = i-1$ for $i=1,2,\cdots, N$. We have
	\begin{equation}
	A_0 = \begin{bmatrix}
	\sigma(0) & \sigma'(0) & 0 & \sigma''(0) & 0 & 0 \\
	\sigma(1) & \sigma'(1) & \sigma'(1) & \sigma''(1)& \sigma''(1) & \sigma''(1)\\
	\sigma(2) & \sigma'(2) & 2\sigma'(2) & \sigma''(2) & 2\sigma''(2) & 4\sigma''(2) \\
	\sigma(3) & \sigma'(3) & 3\sigma'(3) & \sigma''(3) & 3\sigma''(3) & 9\sigma''(3) \\
	\sigma(4) & \sigma'(4) & 4\sigma'(4) & \sigma''(4) & 4\sigma''(4) & 16\sigma''(4) \\
	\sigma(5) & \sigma'(5) & 5\sigma'(5) & \sigma''(5) & 5\sigma''(5) & 25\sigma''(5)
	\end{bmatrix}
	\end{equation}
	which is full-rank. This implies that $\det(A_0)$, as a function of $(w,b, \mathbf{x})$, is not identically zero. By Lemma \ref{lemma::analytic}, for generic $(w,b,\mathbf{x})$, $\det(A_0)$ is non-zero, so $A_0$ is full-rank, implying that $A$ is of full column rank. We further have that for generic $\mathbf{x} \in \mathbb{R}^N$, there exists $(w, b)$ such that the vectors in $\mathcal{A}$ are linearly independent.
	
	Given that the vectors in $\mathcal{A}$ are linearly independent, by Lemma \ref{lemma::dim_requirement} we can find $N-4$ linearly-independent vectors $\{\mathbf{u}_l\}$, each of $N$ dimensions, such that
	\begin{subequations}\label{eq::construct_u2}
		\begin{gather}
		\langle \mathbf{u}_l , \sigma(\mathbf{z}) \rangle = 0 \\ 
		\langle \mathbf{u}_l , \sigma'(\mathbf{z}) \rangle = 0 \\
		\langle \mathbf{u}_l , \sigma'(\mathbf{z})\circ\mathbf{x} \rangle = 0 \\ 
		\langle \mathbf{u}_l , \sigma''(\mathbf{z}) \circ \mathbf{x} \rangle = 0 \\ 
		\langle \mathbf{u}_l , \sigma''(\mathbf{z}) \rangle > 0 \\ 
		\langle \mathbf{u}_l , \sigma''(\mathbf{z})  \circ \mathbf{x} \circ \mathbf{x} \rangle > 0 .
		\end{gather}
	\end{subequations}
	for $l=1,2,\cdots, N-4$. We then set $\mathbf{v}$ to be an arbitrary vector with entries all-positive or all-negative, i.e.,
	\begin{equation}
	\label{eq::proof_thm3_v_same_sign}
	v_i \cdot v_{i'} > 0
	\end{equation}
	for all $1 \leq i,i' \leq d_1$. Finally, we let
	\begin{equation}
	\label{eq::proof_thm3_construct_Y1}
	\mathbf{y} = \mathbf{t} - v_1 \sum^{N-4}_{l=1} \alpha_l \mathbf{u}_l
	\end{equation}
	where each $\alpha_l$ is an arbitrary positive value.
	
	In what follows, we show that the constructed $\Theta = (\mathbf{v}, w\cdot\mathbf{1}_{d_1}, b\cdot\mathbf{1}_{d_1})$ is indeed a local minimum of the considered network with training data $(\mathbf{x}, \mathbf{y})$. We provide the following lemma that identifies a sufficient condition for the local minimum.
	\begin{lemma}
		\label{lemma::localmin_system1}
		Consider a $1$-hidden-layer neural network with $d_0 = d_2= 1$ and twice-differentiable activation. Suppose that at a weight configuration $\Theta = (\mathbf{v},\mathbf{w}, \mathbf{b})$, all the equations and inequalities
		\begin{subequations}
			\label{eq::localmin_system1}
			\begin{gather}
			\langle \Delta \mathbf{y}, \sigma(\mathbf{z}_i) \rangle = 0 \\
			\langle \Delta \mathbf{y}, \sigma'(\mathbf{z}_i) \rangle = 0 \\
			\langle \Delta \mathbf{y}, v_i \sigma'(\mathbf{z}_i) \circ \mathbf{x} \rangle = 0 \\
			\langle \Delta \mathbf{y}, v_i \sigma''(\mathbf{z}_i) \circ \mathbf{x} \rangle = 0 \\
			\langle \Delta \mathbf{y}, v_{i} \sigma''(\mathbf{z}_i)\rangle > 0 \\
			\langle \Delta \mathbf{y}, v_i \sigma''(\mathbf{z}_i) \circ \mathbf{x} \circ \mathbf{x} \rangle > 0
			\end{gather}
		\end{subequations}
		hold for all $1 \leq i \leq d_1$, where $\Delta \mathbf{y} \triangleq \mathbf{t} - \mathbf{y} $. Then $\Theta$ is a local minimum of $E$.
	\end{lemma}
To show that $\Theta$ is a local minimum, it suffices to verify \eqref{eq::localmin_system1}. From \eqref{eq::proof_thm3_construct_Y1} we have
	\begin{equation}
	 	\Delta \mathbf{y} = \mathbf{t} - \mathbf{y} = v_1 \sum^{N-4}_{l=1} \alpha_l \mathbf{u}_l
	\end{equation}
	Then, from \eqref{eq::construct_u2} and \eqref{eq::proof_thm3_v_same_sign} we obtain
	\begin{subequations}
		\label{eq::localmin_system_sigmoid}
		\begin{gather}
		\langle \Delta \mathbf{y}, \sigma(\mathbf{z}) \rangle  =  v_1 \sum ^{N-4}_{l = 1} \alpha_{l} \langle \mathbf{u}_{l} , \sigma(\mathbf{z})\rangle = 0 \\
		\langle \Delta \mathbf{y}, \sigma'(\mathbf{z}) \rangle  =  v_1 \sum ^{N-4}_{l = 1} \alpha_{l} \langle \mathbf{u}_{l} , \sigma'(\mathbf{z})\rangle = 0\\
		\langle \Delta \mathbf{y}, v_i \sigma'(\mathbf{z}) \circ \mathbf{x} \rangle  = v_1 v_i  \sum ^{N-4}_{l = 1} \alpha_{l} \langle \mathbf{u}_{l} , \sigma'(\mathbf{z}) \circ \mathbf{x}\rangle = 0 \\
		\langle \Delta \mathbf{y}, v_i \sigma''(\mathbf{z}) \circ \mathbf{x} \rangle  = v_1 v_i  \sum ^{N-4}_{l = 1} \alpha_{l} \langle \mathbf{u}_{l} , \sigma''(\mathbf{z}) \circ \mathbf{x}\rangle = 0\\
		\langle \Delta \mathbf{y}, v_i \sigma''(\mathbf{z}) \rangle = v_1 v_i  \sum ^{N-4}_{l = 1} \alpha_{l} \langle \mathbf{u}_{l} , \sigma''(\mathbf{z})\rangle > 0 \\
		\langle \Delta \mathbf{y}, v_i \sigma''(\mathbf{z}) \circ \mathbf{x}\circ \mathbf{x} \rangle =  v_1 v_i  \sum ^{N-4}_{l = 1} \alpha_{l} \langle \mathbf{u}_{l} , \sigma''(\mathbf{z}) \circ \mathbf{x}\circ \mathbf{x}\rangle  > 0
		\end{gather}
	\end{subequations}
	for all $1 \leq i \leq d_1$. By Lemma \ref{lemma::localmin_system1}, $\Theta = (\mathbf{v},w \cdot \mathbf{1},b\cdot \mathbf{1})$ is a local minimum of $E$. Noting that $\Delta \mathbf{y} \not = \mathbf{0}$, we have
	\begin{equation}
	E(\Theta) = \|\Delta \mathbf{y}\|^2_2 > 0.
	\end{equation}
	
	Finally, we note that the sigmoid activation function is analytic with non-zero derivatives at the origin up to the $(N-1)$-th order, i.e.,
	\begin{equation}
		\sigma(0), \sigma'(0), \sigma''(0), \cdots, \sigma^{(N-1)}(0) \not =  0.
	\end{equation}
	From the results in \cite{li2018over}, if $d_1 \geq N$ and $\mathbf{x}$ has distinct entries, the network is realizable, i.e., $\min_\Theta E(\Theta) = 0$, regardless of $\mathbf{y}$. Notice that the set of $\mathbf{x}$ with distinct $x_i$ is generic in $\mathbb{R}^N$. We conclude that for generic $\mathbf{x}$, our constructed local minimum $\Theta$ is sub-optimal.
	
	\subsection{Strict Local-min for 1-Neuron Network}
	In the previous subsection, we show that for generic $\mathbf{x}$, there exists $\mathbf{y}$ such that the considered network has a sub-optimal local minimum $\Theta = (\mathbf{v}, w\cdot\mathbf{1}_{d_1}, b\cdot\mathbf{1}_{d_1})$. Now, we considered a 1-neuron sigmoid network with the same training data $(\mathbf{x}, \mathbf{y})$. The empirical loss is represented by
    \begin{subequations}
	\begin{equation}
	\underline{E}(\underline{\Theta}) = \| \mathbf{y} - \underline{v}\sigma(\underline{w}\mathbf{x} + \underline{b} \cdot \mathbf{1})\|^2_2.
	\end{equation}
	where $\underline{\Theta} = (\underline{v},\underline{w},\underline{b})$ is the weight configuration, and $\underline{v},\underline{w}, \underline{b} \in\mathbb{R}$ are the weight to the output layer, the weight to the hidden layer, and the bias to the hidden layer, respectively. In this proof, we use underlined notations to represent variables relating to the 1-neuron network. Similar to the original network, we define
		\begin{gather}
		\underline{\mathbf{z}} \triangleq \underline{w} \cdot \mathbf{x} + \underline{b}  \cdot \mathbf{1} \\ 
		\underline{\mathbf{t}} \triangleq  \underline{v} \cdot \sigma(\underline{\mathbf{z}}) \\
		\Delta \underline{\mathbf{y}} \triangleq  \underline{\mathbf{t}}- \mathbf{y}.
		\end{gather}
	\end{subequations}
 	
 	Given the constructed bad local minimum $\Theta = (\mathbf{v}, w\cdot \mathbf{1},b \cdot \mathbf{1})$ for the original wide network. We set
 	\begin{equation}
 	\underline{v} =  \sum^{d_1}_{i=1}v_i, \quad \underline{w} = w, \quad \underline{b} = b.
 	\end{equation}
 	In the following, we show that $\underline{\Theta} = (\underline{v}, \underline{w}, \underline{b})$ is a strict local minimum of $\underline{E}$.
 	
 	As $\Theta$ satisfies equation/inequality system \eqref{eq::localmin_system_sigmoid} for the original network, it can be readily verified that $\underline{\Theta}$ also satisfies \eqref{eq::localmin_system_sigmoid} for the 1-neuron network. We can then compute the gradient of $\underline{E}$ at $\underline{\Theta}$ as
	\begin{subequations}
		\begin{align}
		\frac{\partial \underline{E}}{\partial \underline{v}} &= 2 \left\langle \Delta \underline{\mathbf{y}} , \sigma(\underline{\mathbf{z}})\right\rangle = 0 \\ 
		 \frac{\partial \underline{E}}{\partial \underline{w}} &= 2 \underline{v} \left\langle \Delta \underline{\mathbf{y}} , \sigma'(\underline{\mathbf{z}} )\circ \mathbf{x}\right\rangle = 0 \\
		\frac{\partial \underline{E}}{\partial \underline{b}} &= 2 \underline{v}\left\langle \Delta \underline{\mathbf{y}} , \sigma'(\underline{\mathbf{z}})\right\rangle = 0
		\end{align}
	\end{subequations}
	Thus, the gradient $\nabla \underline{E}(\underline{\Theta}) = \mathbf{0}$, so $\underline{\Theta}$ is a stationary point of $\underline{E}$. 
	
	Now, consider an arbitrary perturbation $\underline{\Theta}' = (\underline{v}',\underline{w}',\underline{b}') = (\underline{v} + \Delta v, \underline{w} + \Delta w, \underline{b} + \Delta b)$ in the neighbourhood of $\underline{\Theta}$, with $(\Delta v, \Delta w, \Delta b) \not = (0,0,0)$. For brevity, in this proof we use prime notations for variables corresponding to the perturbed weights. We have
	\begin{align}
	\underline{\mathbf{t}}' - \underline{\mathbf{t}} &= \underline{v}'\sigma(\underline{\mathbf{z}}') -  \underline{v}\sigma(\underline{\mathbf{z}}) \nonumber \\
	&= \left[\underline{v}' \sigma(\underline{\mathbf{z}}') -  \underline{v}' \sigma(\underline{\mathbf{z}}) \right]+ 
	\left[\underline{v}'\sigma(\underline{\mathbf{z}}) -  \underline{v} \sigma(\underline{\mathbf{z}})\right] \nonumber \\
	& = \underline{v}'[\sigma(\underline{\mathbf{z}}') - \sigma(\underline{\mathbf{z}})] + \Delta v \sigma(\underline{\mathbf{z}})
	\end{align}
	We consider the following two cases.
	
	\subsubsection{Case 1: $(\Delta w, \Delta b) \not= (0,0)$}
	
	Note that $\mathbf{x}$ has distinct entries and hence linearly independent of $\mathbf{1}$. We have
	\begin{equation}
	\Delta \mathbf{z} \triangleq\Delta w \cdot \mathbf{x} + \Delta b \cdot \mathbf{1}_N \not = \mathbf{0}.
	\end{equation}
	Then, by Taylor's Theorem, we have
	\begin{subequations}
		\begin{align}
		&\sigma(\underline{\mathbf{z}}') - \sigma(\underline{\mathbf{z}})  \nonumber \\
		= &\sigma'(\underline{\mathbf{z}})\circ \Delta \mathbf{z} + \frac{1}{2} \sigma''(\underline{\mathbf{z}})\circ \Delta \mathbf{z} \circ \Delta \mathbf{z} + \mathbf{o}(\|\Delta \mathbf{z}\|_2^2) \\
		= & \Delta w \cdot \sigma'(\underline{\mathbf{z}})\circ \mathbf{x} + \Delta b \cdot \sigma'(\underline{\mathbf{z}})  + \frac{\Delta w^2}{2} \sigma''(\underline{\mathbf{z}})\circ \mathbf{x} \circ \mathbf{x} \nonumber \\ &+ \frac{\Delta b^2}{2} \sigma''(\underline{\mathbf{z}}) + \Delta w \Delta b \cdot \sigma''(\underline{\mathbf{z}})\circ \mathbf{x} + \mathbf{o}(\|\Delta \mathbf{z}\|_2^2)
		\end{align}
	\end{subequations}
	where $\mathbf{o}(\cdot)$ denotes an infinitesimal vector with
	\begin{equation}
		\label{eq::infi_def}
		\lim_{ t \rightarrow 0} \frac{\|\mathbf{o}(t)\|_2}{|t|} = 0.
	\end{equation}
	Then, noting that $\underline{\Theta}$ satisfies \eqref{eq::localmin_system1}, we have
	\begin{subequations}
		\label{eq::infi_analysis_sigmoid1}
		\begin{align}
		&\left\langle \Delta \underline{\mathbf{y}},
		\underline{\mathbf{t}}' - \underline{\mathbf{t}}\right\rangle \nonumber \\
		= & \underline{v}' \left\langle \Delta \underline{\mathbf{y}}, \sigma(\underline{\mathbf{z}}')- \sigma(\underline{\mathbf{z}})\right\rangle + \Delta v\left\langle \Delta \underline{\mathbf{y}},  \sigma(\underline{\mathbf{z}})\right\rangle \\
		= & \underline{v}' \Delta w \left\langle \Delta \underline{\mathbf{y}}, \sigma'(\underline{\mathbf{z}}) \circ \mathbf{x}\right\rangle + \underline{v}' \Delta b \left\langle \Delta \underline{\mathbf{y}}, \sigma'(\underline{\mathbf{z}})\right \rangle + \frac{1}{2} \underline{v}' \Delta w^2 \left\langle \Delta \underline{\mathbf{y}}, \sigma''(\underline{\mathbf{z}}) \circ \mathbf{x} \circ \mathbf{x}\right\rangle  \nonumber \\
		& + \frac{1}{2} \underline{v}' \Delta b^2 \left\langle \Delta \underline{\mathbf{y}}, \sigma''(\underline{\mathbf{z}}) \right\rangle + \underline{v}'\Delta w \Delta b \left\langle \Delta \underline{\mathbf{y}}, \sigma''(\underline{\mathbf{z}}) \circ \mathbf{x} \right\rangle 
		+ \underline{v}'\left\langle \Delta \underline{\mathbf{y}}, \mathbf{o}(\|\Delta\mathbf{z}\|_2^2) \right\rangle \\
		= & \frac{1}{2}  \Delta w^2 \left\langle \Delta \underline{\mathbf{y}}, \underline{v}' \sigma''(\underline{\mathbf{z}}) \circ \mathbf{x} \circ \mathbf{x}\right\rangle  + \frac{1}{2}  \Delta b^2 \left\langle \Delta \underline{\mathbf{y}},\underline{v}' \sigma''(\underline{\mathbf{z}}) \right\rangle  + \underline{v}'\left\langle \Delta \underline{\mathbf{y}}, \mathbf{o}(\|\Delta\mathbf{z}\|_2^2) \right\rangle.
		\end{align}
	\end{subequations}
	First, we note that there exists $\delta_1>0$ such that for any $\underline{\Theta}' \in B(\underline{\Theta}, \delta_1 )$ we have
	\begin{equation}
	\label{eq::v_prime}
	\underline{v}' \underline{v} > 0, \quad \!\! \frac{|\underline{v}|}{2}< |\underline{v}'| < \frac{3|\underline{v}|}{2},
	\end{equation}
	i.e, $\underline{v}'$ is of the same sign with $\underline{v}$, and deviates less than a half of $\underline{v}$. Then, we can denote
	\begin{subequations}
		\begin{gather}
		M_1 = \frac{1}{2}\left\langle \Delta \underline{\mathbf{y}}, \underline{v} \sigma''(\underline{\mathbf{z}}) \circ \mathbf{x} \circ \mathbf{x}\right\rangle > 0 \\
		M_2 = \frac{1}{2}\left\langle \Delta \underline{\mathbf{y}}, \underline{v} \sigma''(\underline{\mathbf{z}}) \right\rangle >0.
		\end{gather}
	\end{subequations}
	Then, from \eqref{eq::v_prime}, we have
	\begin{subequations}
		\label{eq::infi_analysis_shrink1_base}
		\begin{gather}
		\left\langle \Delta \underline{\mathbf{y}}, \underline{v}' \sigma''(\underline{\mathbf{z}}) \circ \mathbf{x} \circ \mathbf{x}\right\rangle >M_1 >  0 \\
		\left\langle \Delta \underline{\mathbf{y}}, \underline{v}' \sigma''(\underline{\mathbf{z}}) \right\rangle >M_2 >0.
		\end{gather}
		and
		\begin{align}
		\left| \underline{v}'\left\langle \Delta \underline{\mathbf{y}}, \mathbf{o}(\|\Delta\mathbf{z}\|_2^2) \right\rangle \right| \leq& \frac{3|\underline{v}|}{2} \left|\langle \Delta \underline{\mathbf{y}}, \mathbf{o}(\|\Delta\mathbf{z}\|_2^2) \rangle \right| \nonumber \\ 
		\leq& \frac{3|\underline{v}|}{2} \left\|\Delta \underline{\mathbf{y}}\right\|_2 \cdot \left\| \mathbf{o}(\|\Delta\mathbf{z}\|_2^2)\right\|_2
		\end{align}
	\end{subequations}
	Then, there exists $\delta_2>0$ such that for any $\underline{\Theta}' \in B(\underline{\Theta}, \delta_2 )$, we have
	\begin{subequations}
		\label{eq::infi_analysis_shrink2_base}
		\begin{align}
		\left\| \mathbf{o}(\|\Delta\mathbf{z}\|^2_2) \right\|_2 <& \frac{1}{12|\underline{v}|\cdot \|\Delta \underline{\mathbf{y}}\|_2}\min\left\{\frac{M_1}{\|\mathbf{x}\|_2^2}, \frac{M_2}{N}\right\} \|\Delta \mathbf{z}\|_2^2 \\
		\leq & \min\left\{\frac{M_1}{\|\mathbf{x}\|_2^2}, \frac{M_2}{N}\right\}\cdot \frac{\Delta w^2 \|\Delta \mathbf{x}\|_2^2+ \Delta b^2 \|\mathbf{1}_N\|_2^2}{6|\underline{v}| \cdot \|\Delta \underline{\mathbf{y}}\|_2}\\ 
		\leq & \frac{\Delta w^2 M_1 + \Delta b^2 M_2}{6|\underline{v}| \cdot \|\Delta \underline{\mathbf{y}}\|_2}.
		\end{align}
	\end{subequations}
	Let $\delta = \min\{\delta_1, \delta_2\}$. For any $\underline{\Theta}' \in B(\underline{\Theta}, \delta)$, we have
	\begin{subequations}
		\label{eq::perturb_dir1}
		\begin{align}
		&\left\langle \Delta \underline{\mathbf{y}},\underline{\mathbf{t}}' - \underline{\mathbf{t}}\right\rangle \nonumber \\
		\label{eq::infi_analysis_shrink1}
		> & \frac{1}{2} \left( \Delta w^2 M_1  + \Delta b^2 M_2 - 3 |\underline{v}| \cdot\| \Delta \underline{\mathbf{y}}\|_2 \cdot \left\|\mathbf{o}(\|\Delta\mathbf{z}\|_2^2)\right\|_2 \right) \\
		\label{eq::infi_analysis_shrink2}
		\geq & \frac{1}{4} \left( \Delta w^2 M_1  + \Delta b^2 M_2 \right) > 0
		\end{align}
	\end{subequations}
	where \eqref{eq::infi_analysis_shrink1} follows from \eqref{eq::infi_analysis_shrink1_base}, and \eqref{eq::infi_analysis_shrink2} follows from \eqref{eq::infi_analysis_shrink2_base}.
	
	Similar to \eqref{eq::loss_decompose}, we can decompose the difference of the empirical loss as
	\begin{align}
	\label{eq::loss_decompose_sigmoid}
	\underline{E}(\underline{\Theta}')-\underline{E}(\underline{\Theta}) = 2\left\langle \Delta \underline{\mathbf{y}}, \underline{\mathbf{t}}' - \underline{\mathbf{t}} \right\rangle + \|\underline{\mathbf{t}}' - \underline{\mathbf{t}} \|^2_2.
	\end{align}
	Then, \eqref{eq::perturb_dir1} implies $\underline{E}(\underline{\Theta}')-\underline{E}(\underline{\Theta})> 0$.
	
	\subsubsection{Case 2: $(\Delta w, \Delta b) = (0,0)$}
	In this case, $\underline{\mathbf{z}}' = \underline{\mathbf{z}}$ and $\Delta v \not = 0$. We have
	\begin{subequations}
		\begin{align}
		&\left\langle \Delta \underline{\mathbf{y}},\underline{\mathbf{t}}' - \underline{\mathbf{t}}\right\rangle =  \Delta v  \left\langle  \Delta \underline{\mathbf{y}},  \sigma(\underline{\mathbf{z}})\right\rangle  = 0.
		\end{align}
	\end{subequations}
	Then, noting that the sigmoid function $\sigma(t)$ is positive for any $t \in \mathbb{R}$, we have
	\begin{align}
		\underline{E}(\underline{\Theta}')-\underline{E}(\underline{\Theta}) &= 2\left\langle \Delta \underline{\mathbf{y}}, \underline{\mathbf{t}}' - \underline{\mathbf{t}} \right\rangle + \|\underline{\mathbf{t}}' - \underline{\mathbf{t}} \|^2_2 \nonumber \\
		& = \|\underline{\mathbf{t}}' - \underline{\mathbf{t}} \|^2_2 = \Delta v ^2 \|\sigma(\underline{\mathbf{z}})\|_2^2 > 0.
	\end{align}
	
	\subsubsection*{}
	Combining Case 1 and 2, we conclude that for any $\underline{\Theta}' \in B(\underline{\Theta},\delta)$ with $\underline{\Theta}' \not = \underline{\Theta}$, $\underline{E}(\underline{\Theta}')>\underline{E}(\underline{\Theta})$, and hence $\underline{\Theta}$ is a strict local minimum.

	\subsection{Local-Min for Perturbed Data}
	In this subsection, we show that local minimum still exists for slightly perturbed training data. Consider a perturbed training data set $(\tilde{\mathbf{x}}, \tilde{\mathbf{y}})$. With a slightly abuse of notation, in the remaining proof, we directly add a tilde on each symbol to denote the corresponding variable of networks with the perturbed data. For example, we denote the empirical loss of the 1-neuron sigmoid network by 
	\begin{equation}
	\underline{\tilde{E}}(\underline{\Theta}, \tilde{\mathbf{x}}, \tilde{\mathbf{y}}) = \| \tilde{\mathbf{y}} - \underline{v}\sigma(\underline{w}\cdot \tilde{\mathbf{x}} + \underline{b} \cdot \mathbf{1})\|^2_2.
	\end{equation}
	Clearly, $\underline{\tilde{E}}$ is analytic (and hence continuous) with respect to $(\underline{\Theta}, \tilde{\mathbf{x}}, \tilde{\mathbf{y}})$. Then, we define a real symmetric matrix 
	\begin{equation}
	\label{eq::proof_thm3_defH}
	H(\underline{\Theta}, \tilde{\mathbf{x}}, \tilde{\mathbf{y}}) = \left[\begin{matrix}
	\left\langle \Delta \underline{\tilde{\mathbf{y}}}, \underline{v} \sigma''(\underline{\tilde{\mathbf{z}}}) \circ \tilde{\mathbf{x}} \circ \tilde{\mathbf{x}}\right\rangle & 
	\left\langle \Delta \underline{\tilde{\mathbf{y}}}, \underline{v} \sigma''(\underline{\tilde{\mathbf{z}}}) \circ \tilde{\mathbf{x}} \right\rangle \\
	\left\langle \Delta \underline{\tilde{\mathbf{y}}}, \underline{v} \sigma''(\underline{\tilde{\mathbf{z}}}) \circ \tilde{\mathbf{x}}\right\rangle & 
	\left\langle \Delta \underline{\tilde{\mathbf{y}}}, \underline{v} \sigma''(\underline{\tilde{\mathbf{z}}}) \right\rangle \\
	\end{matrix}\right] \in \mathbb{R}^{2\times 2}.
	\end{equation}

	From the previous subsection, for training data $(\mathbf{x},\mathbf{y})$ we construct $\underline{\Theta}$, which is a strict local minimum of $\underline{\tilde{E}}(\cdot,\mathbf{x},\mathbf{y})$. 
	Since $\underline{\Theta}$ satisfies \eqref{eq::localmin_system_sigmoid} for $(\mathbf{x}, \mathbf{y})$, we have
	\begin{equation}
	H(\underline{\Theta}, \mathbf{x}, \mathbf{y}) =  \left[\begin{matrix}
	\left\langle \Delta \underline{\mathbf{y}}, \underline{v} \sigma''(\underline{\mathbf{z}}) \circ \mathbf{x} \circ \mathbf{x}\right\rangle & 
	0 \\
	0 & 
	\left\langle \Delta \underline{\mathbf{y}}, \underline{v} \sigma''(\underline{\mathbf{z}}) \right\rangle \\
	\end{matrix}\right]
	\end{equation}
	whose diagonal entries are both positive. Thus, $H(\underline{\Theta}, \mathbf{x}, \mathbf{y})$ is positive definite. Note that $\mathbf{x}$ has distinct entries, and each entry of $H$ is analytic with respect to $(\underline{\Theta}, \tilde{\mathbf{x}}, \tilde{\mathbf{y}})$. There exists $\delta_{1,1}>0$ such that for any $(\underline{\Theta}',\tilde{\mathbf{x}},\tilde{\mathbf{y}})\in B((\underline{\Theta},\mathbf{x},\mathbf{y}), \delta_{1,1})$, $\tilde{\mathbf{x}}$ has distinct entries and $H(\underline{\Theta}', \tilde{\mathbf{x}}, \tilde{\mathbf{y}})$ is still positive definite.

	To show the existence of local minimum for the perturbed data, we present the following lemma.
	\begin{lemma}
		\label{lemma::localmin_after_perturbation}
		Consider a continuous function $G(p,q)$. Suppose that for a given $q$, $G(\cdot,q)$ has a strict local minimum $p$. Then for any $\epsilon >0$, there exists $\delta>0$, such that for any $\tilde{q} \in B(q,\delta)$, $G(\cdot,\tilde{q})$ has a local minimum $\tilde{p}$ with $\|p - \tilde{p}\| < \epsilon$.
	\end{lemma}
	From Lemma \ref{lemma::localmin_after_perturbation}, we see that not only the ``perturbed'' 1-neuron network has a local minimum, but also this new local minimum can be arbitrarily close to the original strict local minimum. 
	
	Formally speaking, there exists $\delta_{1,2}>0$ such that for any $\tilde{\mathbf{x}} \in B(\mathbf{x},\delta_{1,2})$, $\tilde{\mathbf{y}} \in B(\mathbf{y},\delta_{1,2})$, $\underline{\tilde{E}}(\cdot, \tilde{\mathbf{x}},\tilde{\mathbf{y}})$ has a local minimum $\underline{\tilde{\Theta}} = (\underline{\tilde{v}}, \underline{\tilde{w}}, \underline{\tilde{b}})$ with $\underline{\tilde{\Theta}} \in B(\underline{\Theta}, \delta_{1,1})$. Without loss of generality, we can assume $\delta_{1,2} \leq \delta_{1,1}$, since otherwise we can replace $\delta_{1,2}$ by $\delta_{1,1}$ without affecting the above conclusion. Then, we have that $\tilde{\mathbf{x}}$ has distinct entries and that $H(\underline{\tilde{\Theta}},\tilde{\mathbf{x}},\tilde{\mathbf{y}})$ is positive definite. Notice that from \eqref{eq::proof_thm3_defH}, the positive definiteness of $H(\underline{\tilde{\Theta}},\tilde{\mathbf{x}},\tilde{\mathbf{y}})$ also implies $\underline{\tilde{v}} \not=0$.
	
	Before we proceed, we show a special property of the local minimum $\underline{\tilde{\Theta}}$ that is useful for the final step.
	
	Since $\underline{\tilde{\Theta}}$ is a local minimum, $\underline{\tilde{E}}(\cdot, \tilde{\mathbf{x}}, \tilde{\mathbf{y}})$ should have zero gradient at $\underline{\tilde{\Theta}}$, i.e.,
	\begin{subequations}
		\label{eq::gradient_perturb}
		\begin{align}
		\frac{\partial \underline{\tilde{E}}}{\partial \underline{\tilde{v}}} &= 2 \left\langle \Delta \underline{\tilde{\mathbf{y}}} , \sigma(\tilde{\underline{\mathbf{z}}})\right\rangle = 0 \\ 
		\frac{\partial \underline{\tilde{E}}}{\partial \underline{\tilde{w}}}&= 2 \underline{\tilde{v}}\left\langle \Delta \underline{\tilde{\mathbf{y}}} , \sigma'(\underline{\tilde{\mathbf{z}}})\circ \tilde{\mathbf{x}}\right\rangle = 0 \\
		 \frac{\partial \underline{\tilde{E}}}{\partial \underline{\tilde{b}}} &= 2 \underline{\tilde{v}}\left\langle \Delta \underline{\tilde{\mathbf{y}}}  , \sigma'(\tilde{\underline{\mathbf{z}}})\right\rangle = 0.
		\end{align}
	\end{subequations}
	Now, consider an arbitrary perturbation $\underline{\tilde{\Theta}}' = (\underline{\tilde{v}}',\underline{\tilde{w}}',\underline{\tilde{b}}') = (\underline{\tilde{v}} + \Delta v, \underline{\tilde{w}} + \Delta w, \underline{\tilde{b}} + \Delta b)$ in the neighbourhood of $\underline{\tilde{\Theta}}$. Combining \eqref{eq::gradient_perturb} and the Taylor's Theorem, we have
	\begin{subequations}
		\begin{align}
		&\left\langle \Delta \underline{\tilde{\mathbf{y}}},\underline{\tilde{\mathbf{t}}}' - \underline{\tilde{\mathbf{t}}}\right\rangle \nonumber \\
		= & \underline{\tilde{v}}' \left\langle \Delta \underline{\tilde{\mathbf{y}}}, \sigma(\underline{\tilde{\mathbf{z}}}')- \sigma(\underline{\tilde{\mathbf{z}}})\right\rangle + \Delta v\left\langle \Delta \underline{\tilde{\mathbf{y}}},  \sigma(\tilde{\underline{\mathbf{z}}})\right\rangle \\
		= & \underline{\tilde{v}}' \Delta w \left\langle \Delta \underline{\tilde{\mathbf{y}}}, \sigma'(\tilde{\underline{\mathbf{z}}}) \circ \tilde{\mathbf{x}}\right\rangle + \underline{\tilde{v}}' \Delta b \left\langle \Delta \underline{\tilde{\mathbf{y}}}, \sigma'(\underline{\tilde{\mathbf{z}}}) \right\rangle \nonumber \\
		& + \frac{1}{2} \underline{\tilde{v}}' \Delta w^2 \left\langle \Delta \underline{\tilde{\mathbf{y}}}, \sigma''(\underline{\tilde{\mathbf{z}}}) \circ \tilde{\mathbf{x}} \circ \tilde{\mathbf{x}}\right\rangle  + \frac{1}{2} \underline{\tilde{v}}' \Delta b^2 \left\langle \Delta \underline{\tilde{\mathbf{y}}}, \sigma''(\underline{\tilde{\mathbf{z}}}) \right\rangle \nonumber \\
		& + \underline{\tilde{v}}'\Delta w \Delta b \left\langle \Delta \underline{\tilde{\mathbf{y}}}, \sigma''(\underline{\tilde{\mathbf{z}}}) \circ \tilde{\mathbf{x}} \right\rangle 
		+ \underline{\tilde{v}}'\left\langle \Delta \underline{\tilde{\mathbf{y}}}, \mathbf{o}(\|\Delta \tilde{\mathbf{z}}\|_2^2) \right\rangle \\
		= &\frac{1}{2} (\Delta w, \Delta b) H(\underline{\tilde{\Theta}}, \tilde{\mathbf{x}},\tilde{\mathbf{y}}) (\Delta w, \Delta b)^\top + \underline{\tilde{v}}'\left\langle \Delta \underline{\tilde{\mathbf{y}}}, \mathbf{o}(\|\Delta \tilde{\mathbf{z}}\|_2^2) \right\rangle.
		\end{align}
	\end{subequations}
	where $\Delta \tilde{\mathbf{z}} \triangleq \Delta w \cdot \tilde{\mathbf{x}} + \Delta b \cdot \mathbf{1}_{N}$. Noting that $H(\underline{\tilde{\Theta}}, \tilde{\mathbf{x}},\tilde{\mathbf{y}})$ is positive definite, we have
	\begin{equation}
	(\Delta w, \Delta b) H(\underline{\tilde{\Theta}}, \tilde{\mathbf{x}},\tilde{\mathbf{y}}) (\Delta w, \Delta b)^\top \geq \lambda_{\min} \left(|\Delta w|^2 + |\Delta b|^2\right)
	\end{equation}
	where $\lambda_{\min} > 0$ is the minimum eigen-value of $H(\underline{\tilde{\Theta}}, \tilde{\mathbf{x}},\tilde{\mathbf{y}})$. Following a similar analysis as in \eqref{eq::infi_analysis_sigmoid1}-\eqref{eq::perturb_dir1}, there exists $\delta_2 > 0$ such that for any $\underline{\tilde{\Theta}}' \in B(\underline{\tilde{\Theta}},\delta_2)$ and $(\Delta w, \Delta b) \not= (0,0)$, we have
	\begin{equation}
	\left\langle \Delta \underline{\tilde{\mathbf{y}}},\underline{\tilde{\mathbf{t}}}' - \underline{\tilde{\mathbf{t}}} \right\rangle\geq  \frac{\lambda_{\min}}{4}\left(|\Delta w|^2 + |\Delta b|^2\right) > 0.
	\end{equation}
	If $(\Delta w, \Delta b) = (0,0)$, from \eqref{eq::gradient_perturb}, we have 
	\begin{equation}
		\left\langle \Delta \underline{\tilde{\mathbf{y}}},\underline{\tilde{\mathbf{t}}}' - \underline{\tilde{\mathbf{t}}} \right\rangle =  \Delta v  \langle  \Delta \underline{\tilde{\mathbf{y}}},  \sigma(\underline{\tilde{\mathbf{z}}})\rangle  = 0.
	\end{equation}
	We conclude that for any $\underline{\tilde{\Theta}}' \in B(\underline{\tilde{\Theta}},\delta_2)$, we have
	\begin{equation}		
	\label{eq::localmin_property_1hidden}
	\left\langle \Delta \underline{\tilde{\mathbf{y}}},\underline{\tilde{\mathbf{t}}}' - \underline{\tilde{\mathbf{t}}} \right\rangle  = \left\langle \Delta \underline{\tilde{\mathbf{y}}},\underline{\tilde{v}}'\sigma(\underline{\tilde{\mathbf{z}}}') - \underline{\tilde{v}}\sigma(\underline{\tilde{\mathbf{z}}})\right\rangle 
	\geq 0.
	\end{equation}
	
	\subsection{Local-Min for Perturbed Wide Network}
	Given a perturbed training data $(\tilde{\mathbf{x}}, \tilde{\mathbf{y}})$ and the local minimum $\underline{\tilde{\Theta}}$ for the 1-neuron network constructed in the last subsection, what remains is to find a sub-optimal local minimum for the original wide network with the perturbed data. 
	
	We construct the weight configuration $\tilde{\Theta} = (\tilde{\mathbf{v}}, \tilde{\mathbf{w}}$, $\tilde{\mathbf{b}})$ for the perturbed wide network as
	\begin{subequations}
		\begin{equation}
		\tilde{\mathbf{v}} = \underline{\tilde{v}} \cdot \mathbf{q}, \quad
		\tilde{\mathbf{w}} = \underline{\tilde{w}}  \cdot \mathbf{1}_{d_1},\quad 
		\tilde{\mathbf{b}} = \underline{\tilde{b}} \cdot \mathbf{1}_{d_1} 
		\end{equation}
		where $\mathbf{q} \in \mathbb{R}^{d_1}$ is an arbitrary vector satisfying 
		\begin{equation}
		\sum^{d_1}_{i=1} q_i = 1, \quad 0 <q_i < 1, \quad  i= 1,2,\cdots, d_1.
		\end{equation}
	\end{subequations}
	This implies $\tilde{v}_i \underline{\tilde{v}} > 0$ for all $i= 1,2\cdots, d_1$. We show that such $\tilde{\Theta}$ is a local minimum of $\tilde{E}$. First, we note that
	\begin{equation}
	\tilde{\mathbf{z}}_i = \underline{\tilde{w}} \cdot \tilde{\mathbf{x}} + \underline{\tilde{b}} \cdot \mathbf{1}_N =  \underline{\tilde{\mathbf{z}}}, \quad \!\! i=1, 2,\cdots,d_1. 
	\end{equation}
	That is, the input vector to each hidden neuron is the same with that of the perturbed 1-neuron network. Further,  
	\begin{equation}
	\Delta \tilde{\mathbf{y}} \triangleq \sum^{d_1}_{i=1} \tilde{v}_i \sigma(\tilde{\mathbf{z}}_i)- \tilde{\mathbf{y}} = \underline{\tilde{v}}\sigma(\underline{\tilde{\mathbf{z}}})- \tilde{\mathbf{y}} = \Delta \underline{\tilde{\mathbf{y}}}.
	\end{equation}
	That is, for the perturbed wide network, the network output (as well as the difference between the output and the training data) is also the same with that of the perturbed 1-neuron network. From \eqref{eq::localmin_property_1hidden}, for each $1\leq i\leq d_1$, there exists $\delta_{3,i}$ such that for any $\tilde{\Theta}' \in B(\tilde{\Theta},\delta_{3,i})$, we have
	\begin{align}
	\left\langle \Delta \tilde{\mathbf{y}}, \tilde{v}_i'\sigma(\tilde{\mathbf{z}}'_i) -  \tilde{v}_i \sigma(\tilde{\mathbf{z}}_i) \right\rangle = q_i \left\langle \Delta \underline{\tilde{\mathbf{y}}},\frac{\tilde{v}_i'}{q_i}\sigma(\tilde{\mathbf{z}}_i') -  \underline{\tilde{v}} \sigma(\underline{\tilde{\mathbf{z}}}) \right\rangle \geq 0.
	\end{align}
	That is, due to weight perturbation, the difference of the input from the $i$-th hidden neuron to the output layer always has non-negative inner product with $\Delta \tilde{\mathbf{y}}$. Let
	\begin{equation}
	\delta_3 = \min_{1\leq i\leq d_1}\delta_{3,i}.
	\end{equation}
	We have that for any $\tilde{\Theta}' \in B(\tilde{\Theta},\delta_3)$,
	\begin{subequations}
	\begin{align}
	\left\langle \Delta \tilde{\mathbf{y}}, \tilde{\mathbf{t}}' -\tilde{\mathbf{t}} \right\rangle& = 
	\left\langle \Delta \tilde{\mathbf{y}}, \sum^{d_1}_{i=1} \tilde{v}'_i \sigma(\tilde{\mathbf{z}}_i') - \sum^{d_1}_{i=1} \tilde{v}_i \sigma(\tilde{\mathbf{z}}_i) \right\rangle \\
	&=  \sum^{d_1}_{i=1} \left\langle \Delta \tilde{\mathbf{y}}, \tilde{v}'_i\sigma(\tilde{\mathbf{z}}'_i) - \tilde{v}_i \sigma(\tilde{\mathbf{z}}_i)\right\rangle \\
	&\geq 0
	\end{align}
	\end{subequations}
	This implies
	\begin{equation}
	\tilde{E}(\tilde{\Theta}')-\tilde{E}(\tilde{\Theta})  = 2\left\langle \Delta \tilde{\mathbf{y}},\tilde{\mathbf{t}}' - \tilde{\mathbf{t}} \right\rangle + \left\|\tilde{\mathbf{t}}' - \tilde{\mathbf{t}} \right\|^2_2 \geq 0.
	\end{equation}
	Thus, $\tilde{\Theta}$ is a local minimum of $\tilde{E}$. Recall that the sub-optimal local minimum of the original wide network has non-zero empirical loss. We can assume that the range of data perturbation is sufficiently small such that $\tilde{E}(\tilde{\Theta})>0$. Finally, note that $\tilde{\mathbf{x}}$ has distinct entries, from the results in \cite{li2018over}, the network is realizable, i.e., $\min_{\tilde{\Theta}} \tilde{E}(\tilde{\Theta}) = 0$ regardless of $\tilde{\mathbf{y}}$. Thus, $\tilde{\Theta}$ is a sub-optimal local minimum of $\tilde{E}$. We complete the proof. 
	
	
	\ifsmalldata
	\section{Proof of Theorem \ref{thm::bad_localmin_1neuron}}
	Before we start the proof, note that $\sigma(t)\equiv0$ does not satisfy the condition, so there must exist some $t$ such that $\sigma(t)\neq 0$. Since $x\neq 0$, let $w=t/x$, $b=0$ and $v=y/\sigma(t)$, then $\hat{y}=v\sigma(wx+b)=v\sigma(t)=y$ and 
	$E(\Theta)=0$.
	This implies that the network is always realizable.
	
	We first prove the sufficiency of the condition. Due to the realizability of the network, we now only need to show that all $\Theta$ such that $\hat{y}\neq y$ are not local minima. Consider any $\Theta=(v, w, b)$ such that $\hat{y}=v\sigma(wx+b)\neq y$. Note that 
	$E(\Theta)=(y-\hat{y})^2=(y-v\sigma(wx+b))^2$ is convex in $v$.
	
	\begin{itemize}
		\item If $\sigma(wx+b)\neq0$, then there is a strict decreasing path from $(v, w, b)$ to $(v', w, b)$ where $v'=y/\sigma(wx+b)$, so $\Theta=(v, w, b)$ is not a local minimum. 
		\item If $\sigma(wx+b)=0$, then $\hat{y}=v\sigma(wx+b)=0$. Since $\Theta$ is not a global minimum, $y\neq\hat{y}=0$. Due to the condition, $wx+b$ is neither a local maximum nor local minimum of $\sigma$. Therefore, for any $\delta>0$, there exists $b_1, b_2\in B(b, \delta)$ such that $\sigma(wx+b_1)>0, \sigma(wx+b_2)<0$. Furthermore, for any $\delta>0$, there exists $v'\in B(v, \delta)$ such that $v'\neq0$. Thus there is exactly one positive and one negative value in $v'\sigma(wx+b_1)$ and $v'\sigma(wx+b_2)$. Take the one with the same sign as $y$, we obtain a smaller objective value in the neighborhood $B((v, w, b), 2\delta)$. This means that $(v, w, b)$ is not a local minimum. 
	\end{itemize}
	
	Combining the two cases above we finish the sufficiency part of the proof.
	
	For the necessity part, we construct a sub-optimal local minima when the condition does not hold. Without losing generality, assume that $t_0$ is a local minimum and $\sigma(t_0)=0$. Take $w_0=t_0/x, b_0=0$ and any $v_0$ such that $v_0y<0$. Then, $\hat{y}_0=v_0\sigma(t_0)=0$. On the other hand, for any $(v, w, b)$ in the neighborhood of $(v_0, w_0, b_0)$, $\sigma(wx+b)\geq0$ since $t_0$ is a local minimum of $\sigma$. Moreover, since $v$ has the same sign with $v_0$, we have $v\sigma(wx+b)\cdot y\leq0$. 
	Therefore, $(y-v\sigma(wx+b))^2\geq y^2=(y-\hat{y})^2$ for any $(v, w, b)$ in the neighborhood of $(v_0, w_0, b_0)$, which means that $\Theta_0=(v_0, w_0, b_0)$ is a local minimum. Moreover, since $y\neq 0 =\hat{y}_0$, it is a sub-optimal local minimum. The necessity part is completed.
	
	Therefore the condition above is a necessary and sufficient condition for non-existence of sub-optimal local minima in the single-neuron networks.


	\section{Proof of Theorem \ref{thm::bad_localmin_2neuron}}
	For simplicity, we denote the data samples by $\mathbf{x}=(x_1, x_2)^\top, \mathbf{y}=(y_1, y_2)^\top$ and the weight parameters by $\Theta = (\mathbf{w},\mathbf{v},\mathbf{b})$ where $\mathbf{w}=(w_1, w_2)^\top, \mathbf{v}=(v_1, v_2)^\top, \mathbf{b}=(b_1, b_2)^\top$. Although in this paper we only consider quadratic loss, Theorem \ref{thm::bad_localmin_2neuron} also holds for any convex and differentiable loss function. Hence we provide a more general proof here. Specifically, for given data $(\mathbf{x}, \mathbf{y})$, the empirical loss is given by
	\begin{equation}
	E(\Theta) = L(\mathbf{\hat{y}}(\Theta))
	\end{equation}
	where $L: \mathbb{R}^N \rightarrow \mathbb{R}$ is a convex and differentiable function.
	
	We first present a useful lemma.
	\begin{lemma}
		\label{lemma::decrease}
		Consider a convex and differentiable function $L(\cdot)$. Suppose that $\mathbf{\hat{y}}$ is not a global minimum of $L(\cdot)$. Then for any $\epsilon > 0$ there exists $\delta>0$, such that for any $\mathbf{\hat{y}}' \in B_o(\mathbf{\hat{y}}, \delta)$ and $\langle \mathbf{\hat{y}}' - \mathbf{\hat{y}}, - \nabla L(\mathbf{\hat{y}})\rangle > \epsilon \|\mathbf{\hat{y}}'- \mathbf{\hat{y}}\|_2$, we have $L(\mathbf{\hat{y}}') < L(\mathbf{\hat{y}})$.
	\end{lemma}
	
	\begin{proof}
		Note that $\nabla L(\mathbf{\hat{y}}) \not= \mathbf{0}$ since $\mathbf{\hat{y}}$ is not a global minimum. By Taylor expansion of $L(\mathbf{\hat{y}})$ we have
		\begin{equation}
		L(\mathbf{\hat{y}}') = L(\mathbf{\hat{y}}) + \langle \mathbf{\hat{y}}' - \mathbf{\hat{y}}, \nabla L(\mathbf{\hat{y}})\rangle + o(\|\mathbf{\hat{y}}'-\mathbf{\hat{y}}\|_2).
		\end{equation}
		For any $\epsilon>0$, there exists $\delta>0$ such that $|o(\|\mathbf{\hat{y}}'-\mathbf{\hat{y}}\|_2)| < \epsilon\|\mathbf{\hat{y}}'-\mathbf{\hat{y}}\|_2$ for any $\mathbf{\hat{y}}' \in B_o(\mathbf{\hat{y}}, \delta)$. Then for any $\mathbf{\hat{y}}' \in B_o(\mathbf{\hat{y}}, \delta)$ and $\langle \mathbf{\hat{y}}' - \mathbf{\hat{y}}, - \nabla L(\mathbf{\hat{y}})\rangle > \epsilon \|\mathbf{\hat{y}}'- \mathbf{\hat{y}}\|_2$, we have
		\begin{equation}
		\langle \mathbf{\hat{y}}' - \mathbf{\hat{y}}, \nabla L(\mathbf{\hat{y}})\rangle  < - \epsilon \|\mathbf{\hat{y}}' - \mathbf{\hat{y}}\|< - o(||\mathbf{\hat{y}}'-\mathbf{\hat{y}}||_2)
		\end{equation}
		and therefore
		\begin{equation}
		L(\mathbf{\hat{y}}') - L(\mathbf{\hat{y}}) < - \epsilon \|\mathbf{\hat{y}}'-\mathbf{\hat{y}}\|_2 + o(\|\mathbf{\hat{y}}'-\mathbf{\hat{y}}\|_2) < 0.
		\end{equation}
		We complete the proof.
	\end{proof}
	
	Consider a weight parameter $\Theta$ that is not a global minimum. This implies that the corresponding $\mathbf{\hat{y}}$ is not a global minimum of the loss function $L(\cdot)$, and $\nabla L(\mathbf{\hat{y}}) \not = \mathbf{0}$. In what follows, we show that there exists a perturbation of $\Theta$, which can be made arbitrarily small, such that the empirical loss decreases. This implies that $\Theta$ cannot be a local minimum. Specifically, for any $\epsilon > 0$, we prove that there exists $\Theta' = (\mathbf{w}',\mathbf{v}',\mathbf{b}') \in B_o(\Theta, \epsilon)$ such that $E(\Theta') < E(\Theta)$.
	
	Denote
	\begin{equation}
	Z = \sigma(\mathbf{w} \mathbf{x}^\top+\mathbf{b}\mathbf{1}^\top) \in \mathbb{R}^{2 \times 2}.
	\end{equation}
	Let $\mathbf{z_1}^\top$ and $\mathbf{z_2}^\top$ be the first and the second rows of $Z$, respectively.
	
	If $\mathbf{v} =  \mathbf{0}$, then $\mathbf{\hat{y}} = \mathbf{0}$. In this case any perturbation of $\mathbf{w}$ and $\mathbf{b}$ will not change $\mathbf{\hat{y}}$, and hence will not change the empirical loss. Note that the considered network is over-parameterized with $\sigma(0), \sigma'(0) \not=0$. Following the conclusion in \cite{li2018over}, there exists a perturbation of $(\mathbf{w}, \mathbf{b})$, which can be made arbitrarily small, such that there exists a strictly decreasing path from the perturbed point to the global minimum of the loss function, i.e., zero empirical loss. This implies that there exits $\Theta' \in B_o(\Theta, \epsilon)$ such that $E(\Theta') < E(\Theta)$.
	
	If $\mathbf{v} \not =  \mathbf{0}$, without loss of generality we assume $v_1 \not = 0$. Regarding the direction of $\mathbf{z_1}$, we discuss the following two cases. \newline
	
	\textbf{Case 1}: $\langle \mathbf{z_1}, \nabla L(\mathbf{\hat{y}}) \rangle \not = 0$. In this case we can achieve a smaller empirical loss by only perturbing $\mathbf{v}$. Let $a = \langle \mathbf{z_1}, \nabla L(\mathbf{\hat{y}})\rangle$ and $\mathbf{v}' = (v_1- \lambda \text{sign}(a), v_2)^\top$, then
	\begin{equation}
	\mathbf{\hat{y}}' - \mathbf{\hat{y}} = - \lambda \text{sign}(a) \mathbf{z_1}^\top .
	\end{equation}
	By Lemma \ref{lemma::decrease}, there exists $\delta>0$ such that for any $\mathbf{\hat{y}}' \in B_o(\mathbf{\hat{y}}, \delta)$, if
	\begin{equation}
	\label{eq::case1_deltay}
	\langle (\mathbf{\hat{y}}' - \mathbf{\hat{y}})^\top, - \nabla L(\mathbf{\hat{y}})\rangle > \frac{|a|}{2\|\mathbf{z_1}\|_2} \|\mathbf{\hat{y}}- \mathbf{\hat{y}}\|_2
	\end{equation}
	then $L(\mathbf{\hat{y}}') < L(\mathbf{\hat{y}})$. 
	
	By letting 
	\begin{equation}
	\lambda < \min\left\{\frac{\delta}{\|\mathbf{z_1}\|_2},\epsilon\right\}
	\end{equation}
	we have $ ||\mathbf{\hat{y}}'- \mathbf{\hat{y}}||_2 = \lambda \|\mathbf{z_1}\|_2< \delta$, and
	\begin{subequations}
		\begin{align}
		\langle (\mathbf{\hat{y}}' - \mathbf{\hat{y}})^\top, -\nabla L(\mathbf{\hat{y}}) \rangle &= \lambda \text{sign}(a) \langle \mathbf{z_1}, \nabla L(\mathbf{\hat{y}})\rangle \\
		& = \lambda |a| \\
		& = \frac{|a|}{\|\mathbf{z_1}\|_2} \|\mathbf{\hat{y}}'- \mathbf{\hat{y}}\|_2\\
		& >  \frac{|a|}{2\|\mathbf{z_1}\|_2} \|\mathbf{\hat{y}}'- \mathbf{\hat{y}}\|_2.
		\end{align}
	\end{subequations}
	Then $L(\mathbf{\hat{y}}') < L(\mathbf{\hat{y}})$. Note that $\|\mathbf{v}' - \mathbf{v}\|_2 = \lambda < \epsilon$, and hence the perturbation is within $B_o(\Theta, \epsilon)$.

	\textbf{Case 2}: $\langle \mathbf{z_1}, \nabla L(\mathbf{\hat{y}}) \rangle = 0$. In this case we show that we can decrease the empirical loss by only perturbing $\mathbf{w}$ or $\mathbf{b}$. Define
	\begin{equation}
	\begin{aligned}
	\partial_{w_1} \mathbf{z_1} &= (x_1\sigma'(w_1x_1+b_1), x_2\sigma'(w_1x_2+b_1) ) ^ \top \in \mathbb{R}^{2\times 1},\\
	\partial_{b_1} \mathbf{z_1} &= (\sigma'(w_1x_1+b_1), \sigma'(w_1x_2+b_1) ) ^ \top \in \mathbb{R}^{2\times 1}.
	\end{aligned}
	\end{equation}
	We first show that at least one of $\langle \partial_{w_1} \mathbf{z_1}, \nabla L(\mathbf{\hat{y}}) \rangle$ and $\langle \partial_{b_1} \mathbf{z_1}, \nabla L(\mathbf{\hat{y}}) \rangle$ is not $0$. Note that $x_1 \not= x_2$ and that $\sigma'(w_1x_1+b_1), \sigma'(w_1x_2+b_1) \not = 0$ from Assumption \ref{ass::activation3}. Hence $\partial_{w_1} \mathbf{z_1}$ and $\partial_{b_1} \mathbf{z_1}$ are linearly independent in $\mathbb{R}^2$. Therefore, since $\nabla L(\mathbf{\hat{y}})\neq \mathbf{0}$, $\langle \partial_{w_1} \mathbf{z_1}, \nabla L(\mathbf{\hat{y}}) \rangle$ and $\langle \partial_{b_1} \mathbf{z_1}, \nabla L(\mathbf{\hat{y}}) \rangle$ cannot be $0$ at the same time.
	
	Without loss of generality, assume that $\partial_{w_1} \mathbf{z_1}\neq0$. Denote
	\begin{equation}
	a = \langle \partial_{w_1} \mathbf{z_1}, \nabla L(\mathbf{\hat{y}}) \rangle, \quad b = \langle \mathbf{x}, \nabla L(\mathbf{\hat{y}}) \rangle.
	\end{equation}
	Let $\mathbf{w}' = (w_1+ \Delta w_1, w_2)^\top$, $\mathbf{v}' = \mathbf{v}$ and $\mathbf{b}' = \mathbf{b}$ where $\Delta w_1 \not = 0$. We have
	\begin{subequations}
		\label{eq::case2_deltay}
		\begin{align}
		(\mathbf{\hat{y}}' - \mathbf{\hat{y}})^\top \nonumber
		= v_1(& \sigma(w_1x_1 + \Delta w_1x_1+b_1) - \sigma(w_1x_1+b_1) , \\   &\sigma(w_1x_2 + \Delta w_1x_2+b_1)- \sigma(w_1x_2+b_1))\\
		= v_1 (&\Delta w_1x_1 \sigma'(w_1x_1+b_1) + o(\Delta w_1 x_1), \nonumber\\
		&\Delta w_1 x_2 \sigma'(w_1x_2+b_1) + o(\Delta w_1 x_2)) \\
		= v_1& \Delta w_1 \partial_{w_1} \mathbf{z_1} ^\top + o(\Delta w_1)\mathbf{x}^\top
		\end{align}
	\end{subequations}
	and
	\begin{equation}
	\label{eq::case2_innerproduct}
	\langle  (\mathbf{\hat{y}}' - \mathbf{\hat{y}})^\top, \nabla L(\mathbf{\hat{y}})\rangle = v_1 \Delta w_1 a + o(\Delta w_1) b.
	\end{equation}
	Note that $v_1, a, \Delta w_1, \|\partial_{w_1} \mathbf{z_1}\|_2 \not =0$. From \eqref{eq::case2_deltay}, there exists $\delta_1$ such that
	\begin{subequations}
		\label{eq::case2_ineq1}
		\begin{align}
		\|\mathbf{\hat{y}}' - \mathbf{\hat{y}}\|_2 &\leq \|v_1 \Delta w_1 \partial_{w_1} \mathbf{z_1} \|_2 + \|o(\Delta w_1)x\|_2  \\
		&<  2|v_1 \Delta w_1| \cdot \|\partial_{w_1} \mathbf{z_1} \|_2
		\end{align}
	\end{subequations}
	for any $|\Delta w_1|< \delta_1$. Next, from \eqref{eq::case2_innerproduct} there exists an $\delta_2$ such that
	\begin{subequations}
		\begin{align}
		\label{eq::case2_ineq2}
		|\langle  (\mathbf{\hat{y}}' - \mathbf{\hat{y}})^\top, \nabla L(\mathbf{\hat{y}})\rangle| &\geq |v_1 \Delta w_1 a| - |o(\Delta w_1) b|\\
		&> \frac{1}{2}|v_1 \Delta w_1 a|
		\end{align}
	\end{subequations}
	and
	\begin{equation}
	\label{eq::case2_signeq}
	\text{sign}\left( \langle  (\mathbf{\hat{y}}' - \mathbf{\hat{y}})^\top, \nabla L(\mathbf{\hat{y}})\rangle\right)  = \text{sign}\left(v_1 \Delta w_1 a\right)
	\end{equation}
	for any $|\Delta w_1| < \delta_2$.
	
	By Lemma \ref{lemma::decrease}, there exists $\delta_3>0$ such that for any $\mathbf{\hat{y}}' \in B_o(\mathbf{\hat{y}}, \delta_3)$, if
	\begin{equation}
	\langle (\mathbf{\hat{y}}' - \mathbf{\hat{y}})^\top, - \nabla L(\mathbf{\hat{y}})\rangle > \frac{|a|}{4\|\partial_{w_1} \mathbf{z_1}\|_2} \|\mathbf{\hat{y}}'- \mathbf{\hat{y}}\|_2
	\end{equation}
	then $L(\mathbf{\hat{y}}') < L(\mathbf{\hat{y}})$. 
	
	Now we let $\Delta w_1 = - \lambda \text{sign}(v_1a)$ where
	\begin{equation}
	0 < \lambda < \min\left\{ \epsilon, \delta_1 , \delta_2, \frac{\delta_3}{2|v_1|\cdot\|\partial_{w_1} \mathbf{z_1}\|_2}\right\}.
	\end{equation}
	First, we have $\|\mathbf{w}'-\mathbf{w}\|_2 = |\Delta w_1| = \lambda< \epsilon$, so the perturbation is within $B_o(\Theta, \epsilon)$. Second, as $|\Delta w_1| < \delta_1$, \eqref{eq::case2_ineq1} holds, yielding
	\begin{subequations}
		\label{eq::case2_lemmacon1}
		\begin{align}
		\|\mathbf{\hat{y}}' - \mathbf{\hat{y}}\|_2 &<  2|v_1 \Delta w_1| \cdot \|\partial_{w_1} \mathbf{z_1} \|_2 \\
		\label{eq::case1_normineq}
		&= 2 \lambda |v_1| \cdot \|\partial_{w_1} \mathbf{z_1} \|_2\\
		&< \delta_3.
		\end{align}
	\end{subequations}
	Third, as $|\Delta w_1| < \delta_2$, \eqref{eq::case2_ineq2} and \eqref{eq::case2_signeq} hold, yielding 
	\begin{subequations}
		\label{eq::case2_lemmacon2}
		\begin{align}
		\langle  (\mathbf{\hat{y}}' - \mathbf{\hat{y}})^\top,  - \nabla L(\mathbf{\hat{y}})\rangle &= \lambda |v_1 a| + o(\Delta w_1) b\\
		&> \frac{\lambda}{2}|v_1 a| \\
		\label{eq::case2_normineq2}
		& > \frac{|a|}{4\|\partial_{w_1} \mathbf{z_1} \|_2} \|\mathbf{\hat{y}}'- \mathbf{\hat{y}}\|_2
		\end{align}
	\end{subequations}
	where \eqref{eq::case2_normineq2} follows from \eqref{eq::case1_normineq}. Combining \eqref{eq::case2_lemmacon1} and \eqref{eq::case2_lemmacon2}, we have $L(\mathbf{\hat{y}}') < L(\mathbf{\hat{y}})$. We complete the proof.
	
\fi 	
	
	\section{Proof of Lemma \ref{lemma::dim_requirement}}
	Since $\mathbf{a}_1, \mathbf{a}_2, \cdots, \mathbf{a}_{L_1} \in\mathbb{R}^{L_2}$ are also linearly independent, we can add $L_2 - L_1$ vectors $\mathbf{a}_{L_1 + 1}, \mathbf{a}_{L_1 + 2}, \cdots, \mathbf{a}_{L_2} \in \mathbb{R}^{L_2}$, such that $\mathbf{a}_1, \mathbf{a}_2, \cdots, \mathbf{a}_{L_2}$ are linearly independent. We can then construct a matrix
	\begin{equation}
	 A =[\mathbf{a}_1, \mathbf{a}_2, \cdots, \mathbf{a}_{L_2}]  \in \mathbb{R}^{L_2 \times L_2}.
	\end{equation}
	Clearly, $A$ is of full rank, and hence admits an inverse $A^{-1}\in \mathbb{R}^{L_2 \times L_2}$. Denote
	\begin{equation}
	\mathbf{a}'_l \triangleq \left(A^{-1}\right)_{(l,:)} \in \mathbb{R}^{L_2}. 
	\end{equation}
	Then, the set of vectors $\{\mathbf{a}'_l| l = 1, 2,\cdots, L_2 \}$ are also linearly independent, and we have
	\begin{equation}
		\label{eq::lemma4_inner_product}
		\langle  \mathbf{a}'_l, \mathbf{a}_l \rangle = 
		\begin{cases}
		1, \quad& l = l'\\
		0, \quad & l \not= l'
		\end{cases}
	\end{equation}
	for all $1\leq l,l' \leq L_2$. Let $Q \in \mathbb{R}^{(L_2 - d) \times (L_2 - d)}$ be an arbitrary full-rank matrix with all-positive entries, i.e., 
	\begin{equation}
	q_{i,j} > 0 , \quad  1 \leq i,j \leq L_2-d.
	\end{equation}
	Finally, we set
	\begin{equation}
	\mathbf{u}_l = \sum^{L_2-d}_{i=1}q_{i,l} \cdot \mathbf{a}'_{d + i}, \quad l = 1,2,\cdots, L_2 - d.
	\end{equation}
	
	First, we claim that $\mathbf{u}_1, \mathbf{u}_2, \cdots, \mathbf{u}_{L_2-d}$ are linearly independent. Otherwise we can find $\alpha_1, \alpha_2, \cdots, \alpha_{L-d_2}$ that are not all-zero, such that
	\begin{equation}
	\sum^{L_2-d}_{l=1} \alpha_l \cdot \mathbf{u}_l =   \sum^{L_2-d}_{i=1} \sum^{L_2-d}_{l=1} \alpha_l q_{i,l}  \cdot \mathbf{a}'_{d + i} = \mathbf{0}
	\end{equation}
	As $\mathbf{a}'_{d+i}$'s are linearly independent, we must have
	\begin{equation}
	\sum^{L_2-d}_{l=1} \alpha_l q_{i,l} = 0
	\end{equation}
	for all $1\leq i \leq L_2-d$. This is equivalent to
	\begin{equation}
	\sum^{L_2-d}_{l=1} \alpha_l Q_{(i,:)} = \mathbf{0},
	\end{equation}
	a contradiction to the fact that $Q$ is of full rank. Thus, $\mathbf{u}_1, \mathbf{u}_2, \cdots, \mathbf{u}_{L_2-d}$ must be linearly independent.
	
	Second, from \eqref{eq::lemma4_inner_product}, for $1\leq j \leq d$ we have
	\begin{subequations}
		\begin{equation}
		\langle  \mathbf{u}_l, \mathbf{a}_j \rangle = \sum^{L_2-d}_{i=1}q_{i,l} \langle \mathbf{a}'_{d + i}, \mathbf{a}_j \rangle = 0
		\end{equation}
		while for $d+ 1\leq j \leq L_1$ we have
		\begin{equation}
		\langle  \mathbf{u}_l, \mathbf{a}_j \rangle = \sum^{L_2-d}_{i=1}q_{i,l} \langle \mathbf{a}'_{d + i}, \mathbf{a}_j \rangle  = q_{j-d,l} > 0 
		\end{equation}
	\end{subequations}
	which validates \eqref{eq::lemma4_origin} and completes the proof.
		
\section{Proof of Lemma \ref{lemma::perturb_direction}}
To ease notation, we define
\begin{equation}
\Delta T_h = T_h\left(\Theta_{[1:h]}'\right) - T_h\left(\tilde{\Theta}_{[1:h]}'\right)
\end{equation}
for $h=1,2,\cdots, H$. We proof Lemma \ref{lemma::perturb_direction} by induction. For each hidden layer, we first construct $\Theta_{[1:h]}$ satisfying \eqref{eq::proof_thm1_localmin_construct}, and then determine the corresponding $\gamma^{(h)}$ and $\delta^{(h)}$. Notice that once \eqref{eq::proof_thm1_localmin_construct} is satisfied, the input to every hidden-layer neuron is $a$. In our proof, after each $\Theta_{[1:h]}$ is determined, we always assume that $\delta^{(h)}$ is sufficiently small such that the input to neuron (up to the $h$-th hidden layer) is within $[a-\delta, a+\delta]$, and hence the activation is twice-differentiable.

\subsection{First Hidden Layer}
For the first hidden layer, we directly set $W_1$ and $\mathbf{b}_1$ as in \eqref{eq::proof_thm1_localmin_construct1} and \eqref{eq::proof_thm1_localmin_construct1}. Then, we have
	\begin{subequations}
		\label{eq::perturb_component}
		\begin{align}
		\label{eq::first_layer_1_1}
		\left(\Delta T_1\right)_{(j,:)} = &\sigma \left(X^\top (W_1')_{(j,:)}+b'_{1,j}\mathbf{1}_N\right) - \sigma\left(b'_{1,j}\cdot \mathbf{1}_N\right)  \\
		= &\sigma'\left(b'_{1,j}\right)\cdot \left[X^\top (W_1')_{(j,:)}\right] + \frac{1}{2} \sigma''(b'_{1,j})\cdot \left[X^\top (W_1')_{(j,:)}\right] \circ \left[X^\top (W_1')_{(j,:)}\right] \nonumber\\ 
		\label{eq::first_layer_1_2}
		&+\mathbf{o}\left(\left\|X^\top (W_1')_{(j,:)}\right\|^2_2\right) 
		\end{align}
	\end{subequations}
	where \eqref{eq::first_layer_1_1} follows from $W_1 = \mathbf{0}$, and \eqref{eq::first_layer_1_2} is obtained by performing Taylor expansion at $(Z_1)_{(j,:)}(\tilde{\Theta}')$. Here $\mathbf{o}(\cdot)$ denotes an infinitesimal vector with
	\begin{equation}
	\lim_{ t \rightarrow 0} \frac{\|\mathbf{o}(t)\|_2}{|t|} = 0.
	\end{equation}
	From \eqref{eq::innerX}, we have
	\begin{equation}
		\label{eq::inner_product_component1}
		\left\langle \Delta Y_{(i,:)},  X^\top (W_1')_{(j,:)} \right\rangle =  \sum^{d_0}_{k=1}w'_{1,j,k} \left\langle\Delta Y_{(i,:)},X_{(k,:)}\right\rangle = 0
	\end{equation}
	and from \eqref{eq::innerXXdiff}, we have 
	\begin{subequations}
	\label{eq::inner_product_component2}
	\begin{align}
		& \left\langle \Delta Y_{(i,:)}, \left[X^\top (W_1')_{(j,:)}\right] \circ \left[X^\top (W_1')_{(j,:)}\right] \right\rangle \nonumber \\
		= & \left\langle \Delta Y_{(i,:)}, \left(\sum^{d_0}_{k=1}w'_{1,j,k} X_{(k,:)}\right) \circ
		\left(\sum^{d_0}_{k=1}w'_{1,j,k} X_{(k,:)}\right)  \right\rangle
		\\
		=&  \sum^{d_0}_{k=1} (w'_{1,j,k})^2
		\langle \Delta Y_{(i,:)}, X_{(k,:)}\circ X_{(k,:)}\rangle  \nonumber\\
		&+ 2\sum^{d_0}_{k=1} \sum^{k-1}_{k'=1} w'_{1,j,k}w'_{1,j,k'}
		\langle \Delta Y_{(i,:)}, X_{(k,:)}\circ X_{(k',:)}\rangle \\
		=& \sum^{d_0}_{k=1} (w'_{1,j,k})^2
		\langle \Delta Y_{(i,:)}, X_{(k,:)}\circ X_{(k,:)}\rangle.
		\end{align}
	\end{subequations}
	With \eqref{eq::perturb_component}, \eqref{eq::inner_product_component1}, and \eqref{eq::inner_product_component2}, we have
	\begin{multline}
	\left\langle \Delta Y_{(i,:)}, \left(\Delta T_1\right)_{(j,:)}\right\rangle = \frac{1}{2}\sigma''(b'_{1,j}) \sum^{d_0}_{k=1} (w'_{1,j,k})^2
	\left\langle \Delta Y_{(i,:)}, X_{(k,:)}\circ X_{(k,:)}\right\rangle  \\
	+ \left\langle \Delta Y_{(i,:)}, \mathbf{o}\left(\left\|X^\top (W_1')_{(j,:)}\right\|^2_2\right)\right\rangle.
	\end{multline}

	From \eqref{eq::innerXXsame}, each of the inner products $\left\langle \Delta Y_{(i,:)}, X_{(k,:)} \circ X_{(k,:)}\right\rangle$ has the same sign with $[\sigma'(a)]^{H-1}\sigma''(a)$. First, we note that $\sigma''(a) \not = 0$, so there exists $\delta^{(1)}_{1} >0$ such that for any $\Theta' \in B\left(\Theta, \delta^{(1)}_{1}\right)$ and $1 \leq j \leq d_1$, we have
	\begin{subequations}
		\label{eq::proof_lemma2_b_prime}
		\begin{gather}
		\sigma''(b'_{1,j}) \sigma''(a) > 0\\
		\frac{|\sigma''(a)|}{2}< |\sigma''(b'_{1,j})| < \frac{3|\sigma''(a)|}{2},
		\end{gather}
	\end{subequations}
	i.e, each $\sigma''(b_{1,j}')$ is of the same sign with $\sigma''(a)$ and deviates less than a half of $\sigma''(a)$. We can denote
	\begin{equation}
	M_{i,k} = \frac{\sigma''(a)}{4}\left\langle \Delta Y_{(i,:)},  X_{(k,:)} \circ X_{(k,:)} \right\rangle.
	\end{equation}
	Then we have
	\begin{equation}
		\label{eq::proof_lemma2_infi_analysis1_base}
		\frac{[\sigma'(a)]^{H-1}\sigma''(b_i')}{2}\left\langle \Delta Y_{(i,:)},  X_{(k,:)} \circ X_{(k,:)}\right\rangle > [\sigma'(a)]^{H-1} M_{i,k} >  0.
	\end{equation}
	Then, noting that Assumption \ref{ass::inputdata}\ref{ass::inputdata_b} implies $X_{(k,:)} \not = \mathbf{0}$ for all $1\leq k\leq d_0$, we define
	\begin{equation}
	L_i = \min_{1\leq k \leq d_0}\frac{\left|M_{i,k}\right|}{\|X_{(k,:)}\|_2^2} > 0.
	\end{equation}
	Also, noting that \eqref{eq::innerXXsame} implies $\Delta Y_{(i,:)}\not = \mathbf{0}$ for all $1 \leq i \leq d_{H+1}$, there exists $\delta^{(1)}_2>0$ such that for any $\Theta' \in B\left(\Theta, \delta^{(1)}_2 \right)$, the inequality
	\begin{subequations}
		\label{eq::proof_lemma2_infi_analysis2_base}
		\begin{align}	
		\left\|\mathbf{o}\left(\left\|X^\top (W_1')_{(j,:)}\right\|^2_2\right)\right\|_2 
		\leq& \frac{L_i}{2d_0\left\|\Delta Y_{(i,:)}\right\|_2}\left\|X^\top (W_1')_{(j,:)}\right\|^2_2 \\
		= & \frac{L_{i} \left\|\sum_{k = 1}^{d_0}w'_{1,j,k}X_{(k,:)} \right\|_2^2 }{2d_0\left\|\Delta Y_{(i,:)}\right\|_2} \\
		\leq &  \frac{ \sum_{k = 1}^{d_0} (w'_{1,j,k})^2 L_{i} \left\|X_{(k,:)} \right\|_2^2 }{2\left\|\Delta Y_{(i,:)}\right\|_2} \\
		\leq &   \frac{\sum_{k = 1}^{d_0}(w'_{1,j,k})^2 \left|M_{i,k}\right|}{2\left\|\Delta Y_{(i,:)}\right\|_2}
		\end{align}
	\end{subequations}
	holds for $j = 1,2,\cdots, d_0$ and $i = 1,2,\cdots, d_{H+1}$. Let $\delta^{(1)} = \min\left\{\delta^{(1)}_1,\delta^{(1)}_2\right\}$, then for any $\Theta' \in B\left(\Theta, \delta^{(1)}\right)$, we have
	\begin{subequations}
		\label{eq::proof_lemma2_perturb_dir}
		\begin{align}
		&[\sigma'(a)]^{H-1}\left\langle \Delta Y_{(i,:)}, \Delta (T_1)_{(j,:)}\right\rangle \nonumber \\
		\label{eq::proof_lemma2_infi_analysis0}
		=&\frac{1}{2}\sigma''(b'_{1,j})[\sigma'(a)]^{H-1} \sum^{d_0}_{k=1} (w'_{1,j,k})^2
		\left\langle \Delta Y_{(i,:)}, X_{(k,:)}\circ X_{(k,:)}\right\rangle \nonumber \\
		&+ [\sigma'(a)]^{H-1}\left\langle \Delta Y_{(i,:)}, \mathbf{o}\left(\left\|X^\top (W_1')_{(j,:)}\right\|^2_2\right)\right\rangle \\
		\geq & [\sigma'(a)]^{H-1} \sum^{d_0}_{k=1}(w'_{1,j,k})^2M_{i,k} \nonumber \\
		\label{eq::proof_lemma2_infi_analysis1}
		& - \left|[\sigma'(a)]^{H-1}\right| \cdot \left\|\Delta Y_{(i,:)}\right\|_2 \cdot \left\| \mathbf{o}\left(\left\|X^\top (W_1')_{(j,:)}\right\|^2_2\right) \right\|_2 \\
		\label{eq::proof_lemma2_infi_analysis2}
		\geq & \frac{1}{2} \sum^{d_0}_{k=1}( w'_{1,j,k})^2  \left[\sigma'(a)\right]^{H-1} M_{i,k} \\
		\geq & \frac{M_{\min}}{2} \sum^{d_0}_{k=1}( w'_{1,j,k})^2 
		\end{align}
	\end{subequations}
	where
	\begin{equation}
		M_{\min} \triangleq \min_{1\leq i\leq d_{H+1},1\leq k\leq d_0} M_{i,k} [\sigma'(a)]^{H-1} > 0. 
	\end{equation}
	Note that \eqref{eq::proof_lemma2_infi_analysis1} follows from \eqref{eq::proof_lemma2_infi_analysis1_base}, and \eqref{eq::proof_lemma2_infi_analysis2} follows from \eqref{eq::proof_lemma2_infi_analysis2_base}. 
	
	Note that $\sigma$ is twice-differentiable over $[a-\delta, a+\delta]$, and hence has bounded derivative over $[a-\delta, a+\delta]$. As a result, $T_1 = \sigma(Z_1)$ is Lipschitz. Then, there exists $C^{(1)}_j > 0$ such that
	\begin{subequations}
		\begin{align}
		\left\|\left(\Delta T_1\right)_{(j,:)}\right\|_2^2 \leq& C^{(1)}_j \left\| (Z_1)_{(j,:)}(\Theta') - (Z_1)_{(j,:)}(\tilde{\Theta}')\right\|_2^2 \\
		= & C^{(1)}_j\left\| X^\top \left(W_1' - W_1\right)_{(j,:)}\right\|^2_2 \\
		\leq & C^{(1)}_j\left\| X\right\|^2_F \left\|(W_1)'_{(j,:)}\right\|_2^2  =  C^{(1)}_j\left\| X\right\|^2_F \sum^{d_0}_{k=1}(w_{1,j,k})^2 \\
		\leq& \frac{2C^{(1)}_j\left\| X\right\|^2_F}{M_{\min}}[\sigma'(a)]^{(H-1)}\left\langle \Delta Y_{(i,:)}, \left(\Delta T_1\right)_{(j,:)}\right\rangle.
		\end{align}
	\end{subequations}
	for all $1\leq j\leq d_1$. By setting
	\begin{equation}
	\gamma^{(1)} = \frac{M_{\min}}{2\left\| X\right\|^2_F\cdot \max_{1\leq j \leq d_1}C^{(1)}_j}> 0,
	\end{equation}
	we prove \eqref{eq::lemma2_lower_bound} for $h = 1$.
	
	From \eqref{eq::proof_lemma2_perturb_dir}, we see that as long as $(W'_1)_{(j,:)} \not = \mathbf{0}$, we have
	\begin{equation}
		[\sigma'(a)]^{H-1}\left\langle \Delta Y_{(i,:)}, \Delta (T_1)_{(j,:)}\right\rangle \geq \frac{M_{\min}}{2} \sum^{d_0}_{k=1}( w'_{1,j,k})^2 > 0.
	\end{equation} 
	Note that $W_1 = \mathbf{0}$, so $\left(T_1\right)_{(j,:)}\left(\tilde{\Theta}'_{[1:1]}\right)$ is invariant of the input data and thus can be denoted by $\left(T_1\right)_{j,1}\left(\tilde{\Theta}'_{[1:1]}\right)\cdot \mathbf{1}_{N}$. Then, for any $\delta_0 >0$, there exists $\Theta'_{[1:1]} \in B\left(\Theta_{[1:1]}, \delta_0 \right)$ with $(W'_1)_{(j,:)} \not =  \mathbf{0}$, such that
	\begin{subequations}
		\begin{align}
		&[\sigma'(a)]^{H-1}\left\langle
		\Delta Y_{(i,:)},\left(T_1\right)_{(j,:)}\left(\Theta_{[1:1]}'\right)\right\rangle \nonumber \\
		= &[\sigma'(a)]^{H-1} \left(\left\langle
		\Delta Y_{(i,:)},\left(\Delta T_1\right)_{(j,:)}\right\rangle +  \left\langle\Delta Y_{(i,:)}, \left(T_1\right)_{j,1}\left(\tilde{\Theta}'_{[1:1]}\right)\cdot \mathbf{1}_N\right\rangle\right) \\
		= & 
		[\sigma'(a)]^{H-1} \left\langle
		\Delta Y_{(i,:)},\left(\Delta T_1\right)_{(j,:)}\right\rangle
		\geq \gamma^{(1)} \left\| \left(\Delta T_1\right)_{(j,:)}\right\|_2^2 >  0
		\end{align}
	\end{subequations}
	for all $1 \leq i \leq d_{H+1}$ and $1 \leq j \leq d_1$. Thus, \eqref{eq::lemma2_nonzero} holds for $h=1$.
	
	\subsection{Second to $H$-th Hidden Layers}
	Suppose that for the $(h-1)$-th hidden layer where $2\leq h \leq H$, there exist $\Theta_{[1:(h-1)]}$ and $\delta^{(h-1)},\gamma^{(h-1)}>0$ such that \eqref{eq::lemma2_lower_bound} holds for for any $\Theta'_{[1:(h-1)]} \in B\left(\Theta'_{[1:(h-1)]}, \delta^{(h-1)}\right)$. Further, suppose that for any $\delta_0$, there exists $\Theta'_{[1:(h-1)]} \in B\left(\Theta'_{[1:(h-1)]}, \delta_0\right)$, such that \eqref{eq::lemma2_nonzero} holds. Now we consider the $h$-th hiidden layer.
	
	Note that $\sigma$ is twice-differentiable over $[a-\delta, a+\delta]$, and hence has bounded derivative over $[a-\delta, a+\delta]$. As a result, there exists $C^{(h)}$ such that $T_h = \sigma(Z_h)$ is $C^{(h)}$-Lipschitz over $[a-\delta, a+\delta]^{d_h \times N}$. Let
	\begin{subequations}
	\label{eq::proof_lemma2_variable}
	\begin{gather}
	w_{\max}^{(h)} =\frac{\gamma^{(h-1)}}{28 d_{h-1}|\sigma''(a)|\cdot|\sigma'(a)|^{H-h}}\cdot\underset{1\leq i\leq d_{H+1}}{\min} \frac{1}{\|\Delta Y_{(i,:)}\|_2}\\
	\gamma^{(h)} =\frac{\gamma^{(h-1)}}{18C^{(h)}d_{h-1}w_{\max}^{(h)}}.
	\end{gather}
	\end{subequations}
	We construct $\Theta_{[1:h]}$ as follows: 
	\begin{itemize}
		\item Let $\Theta_{[1:h-1]}$ be the same as constructed from the induction;
		\item Let $w_{h,j,k} \in \left(0, w^{(h)}_{\max}\right]$ for any $1\leq j \leq d_h$ and $1\leq k \leq d_{h-1}$;
		\item Set $\mathbf{b}_{h}$ by \eqref{eq::proof_thm1_localmin_construct4}.
	\end{itemize}
	It is easy to verify that such constructed $\Theta_{[1:h]}$ meets \eqref{eq::proof_thm1_localmin_construct}. In the following, we show that there exists $\delta^{(h)}>0$ such that \eqref{eq::lemma2_lower_bound} holds.
	
	To ease notation we define
	\begin{equation}
	\Delta Z_h = Z_h\left(\Theta_{[1:h]}'\right) - Z_h\left(\tilde{\Theta}_{[1:h]}'\right).
	\end{equation}
	Also, note that $W_1=\mathbf{0}$, so $\left(Z_h\right)_{(j,:)}\left(\tilde{\Theta}'_{[1:h]}\right)$ is invariant of the input data and thus can be represented by $\left(Z_h\right)_{j,1}\left(\tilde{\Theta}'_{[1:h]}\right)\cdot\mathbf{1}_N$. 
	Then by the Taylor's Theorem we have
	\begin{subequations}
	\label{eq::proof_lemma2_claim1_induction}	
	\begin{align}
	& \left(\Delta T_h\right)_{(j,:)} 
	 \nonumber \\
	=& \sigma\left((Z_{h})_{(j,:)}\left(\Theta_{[1:h]}'\right)\right) -\sigma\left( (Z_{h})_{(j,:)}\left(\tilde{\Theta}_{[1:h]}'\right) \right) \\
	\label{eq::proof_lemma2_claim1_induction_1}
	= &\sigma'\left((Z_{h})_{(j,:)}\left(\tilde{\Theta}'_{[1:h]} \right) \right)\circ \left(\Delta Z_h \right)_{(j,:)}\nonumber \\
	& +\frac{1}{2}\sigma''\left((Z_{h})_{(j,:)}\left(\tilde{\Theta}_{[1:h]}'\right) \right)\circ\left(\Delta Z_h \right)_{(j,:)}\circ \left(\Delta Z_h \right)_{(j,:)}\nonumber\\
	&+ \mathbf{o}\left(\left\|\left(\Delta Z_h \right)_{(j,:)}\right\|_2^2\right) \\
	\label{eq::proof_lemma2_claim1_induction_2}
	=&\sigma'\left(\left(Z_h\right)_{(j,1)}\left(\tilde{\Theta}'_{[1:h]}\right)\right)\cdot\left[\Delta T^\top_{h-1}\cdot(W_h')_{(j,:)}\right] \nonumber \\
	&+\frac{1}{2}\sigma''\left(\left(Z_{h}\right)_{(j,1)}\left(\tilde{\Theta}'_{[1:h]}\right) \right)\cdot \left[\Delta T^\top_{h-1}\cdot(W_h')_{(j,:)}\right]\circ\left[\Delta T^\top_{h-1}\cdot(W_h')_{(j,:)}\right]\nonumber\\
	&+ \mathbf{o}\left(\left\|\Delta T^\top_{h-1}\cdot(W_h')_{(j,:)}\right\|_2^2\right) 
	\end{align}
	\end{subequations}
	where \eqref{eq::proof_lemma2_claim1_induction_2} holds due to
	\begin{multline}	
	\label{eq::proof_lemma2_z_decomposition}	   
	 (\Delta Z_{h})_{(j,:)}
	= \left[T^\top_{h-1}(\Theta')\cdot(W_h')_{(j,:)}+b_{h,j}'\cdot\mathbf{1}_N\right]  \\  -\left[T^\top_{h-1}(\tilde{\Theta}')\cdot(W_h')_{(j,:)}+b_{h,j}'\cdot\mathbf{1}_N\right] = \Delta T^\top_{h-1}\cdot(W_h')_{(j,:)}.
	\end{multline}
	
	There exists $\delta_1^{(h)}>0$ such that for any $\Theta_{[1:h]}' \in B\left(\Theta_{[1:h]}, \delta^{(h)}_{1}\right)$ we have
	\begin{subequations}
	\label{eq::proof_lemma2_perturbweight}
	\begin{gather}
	    0 < \frac{1}{2}w_{h,j,k}<w_{h,j,k}'<\frac{3}{2}w_{h,j,k}\\
	    \sigma'\left(\left(Z_h\right)_{j,1}\left(\tilde{\Theta}_{[1:h]}'\right)\right)\sigma'(a)> 0\\ \sigma''\left(\left(Z_h\right)_{j,1}\left(\tilde{\Theta}_{[1:h]}'\right)\right)\sigma''(a)> 0 \\ 
	    \frac{1}{2}\left|\sigma'(a)\right| < \left|\sigma'\left(\left(Z_h\right)_{j,1}\left(\tilde{\Theta}'_{[1:h]}\right)\right)\right|<\frac{3}{2}\left|\sigma'(a)\right|\\
	    \frac{1}{2}\left|\sigma''(a)\right|<\left|\sigma''\left(\left(Z_h\right)_{j,1}\left(\tilde{\Theta}_{[1:h]}'\right)\right)\right|<\frac{3}{2}\left|\sigma''(a)\right|
	\end{gather}
	\end{subequations}
	for all $1 \leq j \leq d_h, 1\leq k \leq d_{h-1}$. Taking inner product between $\Delta Y_{(i,:)}$ and the first term in \eqref{eq::proof_lemma2_claim1_induction_2}, we obtain
	\begin{subequations}
	\label{eq::proof_lemma2_induction_base1}
	\begin{align}
	    &[\sigma'(a)]^{H-h}\left\langle\Delta Y_{(i,:)}, \sigma'\left(\left(Z_h\right)_{j,1}\left(\tilde{\Theta}_{[1:h]}'\right)\right)\cdot\left[\Delta T^\top_{h-1}\cdot\left(W_h'\right)_{(j,:)}\right]\right\rangle\\
	    \label{eq::proof_lemma2_term1_1}
	    =&\sigma'\left(\left(Z_h\right)_{j,1}\left(\tilde{\Theta}'_{[1:h]}\right)\right)[\sigma'(a)]^{H-h}\sum_{k=1}^{d_{h-1}}w_{h,j,k}' \left\langle\Delta Y_{(i,:)}, \left(\Delta T_{h-1}\right)_{(k,:)}\right\rangle\\
	    \label{eq::proof_lemma2_term1_2}
	    \geq&\frac{1}{2}[\sigma'(a)]^{H-h+1}\sum_{k=1}^{d_{h-1}}w_{h,j,k}' \left\langle\Delta Y_{(i,:)}, \left(\Delta T_{h-1}\right)_{(k,:)}\right\rangle\\
	    \label{eq::proof_lemma2_term1_3}
	    \geq&\frac{1}{4}\gamma^{(h-1)}\sum_{k=1}^{d_{h-1}}w_{h,j,k}\left\|\left(\Delta T_{h-1}\right)_{(k,:)}\right\|_2^2 \geq 0.
	\end{align}
	\end{subequations}
	where \eqref{eq::proof_lemma2_term1_1} and \eqref{eq::proof_lemma2_term1_2} follow from \eqref{eq::proof_lemma2_perturbweight}, while \eqref{eq::proof_lemma2_term1_3} follows from the induction hypothesis. 
	
	Taking inner product between $\Delta Y_{(i,:)}$ and the second term in \eqref{eq::proof_lemma2_claim1_induction_2}, we obtain
	\begin{subequations}
	\label{eq::proof_lemma2_induction_secondorder}
	\begin{align}
	    &\left| \left \langle\Delta Y_{(i,:)}, \sigma''\left((Z_{h})_{j,1}\left(\tilde{\Theta}_{[1:h]}'\right) \right)\cdot\left[\Delta T^\top_{h-1}\cdot(W_h')_{(j,:)}\right]\circ\left[\Delta T^\top_{h-1}\cdot(W_h')_{(j,:)}\right] \right\rangle\right| \nonumber\\
	    \leq & \left\| \Delta Y_{(i,:)}\right\|_2 \cdot \left|\sigma''\left((Z_{h})_{j,1}\left(\tilde{\Theta}_{[1:h]}'\right) \right)\right| \cdot \left\|\Delta T^\top_{h-1}\cdot(W_h')_{(j,:)}\right\|_2^2\\
	    \leq&\frac{3}{2}\left|\sigma''(a)\right| \cdot\left\|\Delta Y_{(i,:)}\right\|_2\cdot \left(\sum_{k=1}^{d_{h-1}}w_{h,j,k}'\left\|\left(\Delta T_{h-1}\right)_{(k,:)} \right\|_2\right)^2\\
	    \label{eq::proof_lemma2_term2_1}
	    \leq&\frac{3d_{h-1}}{2}\left|\sigma''(a)\right|\cdot\left\|\Delta Y_{(i,:)}\right\|_2\cdot \sum_{k=1}^{d_{h-1}}\left(w_{h,j,k}'\right)^2\left\|\left(\Delta T_{h-1}\right)_{(k,:)} \right\|_2^2\\
	    \leq&\frac{27d_{h-1} w^{(h)}_{\max}}{8}\left|\sigma''(a)\right|\cdot\left\|\Delta Y_{(i,:)}\right\|_2\cdot \sum_{k=1}^{d_{h-1}}w_{h,j,k}\left\|\left(\Delta T_{h-1}\right)_{(k,:)} \right\|_2^2\\
	    \label{eq::proof_lemma2_term2_2}
	    \leq& \frac{27\gamma^{(h-1)}}{224\left|\sigma'(a)\right|^{H-h}}  \sum_{k=1}^{d_{h-1}}w_{h,j,k}\left\|\left(\Delta T_{h-1}\right)_{(k,:)}\right\|_2^2
	\end{align}
	\end{subequations}
	where \eqref{eq::proof_lemma2_term2_1} follows from the Cauchy–Schwarz inequality and \eqref{eq::proof_lemma2_term2_2} follows from \eqref{eq::proof_lemma2_variable}.
	
	Now we consider the 
	infinitesimal term in \eqref{eq::proof_lemma2_claim1_induction_2}. There exists $\delta_2^{(h)}>0$ such that for any $\Theta_{[1:h]}' \in B\left(\Theta_{[1:h]}, \delta^{(h)}_2 \right)$ we have
	\begin{subequations}
	\label{eq::proof_lemma2_induction_smallterm}
	\begin{align}
	    &\left\langle\Delta Y_{(i,:)}, \mathbf{o}\left(\left\|\Delta T^\top_{h-1}\cdot(W_h')_{(j,:)}\right\|_2^2\right) \right\rangle\nonumber \\
	    \leq &\left\|\Delta Y_{(i,:)}\right\|_2\cdot \left\|\mathbf{o}\left(\left\|\Delta T^\top_{h-1}\cdot(W_h')_{(j,:)}\right\|_2^2\right)\right\|_2^2\\
	    \leq&\frac{\left|\sigma''(a)\right|}{18}\cdot\left\|\Delta Y_{(i,:)}\right\|_2\cdot \left(\sum_{k=1}^{d_{h-1}}w_{h,j,k}'\left\|\left(\Delta T_{h-1}\right)_{(k,:)} \right\|_2\right)^2\\
	    \leq&\frac{d_{h-1}}{18}\left|\sigma''(a)\right|\cdot\left\|\Delta Y_{(i,:)}\right\|_2\cdot \sum_{k=1}^{d_{h-1}}\left(w_{h,j,k}'\right)^2\left\|\left(\Delta T_{h-1}\right)_{(k,:)} \right\|_2^2\\
	    \leq&\frac{d_{h-1}}{8}\left|\sigma''(a)\right|\cdot\left\|\Delta Y_{(i,:)}\right\|_2\cdot \sum_{k=1}^{d_{h-1}}\left(w_{h,j,k}\right)^2\left\|\left(\Delta T_{h-1}\right)_{(k,:)} \right\|_2^2\\
	    \leq&\frac{\gamma^{(h-1)}}{224\left|\sigma'(a)\right|^{H-h}}\sum_{k=1}^{d_{h-1}}w_{h,j,k}\left\|\left(\Delta T_{h-1}\right)_{(k,:)}\right\|_2^2.
	\end{align}
	\end{subequations}
	
	Let $\delta^{(h)}=\min\left\{\delta^{(h-1)}, \delta_1^{(h)}, \delta_2^{(h)}\right\}$. Combining \eqref{eq::proof_lemma2_claim1_induction}, \eqref{eq::proof_lemma2_induction_base1}, \eqref{eq::proof_lemma2_induction_secondorder} and \eqref{eq::proof_lemma2_induction_smallterm}, we have
	\begin{equation}
		\label{eq::proof_lemma2_claim1_induction2}
		\left[\sigma'(a)\right]^{H-h} \left\langle \Delta Y_{(i,:)} , \left(\Delta T_h\right)_{(j,:)}\right\rangle \geq \frac{\gamma^{(h-1)}}{8} \sum^{d_{h-1}}_{k=1}w_{h,j,k}\left\|\left(\Delta T_{h-1}\right)_{(k,:)}\right\|_2.
	\end{equation}
	On the other hand, note that $T_h = \sigma(Z_h)$ is $C^{(h)}$-Lipschitz with respect to $Z_h$ over $[a-\delta, a+\delta]^{d_{h} \times N}$. Therefore,
	\begin{subequations}
	\begin{align}
	 \left\|\left(\Delta T_h\right)_{(j,:)}\right\|_2^2  \leq& C^{(h)} \left\| (Z_h)_{(j,:)}(\Theta') - (Z_h)_{(j,:)}(\tilde{\Theta}')\right\|_2^2 \\
	= & C^{(h)}\left\| \sum^{d_{h-1}}_{k=1}w'_{h,j,k}\left(\Delta T_{h-1}\right)_{(k,:)}\right\|^2_2 \\
	\leq & C^{(h)}d_{h-1}\sum^{d_{h-1}}_{k=1}\left(w'_{h,j,k}\right)^2\left\|\left(\Delta T_{h-1}\right)_{(k,:)}\right\|^2_2  \\
	\leq & \frac{9C^{(h)}d_{h-1}w_{\max}^{(h)}}{4}\sum^{d_{h-1}}_{k=1}w_{h,j,k}\left\|\left(\Delta T_{h-1}\right)_{(k,:)}\right\|^2_2  \\
	\label{eq::proof_lemma2_term3_1}
	\leq& \frac{18C^{(h)}d_{h-1}w_{\max}^{(h)}}{\gamma^{(h-1)}}\cdot [\sigma'(a)]^{(H-h)}\left\langle \Delta Y_{(i,:)}, \left(\Delta T_h\right)_{(j,:)}\right\rangle\\
	\label{eq::proof_lemma2_term3_2}
	=&\frac{1}{\gamma^{(h)}}\cdot [\sigma'(a)]^{(H-h)} \left\langle \Delta Y_{(i,:)}, \left(\Delta T_h\right)_{(j,:)}\right\rangle
	\end{align}
	\end{subequations}
	for all $1 \leq i \leq d_{H+1}$ and $1\leq j\leq d_h$, where \eqref{eq::proof_lemma2_term3_1} follows from \eqref{eq::proof_lemma2_claim1_induction2}. We thus prove \eqref{eq::lemma2_lower_bound} for the $h$-th hidden layer. 
	
	To prove \eqref{eq::lemma2_nonzero}, we consider an arbitrary $\delta_0>0$. From the induction hypothesis, there exists $\Theta_{[1:(h-1)]}'\in B\left(\Theta_{[1:(h-1)]}, \delta_0/2\right)$, where $\delta_0>0$ is an arbitrary positive number, such that $\left(\Delta T_{h-1}\right)_{(k,:)} \neq \mathbf{0}$ for all $1\leq k \leq d_{h-1}$. Moreover, since $\sigma'(a)\neq0$, $\sigma(\cdot)$ is monotone in a sufficiently small neighborhood of $a$. Thus, there exists $W_h' \in B(W_h, \delta_0/4)$ and $\mathbf{b}'_h \in B(\mathbf{b}_h,\delta_0/4) $ such that $(\Delta T_h)_{(j,:)}\neq\mathbf{0}$ for every $1\leq j\leq d_h$. As \eqref{eq::lemma2_lower_bound} holds, we have
	\begin{subequations}
		\begin{align}
		&[\sigma'(a)]^{H-h}\left\langle
		\Delta Y_{(i,:)},\left(T_h\right)_{(j,:)}\left(\Theta_{[1:h]}'\right)\right\rangle \nonumber \\
		= &[\sigma'(a)]^{H-h} \left(\left\langle
		\Delta Y_{(i,:)},\left(\Delta T_h\right)_{(j,:)}\right\rangle +  \left\langle\Delta Y_{(i,:)}, \left(T_h\right)_{j,1}\left(\tilde{\Theta}'_{[1:h]}\right)\cdot \mathbf{1}_N\right\rangle\right) \\
		= & 
		[\sigma'(a)]^{H-h} \left\langle
		\Delta Y_{(i,:)},\left(\Delta T_h\right)_{(j,:)}\right\rangle
		\geq \gamma^{(h)} \left\| \left(\Delta T\right)_{(j,:)}\right\|_2^2 >  0
		\end{align}
	\end{subequations}
	for all $1 \leq i \leq d_{H+1}$ and $1 \leq j \leq d_h$. Thus, we prove \eqref{eq::lemma2_nonzero} for the $h$-th hidden layer.
	
	By induction we complete the proof of Lemma \ref{lemma::perturb_direction}.

\ifnonsmooth	
		\section{Proof of Lemma \ref{lemma::proof_thm2_matrix_existence}}
	If $B = \mathbf{0}$, obviously \eqref{eq::proof_thm2_matrix_existence} holds for any $A$ and $C$. 
	
	We consider the case with $B \not = \mathbf{0}$. Denote $d = \mathrm{rank}(B)$. As $L_1\geq L_2 \geq d$, there exists $D_1 \in \mathbb{R}^{d \times L_1}, D_2 \in \mathbb{R}^{L_1 \times L_2}$ such that
	\begin{equation}
	\mathrm{rank}(D_1D_2B) = \mathrm{rank}(D_2B) = \mathrm{rank}(B) = d,
	\end{equation}
	which also implies
	\begin{equation}
		\mathrm{row}(D_1D_2B) = \mathrm{row}(D_2B) = \mathrm{row}(B).
	\end{equation}
	Now define
	\begin{equation}
	p(C) = \det\left(D_1(C+A)BB^\top(C+A)^\top D_1^\top\right).
	\end{equation}
	Note that $p(C)$ is a polynomial with respect to $C$, and hence is an analytic function. Setting $C = -A + D_2$, the matrix
	\begin{equation}
		D_1(C+A)B = D_1D_2B \in\mathbb{R}^{d\times L_3}
	\end{equation}
	is of full row rank. This implies that
	\begin{equation}
		D_1(C+A)BB^\top(C+A)^\top D_1^\top = D_1D_2BB^\top D_2^\top D_1^\top \in\mathbb{R}^{d\times d}
	\end{equation}
	is full-rank, and thus $p(-A+D_2) \not = 0$. Therefore, $p$ is not identically zero. 
	
	From Lemma \ref{lemma::analytic}, $p(C) \not = 0$ holds for generic $C$. That is, for any $\epsilon >0$ there exists $C^*  \in B\left(\mathbf{0}_{L_1 \times L_2}, \epsilon\right)$ such that $p(C^*) \not= 0$. Then we have
	\begin{equation}
		\mathrm{rank}(D_1(C^*+A)B) = d,
	\end{equation}	
	and therefore,
	\begin{equation}
		d = \mathrm{rank}(D_1(C^*+A)B) \leq \mathrm{rank}((C^*+A)B)\leq \mathrm{rank}(B)=d,
	\end{equation}
	implying
	\begin{equation}
		\mathrm{row}((C^*+A)B) = \mathrm{row}(B).
	\end{equation}
	We complete the proof.
\fi

\section{Proof of Lemma \ref{lemma::localmin_system1}}
	Suppose that a weight configuration $\Theta$ satisfies the  equation/inequality system \eqref{eq::localmin_system1}. We consider a perturbed weight configuration
	\begin{equation}
	\Theta' = (\mathbf{v}', \mathbf{w}', \mathbf{b}') = (\mathbf{v} + \Delta \mathbf{v}, \mathbf{w} + \Delta \mathbf{w}, \mathbf{b}+\Delta \mathbf{b})
	\end{equation}
	and denote
	\begin{equation}
		\mathbf{z}_i' = \sigma\left(w_i'\cdot \mathbf{x}^\top + b_i'\cdot \mathbf{1}_N^\top\right) 
	\end{equation}
	as the ``input vector" of the $i$-th hidden neuron at weight configuration $\Theta'$.
	Then, we decompose the difference of the network output by
	\begin{subequations}
		\label{eq::lemma1_deltaz1}
		\begin{align}
		\mathbf{t}(\Theta') - \mathbf{t}(\Theta) 
		= & \sum^{d_1}_{i=1}\left[ v_i' \sigma\left(\mathbf{z}_i'\right) - v_i \sigma\left(\mathbf{z}_i\right) \right]\\
		= & \sum^{d_1}_{i=1}\left[ v_i' \sigma\left(\mathbf{z}_i'\right) - v'_i \sigma\left(\mathbf{z}_i\right) \right] + \sum^{d_1}_{i=1}\left[ v_i' \sigma\left(\mathbf{z}_i\right) - v_i \sigma\left(\mathbf{z}_i\right) \right] \\
		= & \sum^{d_1}_{i=1} v_i'\left[ \sigma\left(\mathbf{z}_i'\right) -  \sigma\left(\mathbf{z}_i\right) \right] + \sum^{d_1}_{i=1} \Delta v_i \sigma\left(\mathbf{z}_i\right) 
		\end{align}
	\end{subequations}
	Since $\sigma$ is twice differentiable, by the Taylor's Theorem, we have
	\begin{subequations}
	\label{eq::lemma1_deltaz2}
	\begin{align}
	&\sigma(\mathbf{z}_i') - \sigma(\mathbf{z}_i)  \nonumber \\
	= &\sigma'(\mathbf{z}_i)\circ \Delta \mathbf{z}_i + \frac{1}{2} \sigma''(\mathbf{z}_i)\circ \Delta \mathbf{z}_i \circ \Delta \mathbf{z}_i + \mathbf{o}(\|\Delta \mathbf{z}_i\|_2^2) \\
	= & \Delta b_i \cdot \sigma'(\mathbf{z}_i) +  \Delta w_{i} \cdot \sigma'(\mathbf{z}_i) \circ \mathbf{x}  + \frac{\Delta b_i ^2}{2} \cdot \sigma''(\mathbf{z}_i) + \Delta w_i \Delta b_i  \cdot \sigma''(\mathbf{z}_i) \circ \mathbf{x} \nonumber \\
	& + \frac{1}{2}\Delta w_i^2 \cdot \sigma''(\mathbf{z}_i)\circ \mathbf{x}\circ \mathbf{x} + \mathbf{o}(\|\Delta \mathbf{z}_i\|_2^2)
	\end{align}
	\end{subequations}
	where $\Delta \mathbf{z}_i \triangleq \mathbf{z}_i' - \mathbf{z}_i$ and $\mathbf{o}(\cdot)$ is an infinitesimal vector as described in \eqref{eq::infi_def}. Combining \eqref{eq::localmin_system1}, \eqref{eq::lemma1_deltaz1}, and \eqref{eq::lemma1_deltaz2}, we have
	\begin{subequations}
	\begin{align}
	&\left \langle  \Delta \mathbf{y}, \mathbf{t}(\Theta') - \mathbf{t}(\Theta) \right\rangle \nonumber \\
	= & \sum_{i=1}^{d_1}v'_i \left\langle \Delta \mathbf{y}, \sigma(\mathbf{z}'_i) - \sigma(\mathbf{z}_i)\right\rangle + \sum_{i=1}^{d_1}\Delta v_i \left\langle \Delta \mathbf{y}, \sigma(\mathbf{z}_i)\right\rangle \\
	= & \frac{1}{2}\sum^{d_1}_{i=1} v'_i \Delta b_i ^2 \left\langle \Delta \mathbf{y}, \sigma''(\mathbf{z}_i)\right\rangle  + \frac{1}{2}\sum^{d_1}_{i=1} v'_i \Delta w_i ^2 \left\langle \Delta \mathbf{y}, \sigma''(\mathbf{z}_i) \circ \mathbf{x} \circ \mathbf{x}\right\rangle \nonumber \\
	&+ \sum^{d_1}_{i=1} v_i' \left\langle \Delta \mathbf{y}, \mathbf{o}\left(\left\|\Delta \mathbf{z}_i\right\|^2_2\right)\right \rangle
	\end{align}
	\end{subequations}
	
	First, we note that \eqref{eq::localmin_system1} implies $v_{i} \not = 0$. There exists $\delta_1 >0$ such that for any $\Theta' \in B(\Theta, \delta_1 )$ such that
	\begin{equation}
	\label{eq::v_prime2}
	v_{i}' v_{i} > 0, \quad \!\! \frac{|v_{i}|}{2}< |v'_{i}| < \frac{3|v_{i}|}{2},
	\end{equation}
	i.e, each $v_{i}'$ deviates less than a half of $v_i$. Then, we can denote
	\begin{subequations}
		\begin{gather}
		M_{1,i} = \frac{v_{i}}{2}\left\langle \Delta \mathbf{y},  \sigma''(\mathbf{z}_i) \circ \mathbf{x} \circ \mathbf{x} \right\rangle\\
		M_{2,i} = \frac{v_{i}}{2}\left \langle \Delta \mathbf{y}, \sigma''(\mathbf{z}_i) \right\rangle.
		\end{gather}
	\end{subequations}
	Then, from \eqref{eq::localmin_system1} and \eqref{eq::v_prime2}, we have
	\begin{subequations}
		\label{eq::infi_analysis1_base}
		\begin{gather}
		v_{i}'\left\langle \Delta \mathbf{y},  \sigma''(\mathbf{z}_i) \circ \mathbf{x} \circ \mathbf{x}\right\rangle >M_{1,i} >  0 \\
		v_{i}'\left \langle \Delta \mathbf{y}, \sigma''(\mathbf{z}_i) \right\rangle>M_{2,i} >0.
		\end{gather}
		and
		\begin{align}
		v'_{i}\left\langle \Delta \mathbf{y} , \mathbf{o}\left(\|\Delta \mathbf{z}_i\|_2^2\right) \right\rangle \geq& -\frac{3|v_{i}|}{2} \left|\left\langle \Delta \mathbf{y}, \mathbf{o}\left(\|\Delta\mathbf{z}_i\|_2^2\right) \right\rangle \right| \nonumber \\ 
		\geq& -\frac{3|v_{i}|}{2} \left\|\Delta \mathbf{y}\right\|_2 \cdot \left\| \mathbf{o}\left(\|\Delta\mathbf{z}_i\|_2^2\right)\right\|_2
		\end{align}
	\end{subequations}
	Now, we define
	\begin{equation}
	L = \min_{1\leq i \leq d_1}\left\{\frac{M_{1,i}}{\|\mathbf{x}\|_2^2}, \frac{M_{2,i}}{N}\right\} > 0
	\end{equation}
	Then, there exists $\delta_2>0$ such that for any $\Theta' \in B(\Theta, \delta_2 )$, the inequality
	\begin{subequations}
		\label{eq::infi_analysis2_base}
		\begin{align}	
		\left\| \mathbf{o}\left(\|\Delta\mathbf{z}_i\|^2_2\right) \right\|_2
		\leq & \frac{L}{12|v_{i}|\cdot \|\Delta \mathbf{y}\|_2}\|\Delta \mathbf{z}_i\|^2_2 \\
		\leq & \frac{L\left(\Delta w_{i}^2 \|\mathbf{x}\|^2_2 + \Delta b_i^2 \|\mathbf{1}_N\|_2^2\right)}{6|v_i|\cdot \|\Delta \mathbf{y}\|_2}
		\\
		\leq & \frac{\Delta w_{i}^2  M_{1,i} + \Delta b_i^2  M_{2,i}}{6|v_{i}|\cdot \|\Delta \mathbf{y}\|_2}
		\end{align}
	\end{subequations}
	holds for any $1\leq i \leq d_1$. Let $\delta = \min\{\delta_1, \delta_2\}$. For any $\Theta' \in B(\Theta, \delta)$, we have
	\begin{subequations}
		\label{eq::perturb_dir}
		\begin{align}
		&\left \langle  \Delta \mathbf{y}, \mathbf{t}(\Theta') - \mathbf{t}(\Theta)\right\rangle \nonumber \\
		\label{eq::infi_analysis1}
		\geq & \frac{1}{2}\sum^{d_1}_{i=1} \left(\Delta w_{i}^2  M_{1,i} + \Delta b_i^2  M_{2,i}\right) - \frac{3}{2} \sum^{d_1}_{i=1} |v_{i}|\cdot \| \Delta \mathbf{y}\|_2\cdot \left\| \mathbf{o}\left(\|\Delta\mathbf{z}_i\|_2^2\right)\right\|_2 \\
		\label{eq::infi_analysis2}
		\geq & \frac{1}{4} \left( \Delta w_{i}^2  M_{1,i} + \Delta b_i^2  M_{2,i} \right) \geq 0
		\end{align}
	\end{subequations}
	where \eqref{eq::infi_analysis1} follows from \eqref{eq::infi_analysis1_base}, and \eqref{eq::infi_analysis2} follows from \eqref{eq::infi_analysis2_base}. 
	
	As in \eqref{eq::loss_decompose}, we decompose the difference of the empirical loss as
	\begin{align}
	E(\Theta')-E(\Theta) = 2 \left \langle  \Delta \mathbf{y}, \mathbf{t}(\Theta') - \mathbf{t}(\Theta)\right\rangle + \|\mathbf{t}(\Theta') - \mathbf{t}(\Theta)\|_2^2.
	\end{align}
	Then, for any $\Theta' \in B(\Theta, \delta)$,  \eqref{eq::perturb_dir} implies $E(\Theta')-E(\Theta) \geq 0$, and hence $\Theta$ is a local minimum of $E$. We complete the proof. 

	\section{Proof of Lemma \ref{lemma::localmin_after_perturbation}}
	Since $p$ is a strict local minimum of $G(\cdot,q)$, there exists $\epsilon_1>0$ such that $G(p',q)>G(p, q)$ for all $0 < \|p' - p\| \leq \epsilon_1$. Denote $\epsilon_2=\min\{\epsilon, \frac{1}{2}\epsilon_1\}$ and consider the closed ball 
	with center $p$ and radius $\epsilon_2$, denoted by $\bar{B}(p, \epsilon_2)$. Thus, $G(p',q)>G(p,q)$ for all $p'\in\partial \bar{B}(p, \epsilon_2)$, where $\partial\bar{B}(p, \epsilon_2)$ denotes the boundary of $\bar{B}(p, \epsilon_2)$. As $\bar{B}(p, \epsilon_2)$ and $\partial \bar{B}(p, \epsilon_2)$ are both compact for a given $q$, $G(\cdot,q)$ is uniformly continuous on $\bar{B}(p, \epsilon_2)$ and $\partial \bar{B}(p, \epsilon_2)$. Therefore, we have
	\begin{equation}
	    \underset{p'\in\partial \bar{B}(p, \epsilon_2)}{\inf}G(p',q)>G(p,q).
	\end{equation} 
	
	Define
	\begin{equation}
	    \xi = \underset{p'\in\partial \bar{B}(p, \epsilon_2)}{\inf}G(p',q) - G(p,q) >0.
	\end{equation}
	Now, in the parameter space of $(p,q)$,  consider the closed ball with center $(p, q)$ and radius $\epsilon_2$, denoted by $\bar{B}\left((p, q), \epsilon_2\right)$. Since $G(\cdot,\cdot)$ is uniformly continuous on $\bar{B}\left((p, q), \epsilon_2\right)$, there exists $\delta\in(0, \epsilon_2]$ such that
	\begin{equation}
	|G(p', q')-G(p'', q'')|<\frac{1}{3}\xi
	\end{equation}
	for all $(p', q'), (p'', q'')\in \bar{B}((p, q), \epsilon_2)$ with $\|(p', q')-(p'', q'')\|<\delta$.
	Therefore, for any given $\tilde{q} \in B(q,\delta)$, we have
	\begin{equation*}
	\begin{aligned}
	G(p,\tilde{q})&<G(p,q)+\frac{1}{3}\xi=\underset{p'\in\partial \bar{B}(p, \epsilon_2)}{\inf}G(p',q)-\frac{2}{3}\xi\\
	&\leq G(p^*(\tilde{q}),q)-\frac{2}{3}\xi<G(p^*(\tilde{q}),\tilde{q})-\frac{1}{3}\xi\\
	&=\underset{p'\in\partial \bar{B}(p, \epsilon_2)}{\inf}G(p',\tilde{q})-\frac{1}{3}\xi<\underset{p'\in\partial \bar{B}(p, \epsilon_2)}{\inf}G(p',\tilde{q})
	\end{aligned}
	\end{equation*}
	where $p^*(\tilde{q})\triangleq\underset{p'\in\partial \bar{B}(p, \epsilon_2)}{\arg\inf}G(p',\tilde{q})$.
	
	Finally, for any given $\tilde{q} \in B(q,\delta)$, consider the function $G(\cdot, \tilde{q})$ in the closed ball $\bar{B}(p, \epsilon_2)$. By the extreme value theorem, there exists $\tilde{p}\in \bar{B}(p, \epsilon_2)$ such that $G(\tilde{p}, \tilde{q})=\underset{p'\in\bar{B}(p, \epsilon_2)}{\inf}G(p',\tilde{q})$. Then we have
	\begin{equation}
	    G(\tilde{p}, \tilde{q})\leq G(p, \tilde{q})<\underset{p'\in\partial \bar{B}(p, \epsilon_2)}{\inf}G(p',\tilde{q}),
	\end{equation}
	implying that $\tilde{p}$ is an interior point of $\bar{B}(p, \epsilon_2)$. Then, in the interior of $\bar{B}(p, \epsilon_2)$, we can find a closed ball centered at $\tilde{p}$. That is, there exists $\epsilon_3>0$ such that
	\begin{equation}
	\bar{B}(\tilde{p}, \epsilon_3)\subseteq\bar{B}(p, \epsilon_2).
	\end{equation}
	Thus, $G(\tilde{p}, \tilde{q})\leq G(p', \tilde{q})$ for all $p'\in \bar{B}(\tilde{p}, \epsilon_3)$, implying that $\tilde{p}$ is a local minimum of $G(\cdot, \tilde{q})$. Moreover, since $\epsilon_2\leq\epsilon$, we have $\|\tilde{p} - p\| < \epsilon_2\leq\epsilon$. The proof is completed.
	
	
\end{document}